\documentclass[11 pt]{article}
\usepackage{amsmath}
\usepackage{style}
\usepackage{dsfont}
\title{\LARGE\bfseries Overcoming the Curse of Dimensionality in Reinforcement Learning Through Approximate Factorization}
\author{Chenbei Lu\textsuperscript{$*$},  Laixi Shi\textsuperscript{$\dagger,1$}, Zaiwei Chen\textsuperscript{$\ddagger$}, Chenye Wu\textsuperscript{$\S$}, Adam Wierman\textsuperscript{$\dagger,2$}\\
{\small
\textsuperscript{$*$}\textit{Institute for Interdisciplinary Information Sciences, Tsinghua University,} \href{mailto:clu5@caltech.edu}{\textit{lcb20@mails.tsinghua.edu.cn}}}\\
{\small
\textsuperscript{$\dagger$}\textit{Computing \& Mathematical Sciences Department, Caltech,} , \href{mailto:laixis@caltech.edu}{\textit{\textsuperscript{$1$}laixis@caltech.edu}}, \href{mailto:adamw@caltech.edu}{\textit{\textsuperscript{$2$}adamw@caltech.edu}}
}\\
{
\small
\textsuperscript{$\ddagger$}\textit{Edwardson School of Industrial Engineering, Purdue University,} \href{mailto: chen5252@purdue.edu}{\textit{chen5252@purdue.edu}}
} \\
{\small
\textsuperscript{$\S$}\textit{School of Science and Engineering, CUHK Shenzhen,} \href{mailto:chenyewu@yeah.net}{\textit{chenyewu@yeah.net}}}
}
\date{\vspace{-0.4 in}}
\begin{document}
\maketitle

\begin{abstract} 
Reinforcement Learning (RL) algorithms are known to suffer from the curse of dimensionality, which refers to the fact that large-scale problems often lead to exponentially high sample complexity. A common solution is to use deep neural networks for function approximation; however, such approaches typically lack theoretical guarantees. To provably address the curse of dimensionality, we observe that many real-world problems exhibit task-specific model structures that, when properly leveraged, can improve the sample efficiency of RL. Building on this insight, we propose overcoming the curse of dimensionality by approximately factorizing the original Markov decision processes (MDPs) into smaller, independently evolving MDPs. This factorization enables the development of sample-efficient RL algorithms in both model-based and model-free settings, with the latter involving a variant of variance-reduced Q-learning. We provide improved sample complexity guarantees for both proposed algorithms. Notably, by leveraging model structure through the approximate factorization of the MDP, the dependence of sample complexity on the size of the state-action space can be exponentially reduced. Numerically, we demonstrate the practicality of our proposed methods through experiments on both synthetic MDP tasks and a wind farm-equipped storage control problem.
\end{abstract}

\section{Introduction}\label{sec:intro}
In recent years, reinforcement learning (RL) \citep{sutton2018reinforcement} has become a popular framework for solving sequential decision-making problems in unknown environments, with applications across different domains such as robotics \citep{kober2013reinforcement}, transportation \citep{haydari2020deep}, power systems \citep{chen2022reinforcement}, and financial markets \citep{charpentier2021reinforcement}. Despite significant progress, {\em the curse of dimensionality} remains a major bottleneck in RL tasks \citep{sutton2018reinforcement}. Specifically, the sample complexity grows geometrically with the dimensionality of the state-action space of the environment, posing challenges for large-scale applications.
For example, in robotic control, even adding one more degree of freedom to a single robot can significantly increase the complexity of the control problem \citep{spong2020robot}. 

To overcome the curse of dimensionality in sample complexity, a common approach is incorporating function approximation to approximate either the value function or the policy using a prespecified function class (e.g., neural networks) \citep{sutton2018reinforcement}. While this approach works in certain applications, these methods heavily rely on the design of the function approximation class, tailored parameter tuning, and other empirical insights. Moreover, they often lack theoretical guarantees. To the best of our knowledge, most existing results are limited to basic settings with linear function approximation \citep{tsitsiklis1996analysis,bhandari2018finite,srikant2019finite,chen2023target}. A few exceptions using non-linear approximation (e.g., neural networks) \citep{fan2020theoretical,xu2020finite} typically require additional assumptions tailored for theoretical analysis, which creates a gap with practical algorithms. As a result, achieving provable sample efficiency for large-scale sequential decision-making problems remains a significant challenge in RL \citep{sutton2018reinforcement}.

Formally, an RL problem is usually modeled as a Markov decision process (MDP), where the environment is represented by an unknown transition kernel and reward function. It is important to note that, in many real-world applications, the transition probabilities and reward functions exhibit inherent factored structures, which, if properly leveraged, could mitigate the challenges of the curse of dimensionality in sample complexity.
As an illustration, consider model-based RL, where the curse of dimensionality translates to the challenge of estimating the transition probabilities that are potentially high-dimensional. To address this challenge, fortunately, some partial dependency structures are prevalent in applications such as sustainability efforts \citep{yeh2023sustaingym}, queuing systems \citep{wei2023sample}, and networked systems \citep{qu2020scalable}. Specifically, the transition dependency between the current state-action pair and the next state is not fully connected, with the next state often depending on only a subset of the current state-action variables.
For example, in power system control, the dynamics of power demand, while being stochastic, are independent of the power generation decisions.
Similarly, in multi-agent robotic control tasks, the actions of one robot often do not affect others that are far away geographically, allowing for more localized and sparse modelings and solutions. By identifying and leveraging these structures in algorithm design, it is possible to significantly reduce sample complexity.

One notable framework that captures this idea is the factored MDP framework \citep{osband2014near}, where the original MDP is decomposed into a set of smaller and independently evolving MDPs. In this case, the overall sample complexity depends on the sum (rather than the product) of the sizes of the state-action spaces of each individual MDP. As a result, the problem size no longer scales exponentially with the dimension of the problem, successfully overcoming the curse of dimensionality. While the factored MDP framework seems to be a promising approach, it relies on the assumption that the original MDP can be perfectly factorized into a set of small MDPs. Such a structural assumption limits the applicability of the factored MDP approach. In addition, to the best of our knowledge, there are no provable model-free algorithms for the factored MDP setting, which restricts the development of more flexible approaches that do not require explicit estimation of the environment model.

\subsection{Contributions} In this paper, we remove the limitations of the factored MDP framework by proposing an approximate factorization scheme that decomposes an MDP into low-dimensional components. This scheme can be viewed as an extended framework of the factored MDP, allowing for imperfectly factorized models. Building on this framework, we develop provable RL algorithms that overcome the curse of dimensionality and achieve improved sample efficiency compared to prior works in both model-based and model-free paradigms.
The main contributions are summarized below.

\paragraph{Approximate Factorization of MDP and Efficient Sampling.} To tackle the curse of dimensionality, we introduce an approximate factorization scheme that flexibly decomposes any MDP into low-dimensional components. This scheme offers more flexibility for factorization operations to find an optimal trade-off between sample complexity and optimality by identifying and exploiting the dependency structure within the environment model. Using these approximately factorized MDPs, we propose a multi-component factorized synchronous sampling approach to estimate the model, framed as a cost-optimal graph coloring problem. This method allows simultaneous sampling of multiple components from a single sample, unlike previous factored MDP methods where each sample estimates the transition of only one component. This sampling approach is key to reducing sample complexity in both our model-based and model-free algorithms.

\paragraph{Model-Based Algorithm with Approximate Factorization.} 
We develop a novel model-based RL algorithm that leverages the proposed synchronous sampling method and fully exploits the low-dimensional structures identified through approximate factorization. By breaking down the large-scale kernel estimation problem into smaller, manageable components, our algorithm achieves provable problem-dependent sample complexity (cf. Theorem \ref{thm:summary}), outperforming existing minimax-optimal bounds for standard MDPs \citep{azar2012sample}. Moreover, our approach generalizes the factored MDP framework and, when applied to this setting, improves the best-known sample complexity \citep{chen2020efficient} by up to a factor of the number of components.  The key technical challenge is to address the statistical correlation between samples in synchronous and factorized sampling. By controlling these correlations in the variance analysis, we reduce the sample complexity dependence from the sum of component sizes to the size of the largest component.

\paragraph{Model-Free Variance-reduced Q-Learning with Approximate Factorization.} We propose a model-free algorithm --- \emph{variance-reduced Q-learning with approximate factorization} (VRQL-AF) --- that incorporates the same synchronous sampling method. It achieves the same problem-dependent sample complexity guarantees (cf. Theorem \ref{thm_variance-reduced-Q-learning}) as our model-based approach (up to logarithmic terms), outperforming existing minimax-optimal algorithms for standard MDPs \citep{wainwright2019variance}. When applied to the factored MDPs (which are special cases of our setting), to the best of our knowledge, VRQL-AF is the first provable model-free algorithm that achieves near-optimal sample complexity guarantees. The improvement results from our tailored factored empirical Bellman operator for approximately factorized MDPs, integrated with a variance-reduction approach to minimize the variance of the stochastic iterative algorithm. This integration requires a refined statistical analysis to tightly control estimation errors and iteration variance across the multiple factored components.

\subsection{Related Work}
Our work contributes to a few key literatures within the RL community. We discuss each in turn below.

\paragraph{Finite-Sample Analysis for Model-Based Algorithms.} 
Our proposed algorithm (cf. Algorithm \ref{alg:Value Iteration (VI)}) follows the model-based RL approach, where the learning process  involves model estimation and planning. Model-based approaches have been extensively studied \citep{azar2012sample, agarwal2020model, gheshlaghi2013minimax, sidford2018near, azar2017minimax, jin2020provably}, achieving minimax-optimal sample complexity of $\widetilde{\mathcal{O}}({|\mathcal{S}||\mathcal{A}|}{\epsilon^{-2}(1-\gamma)^{-3}})$ in the generative model setting \citep{li2020breaking}. This minimax-optimal bound is established by considering all possible MDPs in a worst-case manner, without leveraging any additional structure in the problem. In contrast, by leveraging the structure for algorithm design, we achieve sample complexity with exponentially reduced dependency on the size of the state and action space (cf. Theorem~\ref{thm:summary}) when the MDP is perfectly factorizable.

\paragraph{Finite-Sample Analysis for Model-Free Algorithms.} 
Our proposed algorithm also aligns with model-free RL, which does not estimate the model but directly optimize the policy \citep{sutton2018reinforcement}. A vast body of literature focuses on Q-learning \citep{tsitsiklis1994asynchronous, jaakkola1993convergence, szepesvari1997asymptotic, kearns1998finite, even2003learning, wainwright2019stochastic, chen2024lyapunov,li2024q} with various sampling settings, demonstrating a minimax sample complexity of $\widetilde{\mathcal{O}}({|\mathcal{S}||\mathcal{A}|}{\epsilon^{-2}(1-\gamma)^{-4}})$.
With further advancements like variance reduction, Q-learning has been shown to achieve a minimax-optimal sample complexity of $\widetilde{\mathcal{O}}({|\mathcal{S}||\mathcal{A}|}{\epsilon^{-2}(1-\gamma)^{-3}})$ in the generative model setting \citep{wainwright2019variance}. In contrast, our work leverages the approximate factorization structure of MDPs to further enhance sample efficiency. By designing a factored empirical Bellman operator with variance reduction, we achieve exponentially reduced sample complexity with respect to the state-action space size (cf. Theorem~\ref{thm_variance-reduced-Q-learning}) with matching minimax dependence on the other parameters. 

\paragraph{Factored MDPs.} Our model generalizes the framework of factored MDPs \citep{boutilier1995exploiting, boutilier1999decision}, extending it to account for approximation errors. Most existing work on factored MDPs is set in an episodic framework and primarily analyzes regret performance \citep{guestrin2003efficient, osband2014near, xu2020reinforcement, tian2020towards, chen2020efficient}. In particular, the state-of-the-art results translate into a sample complexity of \( \widetilde{\mathcal{O}} ( \sum_{k=1}^K |\mathcal{X}_k| \epsilon^{-2} (1-\gamma)^{-3} ) \), where the complexity scales with the sum of the state-action space sizes \( |\mathcal{X}_k| \) across all factored components.
Building on this line of work, we propose a factorized synchronous sampling technique that enables simultaneous updates for multiple components using a single sample. By coupling this with refined cross-component variance analysis, we reduce the sample complexity to as low as \( \widetilde{\mathcal{O}} \left( \max_k |\mathcal{X}_k| \epsilon^{-2} (1-\gamma)^{-3} \right) \) in an instance-dependent manner, which only depends on the maximal component size rather than the sum. This result matches the lower bounds established in prior work \citep{xu2020reinforcement, chen2020efficient}, up to logarithmic factors.
Crucially, our approach does not require the MDP to exhibit perfect factorizability, allowing for broader applicability to general MDPs. 

\paragraph{RL with Function Approximation.} 
To make RL problem sample efficient for large-scale problems, a common approach is to employ function approximation \citep{sutton2018reinforcement}. Intuitively, the key idea is to limit the searching space of an RL problem to a predefined function class, in which each function can be specified with a parameter that is low-dimensional. This approach has achieved significant empirical success \citep{mnih2015human,silver2017mastering}. However, RL with function approximation is not theoretically well understood except under strong structural assumptions on the approximating function class, such as the function class being linear \citep{tsitsiklis1996analysis,bhandari2018finite,srikant2019finite,chen2023target,chen2023two}, the Bellman completeness being satisfied \citep{fan2020theoretical}, or others \citep{dai2018sbeed, wang2020reinforcement}. Also, the function approximation often targets to approximate the Q-values, instead of exploiting the inherent transition kernel and reward function structures.
In this work, we take a different approach by leveraging approximate factorization structures instead of implementing function approximation. It is also worth noting that our approach is highly flexible and can be further extended by incorporating function approximation techniques, providing an even broader framework for tackling large-scale RL problems.

\section{Model and Background}
We consider an infinite-horizon discounted MDP $M = (\mathcal{S}, \mathcal{A}, P, r, \gamma)$, where $\mathcal{S}$ is the finite state space, $\mathcal{A}$ is the finite action space, $P$ is the transition kernel, with $P(s' \mid s, a)$ denoting the transition probability from state \( s \) to \( s' \) given action \( a \), $r: \mathcal{S} \times \mathcal{A} \to [0,1]$\footnote{Since we work with finite MDPs, assuming bounded reward is without loss of generality.} is the reward function, and $\gamma \in (0,1)$ is the discount factor. The model parameters of the underlying MDP (i.e., the transition kernel $P$ and the reward function $r$) are unknown to the agent. 

At each step, the agent takes action based on the current state of the environment and transitions to the next state in a stochastic manner, according to the transition probabilities, and receives a one-stage reward. This process repeats, and the goal of the agent is to find an optimal strategy (also known as policy) for choosing actions to maximize the long-term discounted cumulative reward. More formally, a policy $\pi$ is a mapping from the state space $\mathcal{S}$ to the set of probability distributions supported on the action space. Given a policy $\pi$, its Q-function $Q^\pi\in\mathbb{R}^{|\mathcal{S}||\mathcal{A}|}$ is defined as $Q^\pi(s,a)=\mathbb{E}_{\pi,P}[\sum_{t=0}^\infty\gamma^tr(s_t,a_t)\mid s_0=s,a_0=a]$ for all $(s,a)\in\mathcal{S}\times \mathcal{A}$, where $\mathbb{E}_{\pi.P}[\,\cdot\,]$ means that the trajectories are generated by following the transition $P$ and choosing actions according to the policy $\pi$, i.e., $s_{t+1} \sim P(\cdot\mid s_t, a_t)$, $a_{t+1}\sim \pi(\cdot\mid s_{t+1})$ for all $t\geq 0$. With the Q-function defined above, the goal of an agent is to find an optimal policy $\pi^*$ such that its associated Q-function is maximized uniformly for all state-action pairs $(s,a)\in\mathcal{S}\times \mathcal{A}$. It has been shown that such an optimal policy always exists \citep{puterman2014markov}.

The heart of solving an MDP is the Bellman optimality equation, which states that the optimal Q-function, denoted by $Q^*$, is the unique solution to the fixed-point equation $Q=\mathcal{H}(Q)$, where $\mathcal{H}:\mathbb{R}^{|\mathcal{S}||\mathcal{A}|}\xrightarrow[]{} \mathbb{R}^{|\mathcal{S}||\mathcal{A}|}$ is the Bellman optimality operator defined as
\begin{align*}
    [\mathcal{H}(Q)](s,a)=\mathbb{E}_{s'\sim P(\cdot\mid s,a)}\left[r(s,a)+\gamma \max_{a'\in\mathcal{A}}Q(s',a')\right],\quad \forall\,(s,a)\in\mathcal{S}\times \mathcal{A}.
\end{align*}
Moreover, once $Q^*$ is obtained, an optimal policy can be computed by choosing actions greedily based on $Q^*$. Therefore, the problem reduces to estimating the optimal Q-function $Q^*$.

Suppose that the model parameters of the MDP are known. Then, we can efficiently find $Q^*$ through the value iteration method: $Q_{k+1}=\mathcal{H}(Q_k)$, where $Q_0$ is initialized arbitrarily. Since the operator $\mathcal{H}(\cdot)$ is a contraction mapping with respect to the $\ell_\infty$-norm \citep{bertsekas1996neuro,sutton2018reinforcement,puterman2014markov}, the Banach fixed-point theorem \citep{banach1922operations} guarantees that the value iteration method converges geometrically to $Q^*$.

In RL, the model parameters of the underlying MDP are unknown to the agent. Therefore, one cannot directly perform the value iteration method to find an optimal policy. A natural way to overcome this challenge is to first estimate the model parameters through empirical sampling and then perform the value iteration method based on the estimated model. This is called the model-based method. Another approach is to solve the Bellman optimality equation using the stochastic approximation method without explicitly estimating the model, which leads to model-free algorithms such as Q-learning \citep{watkins1992q}.

Throughout this paper, we work on a widely-used sampling mechanism setting --- generative model \citep{kearns2002sparse, kakade2003sample}. It enables us to query any state-action pair \((s, a)\) to obtain a random sample of the next state \(s'\) and the immediate reward \(r(s, a)\) according to the potentially stochastic transition kernel and reward function.

\section{Approximate Factorization of Markov Decision Processes}
\label{Sec_def_AF}

To begin with, we introduce the concept of  {\em approximate factorization}, which is the foundation of the contributions in this paper.
Suppose that each state $s \in \mathcal{S}$ is an $n$-dimensional vector, and the state space can be written as $\mathcal{S} = \prod_{i=1}^n \mathcal{S}_i$, where $\mathcal{S}_i \subseteq \mathbb{R}$ for all $i \in [n]:=\{1,2,\dots,n\}$. Similarly, suppose that each action is an $m$-dimensional vector and $\mathcal{A} = \prod_{j=1}^m \mathcal{A}_j$, where $\mathcal{A}_j \subseteq \mathbb{R}$ for all $j \in [m]$. As $n$ and $m$ increase, the total number of state-action pairs increases geometrically, which is the source of the curse of dimensionality. Throughout this paper, each component $s[i]\in\mathcal{S}_i$ of a state $s\in\mathbb{R}^n$, where $i\in [n]$, is called a substate. A subaction is defined similarly. To introduce the approximate factorization scheme, we start by introducing the following useful definitions: \emph{scope}, \emph{scope set} and \emph{scope variable}.
\begin{definition}\label{def:scope}
    Consider a factored $d$-dimensional set \( \mathcal{X} = \prod_{i=1}^d\mathcal{X}_i \), where $\mathcal{X}_i\subseteq \mathbb{R}$ for all $i\in [d]$. For any \( Z \subseteq [d] \) (which we call a \emph{scope}), the corresponding \emph{scope set} \( \mathcal{X}[Z] \) is defined as $\mathcal{X}[Z] = \prod_{i \in Z} \mathcal{X}_i$. Any element $x[Z]\in \mathcal{X}[Z]$ (which is a $|Z|$-dimensional vector) is called a \textit{scope variable}.
    % where \( \bigotimes \) denotes the Cartesian product. 
\end{definition}
When $Z$ contains only one element, , i.e., \( Z = \{i\} \) for some $i\in [d]$, we denote \( x[i] \) as \( x[\{i\}] \).

With the help of Definition \ref{def:scope}, now we are ready to introduce the approximate factorization scheme of the MDP $M = (\mathcal{S}, \mathcal{A}, P, r, \gamma)$. Specifically, an approximate factorization is characterized by a tuple $\omega = \{\omega^P, \omega^R\}$, involving operations on both the transition kernel and the reward function. We begin with the transition kernel.

\subsection{Approximate Factorization of the Transition Kernel}
A factorization scheme $\omega^P$ of the transition kernel is characterized by a tuple $\omega^P = \big\{ K_\omega, \{Z_k^S \subseteq [n] \mid k\in [K_\omega]\}, \{Z_k^P \subseteq [n+m] \mid k\in [K_\omega]\}\big\}$, where 
\begin{itemize}
    \item \( K_\omega \in [n] \) represents the number of components into which we partition the transition kernel;
    \item {\( \{Z_k^S\mid k\in [K_\omega]\} \) is a set of scopes with respect to the $n$-dimensional state space $\mathcal{S}=\prod_{i=1}^n\mathcal{S}_i$. Moreover,  \( \{Z_k^S\mid k\in [K_\omega]\} \) forms a partition of $[n]$, i.e., $\cup_{k=1}^{K_\omega}Z_{k}^S = [n]$ and $Z_{k_1}^S\cap Z_{k_2}^S = \emptyset$ for any $k_1\neq k_2$.} For each scope \( Z_k^S \subseteq [n] \), there is an associated scope set $\mathcal{S}[Z_k^S]$.
    \item \( \{Z_k^P\mid k \in [K_\omega]\} \) is a set of scopes with respect to the joint state-action space $\mathcal{X}:=\mathcal{S}\times \mathcal{A}=\prod_{i=1}^n\mathcal{S}_i\prod_{j=1}^m\mathcal{A}_i$. For example, consider an MDP with a $3$-dimensional state space and a $2$-dimensional action space. If $Z_k^P = \{1,2,4\}$, then its associated scope set $\mathcal{X}[Z_k^P] = \mathcal{S}_1\times \mathcal{S}_2\times\mathcal{A}_1$.
\end{itemize}

For simplicity of notation, we will use $-Z_k^S$ to denote the complement of $Z_k^S$ with respect to $[n]$, i.e., $-Z_k^S:=[n]\setminus Z_k^S$. The factorization scheme $\omega^P$ is used to characterize the component-wise dependence structure within the transition kernel. To motivate our approximate factorization scheme, we first consider the special case where the transition kernel $P$ is \emph{perfectly factorizable} with respect to $\omega^P$. That is,
\begin{align}\label{eq:marginal}
\forall x=(s,a)\in \mathcal{X}, s'\in \mathcal{S},k \in[K_\omega]: \quad P(s' \mid x) = \prod_{k=1}^{K_\omega}P_k(s'[Z_k^S] \mid x[Z_k^P]),
\end{align}
where $P_k(\cdot \mid x[Z_k^P])$ is a probability distribution for all $x[Z_k^P]\in \mathcal{X}[Z_k^P]$.
Eq. \eqref{eq:marginal} suggests that the transitions of $s'[Z_k^S]$ for different $k$ are independent and depend only on the substate-subaction pairs in \( \mathcal{X}[Z_k^P] \).
An MDP that permits both perfect transition and reward decompositions (introduced later in Section \ref{Approximate Factorization of the Reward Function}) is called a factored MDP \citep{osband2014near}. Intuitively, this means that the original MDP can be broken down into $K_\omega$ smaller components with transitions and rewards depending only on subsets of variables, rather than the full state and action space, thereby reducing the complexity of finding an optimal policy for the MDP.

As illustrated in Section \ref{sec:intro}, while the factored MDP framework is a promising approach to address the curse of dimensionality challenge, requiring the underlying MDP to be perfectly factorizable is a strong assumption. 
% In addition, there are no provable model-free algorithms yet. 
To remove this limitation, in this work, we propose a more general scheme—approximate factorization—that allows the transition kernel to be imperfectly factorizable, meaning Eq. (\ref{eq:marginal}) does not hold exactly. Specifically, to develop an approximate factorization scheme (i.e., to get an estimate of the ``marginal'' transition probabilities $P(s'[Z_k^S] \mid x[Z_k^P])$ in Eq. (\ref{eq:marginal})), we begin by defining the set of feasible marginal transition probabilities in the following. For any $k\in [K_\omega]$, let 
\begin{align*}
    \mathcal{P}_k=\,&\bigg\{P_k\in \mathbb{R}^{ |\mathcal{X}[Z_k^P]|\times |\mathcal{S}[Z_k^S]|}\,\bigg|\,\exists\, c\in\Sigma^{|\mathcal{X}[-Z_k^P]|}\text{ such that }P_k(s'[Z_k^S]\mid x[Z_k^P]) \nonumber\\
    &\quad =\sum_{x'[-Z_k^P]}c[x'[-Z_k^P]] P(s'[Z_k^S] \mid x[Z_k^P],x'[- Z_k^P]),\,\forall \,s'[Z_k^S]\in \mathcal{S}[Z_k^S],x[Z_k^P]\in \mathcal{X}[Z_k^P]\bigg\},
\end{align*}
where we use $\Sigma^d$ denotes the $d$-dimensional probability simplex for any $d\geq 0$.
Then, we choose $P_k\in\mathcal{P}_k$ arbitrarily for all $k\in [K_\omega]$. As we will see in later sections, the development of our theoretical results does not rely on the specific choice of $P_k\in\mathcal{P}_k$. To understand the definition of $\mathcal{P}_k$, again consider the perfectly factorizable setting, in which case $P(s'[Z_k^S] \mid x[Z_k^P],x'[- Z_k^P])=P(s'[Z_k^S] \mid x[Z_k^P],x''[- Z_k^P])$ for any $x'[- Z_k^P],x''[- Z_k^P]\in\mathcal{X}[- Z_k^P]$ because the transition of $s'[Z_k^S]$ does not depend on the substate-subaction pairs in $\mathcal{X}[- Z_k^P]$. As a result, $\mathcal{P}_k$ is a singleton set for all $k$, which is consistent with Eq. (\ref{eq:marginal}). Coming back to the case where the underlying MDP is not perfectly factorizable, choosing $P_k$ from $\mathcal{P}_k$ can be viewed as a way of approximating the ``marginal'' transition probabilities.

Given an approximation scheme $\omega^P$, we define the \emph{approximation error} with respect to the transition kernel as
\begin{align}
    \Delta^P_\omega = \sup_{P_1 \in \mathcal{P}_1, \dots, P_{K_\omega} \in \mathcal{P}_{K_\omega}}\max_{s' \in \mathcal{S},x \in \mathcal{X}} \left|\, P(s' \mid x) - \prod_{k=1}^{K_\omega} P_k(s'[Z_k^S] \mid x[Z_k^P]) \,\right|. \label{error_def}
\end{align}
To understand $\Delta_\omega^P$, observe that for any $s'\in\mathcal{S}$ and $x\in\mathcal{X}$, we have
    \begin{align*}
        &P(s' \mid x) -\prod_{k=1}^{K_\omega} P_k(s'[Z_k^S] \mid x[Z_k^P]) 
        \\
        =& \underbrace{P(s' \mid x) - \prod_{k=1}^{K_\omega} P_k(s'[Z_k^S] \mid x)}_{\Delta_1: \text{error from correlated transitions}} +\underbrace{\prod_{k=1}^{K_\omega} P_k(s'[Z_k^S] \mid x) - \prod_{k=1}^{K_\omega} P_k(s'[Z_k^S] \mid x[Z_k^P])}_{\Delta_2: \text{error from incomplete dependence}}.
    \end{align*}
    The first term \( \Delta_1 \) on the right-hand side represents the error due to ignoring the correlation between the transitions of $s'[Z_k^S]$ for different $k$, and vanishes if the transitions of substates from different scope sets are independent.
    The second term \( \Delta_2 \) on the right-hand side arises due to ignoring the dependence of the transitions of $s'[Z_k^S]$ from $x[-Z_k^P]$. Note that when the MDP is perfectly factorizable, we have $\Delta_\omega^P=0$ (cf. Eq. (\ref{eq:marginal})).

\subsection{Approximate Factorization of the Reward Function}
\label{Approximate Factorization of the Reward Function}
Moving to the reward function, the factorization is characterized by the tuple $\omega^R = \big\{\ell_\omega, \{Z_i^R\in[n+m]\mid i \in [\ell_\omega]\}\big\}$,
where
\begin{itemize}
    \item \( \ell_\omega \) is a 
    positive integer representing the number of components into which we decompose the reward function.
    \item \( \{Z_i^R\mid i \in [\ell_\omega]\} \) is a set of scopes with respect to the joint state-action space $\mathcal{X}$.
\end{itemize}
Under this factorization scheme, the global reward function \( r(x) \) is approximated by the sum of ``local'' reward functions from $\{r_i:\mathcal{X}[Z_i^R]\xrightarrow[]{}[0,1]\}_{i\in [\ell_\omega]}$, where each $r_i$ depends only on the scope variable \( x[Z_i^R]\).
Similarly, given an approximation scheme $\omega^R$ for the reward function, we define its approximation error as
\begin{align*}
    \Delta^R_\omega = \max_{x \in \mathcal{X}} \left| r(x) - \sum_{i=1}^{\ell_\omega} r_i(x[Z_i^R]) \right|,
\end{align*}
which is a measurement of the deviation from our approximation to the true reward function. The approximation error \( \Delta^R_\omega \) arises due to the fact that the reward function \( r(x) \) may depend on substate-subaction pairs outside the scope set $\mathcal{X}[Z_i^R]$. In the special case where the underlying MDP can be perfectly factorized into several small MDPs and the sum of their reward functions is equal to the original reward function, we have $\Delta_\omega^R=0$.

\subsection{An Illustrative Example}
\label{An Illustrative Example}
To provide intuition, we present an example of applying the approximate factorization scheme in a real-world application, demonstrating that it offers additional opportunities to achieve a better balance between desired solution accuracy and sample efficiency. In particular, we consider the storage control problem in wind farms. The storage controller (the agent) aims to align the real-time wind power generation output (the state) with the desired prediction values by flexibly charging or discharging the energy storage system.

This problem can be modeled as an MDP, where the state at time \(t\), \(s_t = (w_t, p_t, c_t)\), captures the wind power generation \(w_t\), electricity price \(p_t\), and the state of charge (SoC) of the storage \(c_t\). The action \(a_t\) represents the storage charging decision at time \(t\).
The transition dependence structure can be represented as a bipartite graph, as illustrated in Figure \ref{Bipartite Graph Representation}. In the graph, nodes on the left-hand side (LHS) represent the state and action at time \(t\), while nodes on the right-hand side (RHS) represent the state at time \(t+1\). A solid blue line between nodes indicates a strong dependence. For instance, the electricity price \(p_{t+1}\) strongly depends on the previous price \(p_t\). While a dashed blue line represents weak dependence. For example, the wind power generation $w_t$ can influence the next step's electricity price $p_{t+1}$ in the market, while its impact is weaker compared to the direct influence of the previous price $p_t$.

Figure \ref{Approximate Factorization K=3} illustrates an approximate factorization scheme on this MDP, where the system's dynamics are divided into three smaller components --- the dynamics of the wind power generation, electricity price, and storage levels. Strong transition dependencies are preserved within each component, but weaker dependencies, such as the influence of wind power generation on electricity price, are disregarded.
This leads to the following approximation of the transition probability:
\[
\widehat{P}(s_{t+1} \mid s_t, a_t) = P(w_{t+1} \mid w_t) \cdot P(p_{t+1} \mid p_t) \cdot P(c_{t+1} \mid c_t, a_t).
\]

This approximation simplifies the model by factoring the transition dynamics into smaller, more manageable components while retaining key dependencies. Note that this problem does not fall under a factored MDP, as the dynamics cannot be perfectly broken down into smaller components. However, our approximate factorization scheme offers a more flexible framework, allowing imperfectly factorizable transition dynamics to be divided into smaller components while maintaining sufficient model accuracy. This expands the factorization step's action space, enabling a more effective search for optimal trade-offs between solution accuracy and sample efficiency, achieved through tailored to the specific dependency structures of the problems. 

\begin{figure}[htbp]
    \centering
    \subfigure[Bipartite Graph Representation]
    {
        \includegraphics[width=0.38\linewidth]{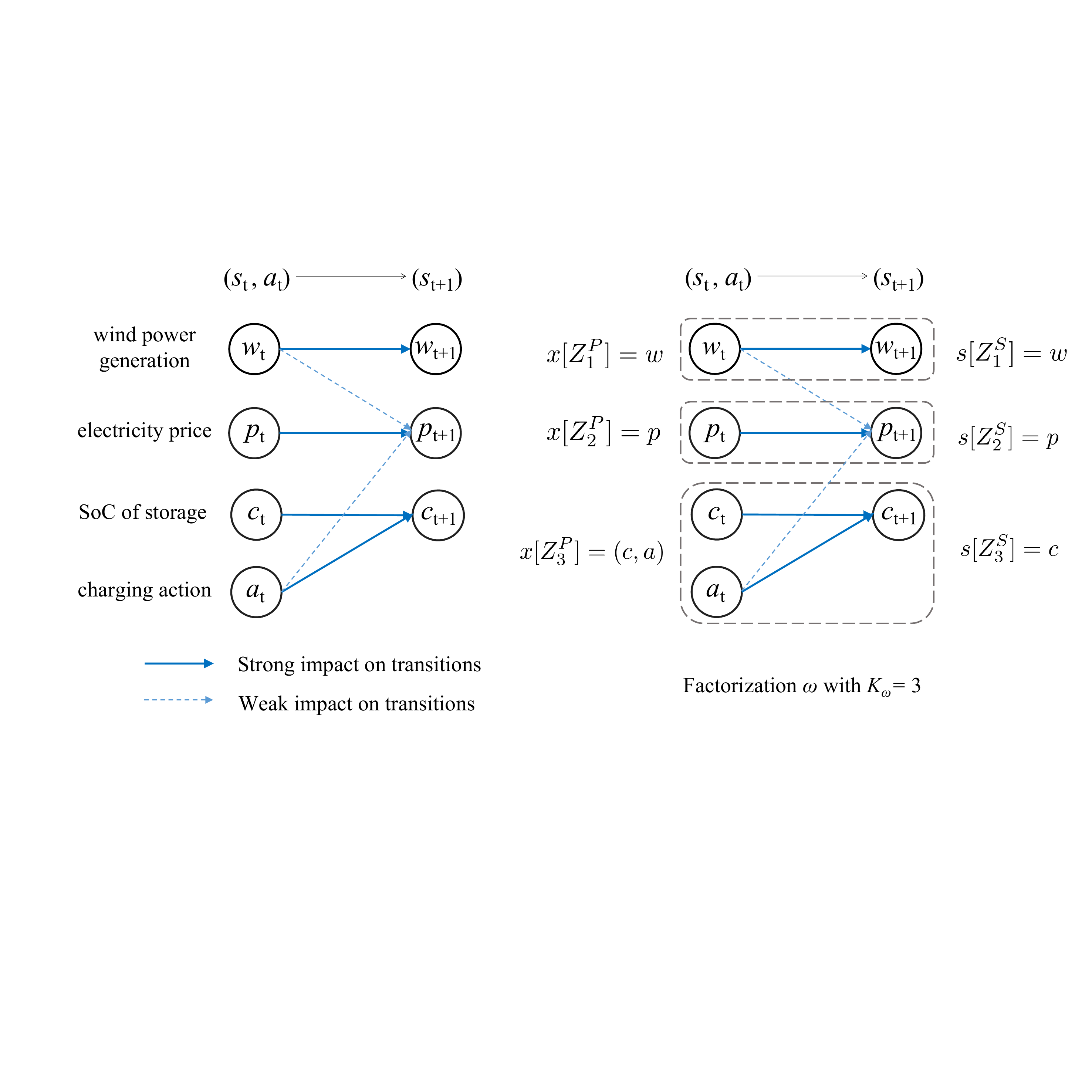}
        \label{Bipartite Graph Representation}
    }
    \subfigure[Approximate Factorization Scheme]
    {
        \includegraphics[width=0.48\linewidth]{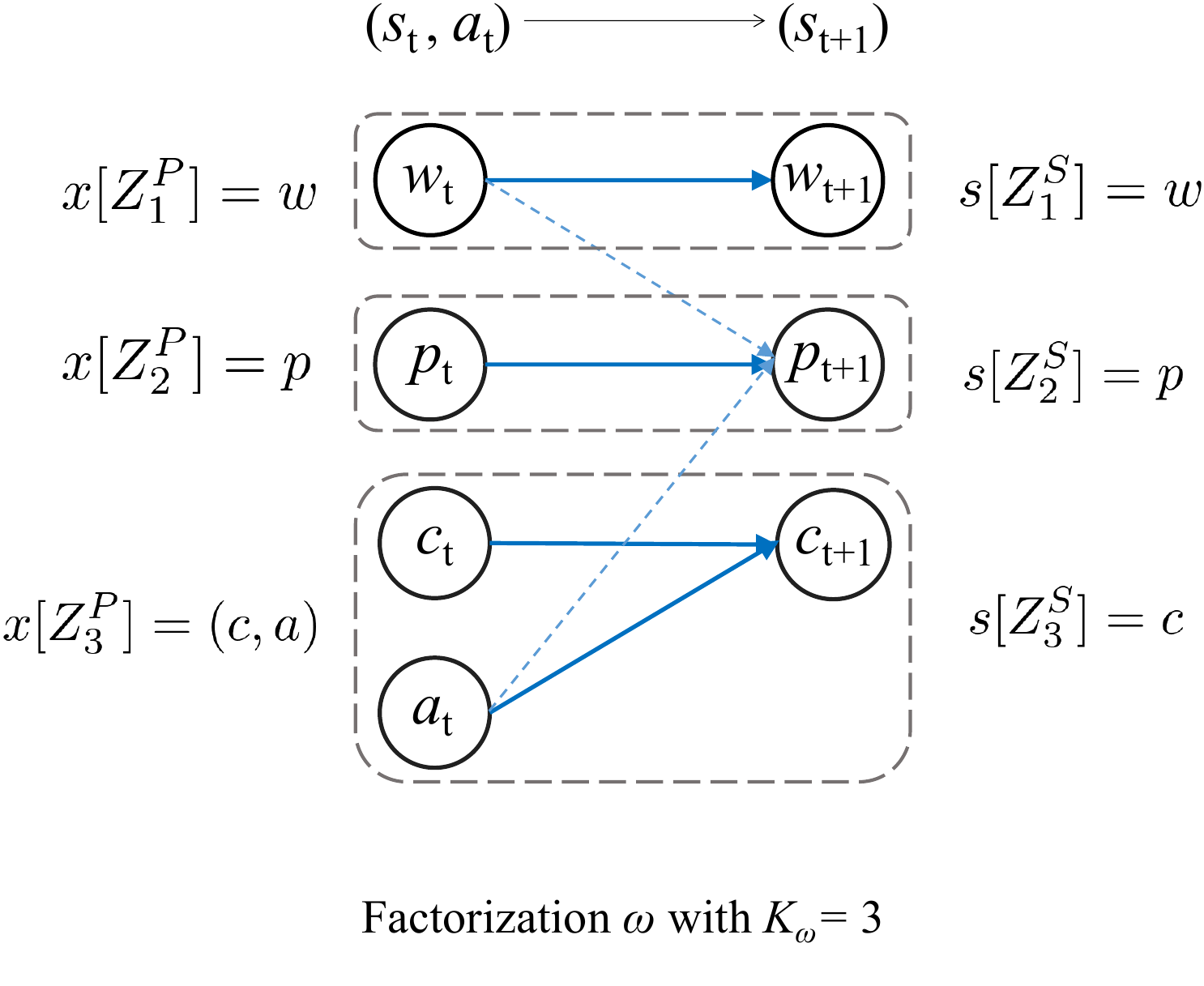}
    \label{Approximate Factorization K=3}
    }
    \caption{Bipartite Graph Representation and Approximate Factorization.}
    \label{Bipartite Graph Representation and Approximate Factorization}
    
\end{figure}

\section{Sampling through Structure: Multi-Component Factorized Synchronous Sampling}
\label{Multi-Component Factorized Synchronous Sampling}

Our main contributions propose tailored model-free and model-based approaches to solve approximately factored MDPs sample-efficiently. At the core of both of these approaches is an efficient sampling algorithm that leverages the factorization structure.  In this section, we introduce this core module and then apply it in both model-based and model-free algorithms in the sections that follow.

To illustrate the role of this sampling algorithm, consider classical model-based RL, which consists of two main steps: (1) estimating the transition kernel and the reward function through empirical sampling and (2) performing value iteration (or policy iteration) on the estimated model. The curse of dimensionality in terms of sample efficiency typically arises in the first step. Specifically, to obtain a relatively accurate estimate for the model, starting from each state-action pair, we need to independently sample the next state a sufficient amount of times, leading to a sample complexity proportional to $|\mathcal{S}||\mathcal{A}|$, which scales linearly with the size of the state-action space (hence exponentially with the dimension of the state-action space). By exploiting the approximate factorization structure, we develop a more efficient sampling algorithm. For ease of presentation, we focus on estimating the transition kernel here. The same approach applies to estimating the reward function. 

The goal is to construct an approximation $\widehat{P}$ of the original transition kernel $P$. We achieve this by constructing approximations of the low-dimensional transition probabilities \( \{\widehat{P}_k\mid k\in[K_\omega]\}\) and then compute $\widehat{P}$ according to
\begin{align}\label{eq:from_each_to_overall}
    \widehat{P}(s' \mid x) = \prod\nolimits_{k=1}^{K_\omega} \widehat{P}_k(s'[Z_k^S] \mid x[Z_k^P])
\end{align}
for all $s'\in\mathcal{S}$ and $x=(s,a)\in\mathcal{S}\times \mathcal{A}$.
In this case, instead of sampling from the global state-action space, we only need to sample from the substate-subaction spaces, which reduces the sample complexity.
To estimate $\widehat{P}_k(s'[Z_k^S] \mid x[Z_k^P])$ for any $k\in [K_\omega]$, let \( X^P_k \) be defined as $X^P_k = \{x\in\mathcal{X} \mid x[-Z_k^P]=x^{\text{default}}[-Z_k^P]\}$,
where \( x^{\text{default}}\in\mathcal{X} \) is an arbitrary (but fixed) element from $\mathcal{X}$. Sampling from \( X^P_k \) allows us to estimate the transition kernel for the \( k \)-th factorized component. Importantly, while sampling from $X_k$, we set $x[-Z_k^P]$ to $x^{\text{default}}[-Z_k^P]$ so that we do not need to cover the rest of the state-action space. For each \( x \in X^P_k \), we sample the next state \( N \) times, resulting in a set of samples \( \{s^k_{x,i}\}_{i\in [N]} \). The low-dimensional empirical transition kernel \( \widehat{P}_k \) is then computed as:
\begin{align}
    \widehat{P}_k(s'[Z_k^S] \mid x[Z_k^P]) = \frac{1}{N} \sum\nolimits_{i=1}^{N} \mathds{1}(s^k_{x,i}[Z_k^S] = s'[Z_k^S]), \quad \forall\, s'[Z_k^S]\in\mathcal{S}[Z_k^S],x\in X_k^P. \label{eq:empirical_transition}
\end{align}
The overall estimate $\widehat{P}$ of the transition kernel can then be computed using Eq. (\ref{eq:from_each_to_overall}).
Using this approach, the total number of required samples is \( \sum_{k=1}^{K_\omega} |\mathcal{X}[Z_k^P]|N \). Since the sum of the sizes of the substate-subaction space for each component, i.e., $\sum_{k=1}^{K_\omega} |\mathcal{X}[Z_k^P]|$, is usually much smaller (in particular, exponentially smaller) compared to the size of the overall state-action space $|\mathcal{S}||\mathcal{A}|$, the total number of samples required is substantially smaller than what would be needed in classical model-based RL.

\subsection{Synchronous Sampling Properties}
The sampling method described above can be further optimized to improve the sample efficiency by exploiting the structure of the factorization scheme. Previously, we used the samples \( \{ s^k_{x,i} \}_{i \in [N]} \) solely to construct the estimator \( \widehat{P}_k(s'[Z_k^S] \mid x[Z_k^P]) \) for component \( k \). However, these samples can potentially be reused to estimate other low-dimensional transition kernels, depending on the relationships between their associated scopes. To formalize this idea, we introduce two key strategies that leverage the relationships between the scope sets of the components.

\paragraph{Synchronous Sampling with Inclusive Scopes.} For any \( k_1, k_2 \in [K_\omega] \) such that \( Z_{k_1}^P \subseteq Z_{k_2}^P \), the samples used to estimate the transition kernel of component \( k_2 \) can be reused to estimate that of component \( k_1 \). Specifically, consider the sampling set \( X^P_{k_2} = \{ x \in \mathcal{X} \mid x[-Z_{k_2}^P] = x^{\text{default}}[-Z_{k_2}^P] \} \) for component \( k_2 \), where \( x^{\text{default}}[-Z_{k_2}^P] \) is a fixed arbitrary element from \( \mathcal{X}[-Z_{k_2}^P] \). By sampling from each \( x \in X^P_{k_2} \) , we obtain samples of the next state \( \{ s^{k_2}_{x,i} \}_{i \in [N]} \). Since \( Z_{k_1}^P \subseteq Z_{k_2}^P \), these samples inherently contain information about the transitions of component \( k_1 \). Therefore, we can estimate the transition probabilities for component \( k_1 \) according to Eq. (\ref{eq:empirical_transition}) with $k=k_1$ using the same samples, where the sampling set $X_{k_1}^P$ is defined based on $X_{k_1}^P$ as
\begin{align*}
    X^P_{k_1} = \{ x \in X^P_{k_2} \mid x[Z_{k_2}^P\setminus Z_{k_1}^P] = x^{\text{default}}[Z_{k_2}^P\setminus Z_{k_1}^P] \}.
\end{align*}
The reuse of samples here improves the overall sample efficiency by avoiding sampling again for component $k_1$.

\paragraph{Synchronous Sampling with Exclusive Scopes.} For any \( k_1, k_2 \in [K_\omega] \) such that their associated scopes are disjoint, i.e., \( Z_{k_1}^P \cap Z_{k_2}^P = \emptyset \), we can estimate the transitions for both components simultaneously using shared samples. 
{Specifically, define the \emph{joint sampling set} \( X^P_{k_1,k_2} \) as: \begin{align}
        X^P_{k_1,k_2} = \{(x[Z_{k_1}^P]_{(i \bmod |\mathcal{X}[Z_{k_1}^P]|+1)},x[Z_{k_2}^P]_{(i\bmod |\mathcal{X}[Z_{k_2}^P]|+1)}, x^{\text{default}}[- (Z_{k_1}^P\cup Z_{k_2}^P)]) \mid i \in [D_{\max}]\},
    \end{align}
where \( D_{\max} = \max(|\mathcal{X}[Z_{k_1}^P]|, |\mathcal{X}[Z_{k_2}^P]|) \), \( x[Z_{k_1}^P]_{(i)} \) denotes the \( i \)-th element in a fixed ordering of \(\mathcal{X}[Z_{k_1}^P] \). The modulo operation ensures that we cycle through all possible values of each component's state-action space. By sampling from each \( x \in X^P_{k_1,k_2} \) for \( N \) times, we obtain samples \( \{ s_{x,i} \}_{i=1}^{N} \), which are used to estimate the transition probabilities for both components according to Eq. (\ref{eq:empirical_transition}). Figure \ref{illu:joint sampling set} illustrates an example of synchronous sampling with exclusive scopes.
This strategy improves sample efficiency by reducing the total number of samples needed compared to independently sampling each component.

\begin{figure}[t]
\centerline{\includegraphics[width=0.98\linewidth]{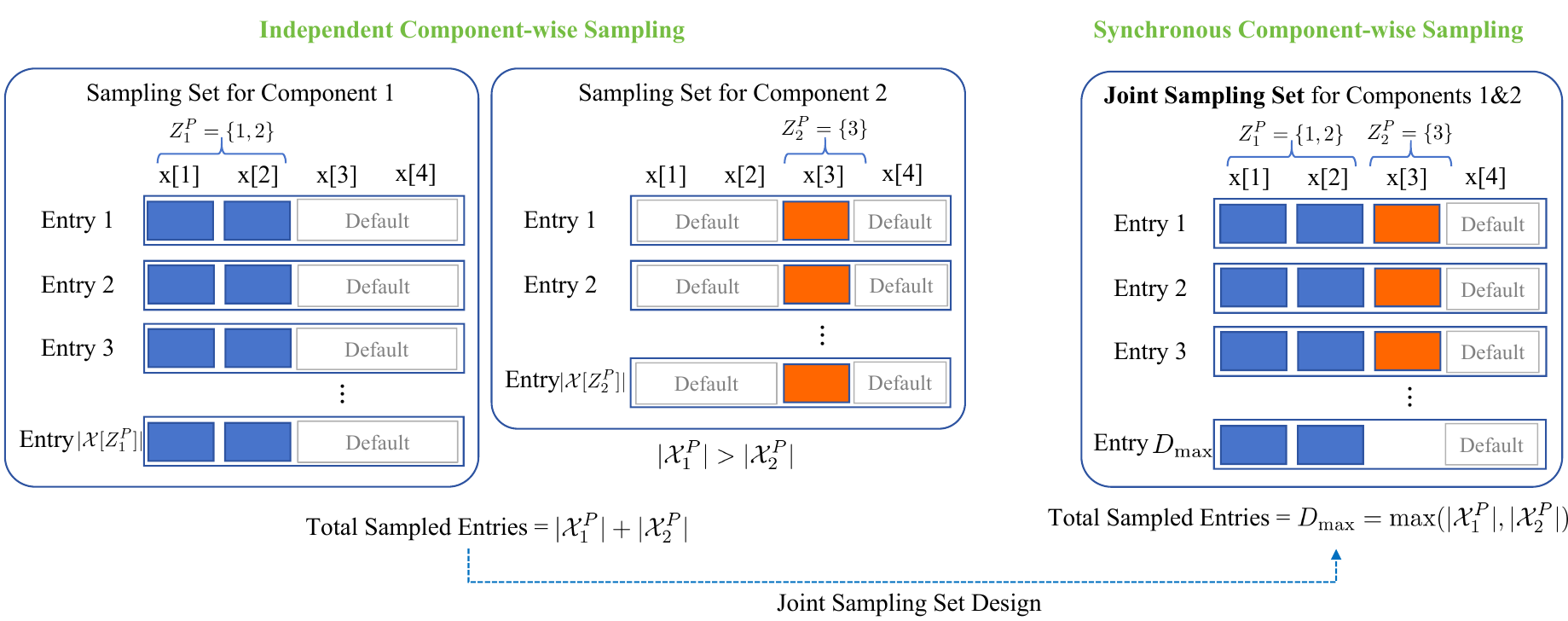}}
\caption{Synchronous Sampling with Exclusive Scopes.}
\label{illu:joint sampling set}

\end{figure}

\subsection{Cost-Optimal Synchronous Sampling Strategy}
\label{Cost-Optimal Synchronous Sampling Strategy}
Building on the aforementioned two key strategies, we next present our sampling approach. 
Recall that we have \( K_\omega \) components associated with scope sets \( Z_1^P, Z_2^P, \dotsc, Z_{K_\omega}^P \). For each $k\in [K_\omega]$, we define its corresponding sampling cost as the size of the scope set $\mathcal{X}[Z_k^P]$. Using the inclusive scope property, we first eliminate components whose scope sets are subsets of others. Let \( \mathcal{K}_\omega^* \) denote the set of remaining components: 
\begin{align*}
    \mathcal{K}_\omega^* = \{ k \in [K_\omega] \mid \nexists\, k_1 \neq k \ \text{such that} \ Z_{k}^P \subsetneq  Z_{k_1}^P \}.
\end{align*}
Next, to use the exclusive scope property, we divide \( \mathcal{K}_\omega^* \) into subsets $\{\mathcal{G}_{i}\}_{1\leq i\leq \kappa_p}$ (where $\kappa_p$ denotes the total number of subsets)
such that within each subset \( \mathcal{G}_i \), the components have disjoint scopes, i.e., \( Z_{k_1}^P \cap Z_{k_2}^P = \emptyset \) for all \( k_1,k_2\in \mathcal{G}_i \). This enables us to apply the strategy of synchronous sampling with exclusive scopes within each subset. Specifically, for each subset \( \mathcal{G}_k \) of $\mathcal{K}_\omega^*$, we construct the joint sampling sets as follows:
\begin{align*}
    X_{\mathcal{G}_k}^P = \left\{ \left\{ x[Z_j^P]_{(i \bmod |\mathcal{X}[Z_j^P]|+1)} \right\}_{j \in \mathcal{G}_k}, \ x^{\text{default}}[-(\cup_{j \in \mathcal{G}_k} Z_j^P)] \ \Big| \ i \in [D_{\max,k}] \right\},
\end{align*}
where \( D_{\max,k} = \max_{j \in \mathcal{G}_k} |\mathcal{X}[Z_j^P]| \) ensures that all scope sets associated with this subset are sampled. Therefore, the cost of sampling a joint sampling set is determined by the largest scope sets in the group:
\begin{align*}
    c_k(\mathcal{G}_k) =  D_{\max,k} = \max_{j \in \mathcal{G}_k} |\mathcal{X}[Z_j^P]|.
\end{align*}

Our objective is to minimize the total sampling cost across all groups:
\begin{align*}
    \min_{\kappa_p} \min_{\mathcal{G}_1, ..., \mathcal{G}_{\kappa_p}}\sum\nolimits_{k=1}^{\kappa_p} c_k(\mathcal{G}_k).
\end{align*}

This problem is closely related to the \emph{Graph Coloring Problem} (GCP) \citep{jensen2011graph,karp2010reducibility}. In Appendix \ref{Discussion_Cost-Optimal Sampling Problem}, we elaborate on the connection between our cost-optimal sampling problem and the GCP, and provide a proof of its computational complexity. Due to such connection, a rich array of well-established algorithms are available to tackle the problem effectively in practice. The solution yields the optimal number of groups \( \kappa_p\in[K_\omega] \) and the division scheme \( \{ \mathcal{G}_k \}_{k \in [\kappa_p]} \). Notably, \( \kappa_p \) corresponds to the \emph{chromatic number} \citep{erdHos1966chromatic} (the minimal number of colors required) in GCP. Existing literature shows that, under mild conditions, the chromatic number is often much smaller than the number of nodes in the graph \citep{khot2001improved,sopena1997chromatic,goddard2012s}. This result translates to \( \kappa_p \ll K_\omega \) in our context, leading to substantial sampling cost reductions. A more in-depth discussion on $\kappa_p$ can refer to Section \ref{model-based sample complexity}.

\subsection{MDP Model Estimation based on Synchronous Sampling }
Given \( \{\mathcal{G}_k\}_{k\in [\kappa_p]} \) and the joint sampling sets $\{X_{\mathcal{G}_k}^P\}_{k\in [\kappa_p]}$, we are ready to sample and estimate the transition kernel. Specifically, for each state \( x \in X_{\mathcal{G}_k}^P \), we perform \( N \) sampling trials to obtain next-state samples based on the generative model. Specifically, we generate \( N \) samples \( \{ s^k_{x,i} \}_{i=1}^{N} \) by sampling from \( P(\cdot \mid x) \) as detailed in Algorithm \ref{Cost-Optimal Factorized Synchronous Sampling Algorithm}. For each component \( j \in \mathcal{G}_k \), we compute the empirical transition kernel by:
\begin{align}
    \widehat{P}_j(s'[Z_j^S] \mid x[Z_j^P]) = \frac{1}{N} \sum\nolimits_{i\in[N]} \mathds{1}\big( s^k_{x,i}[Z_j^S] = s'[Z_j^S] \big), \quad \forall\, s'[Z_j^S]\in \mathcal{S}[Z_j^S].
\end{align}

\begin{algorithm}[h]
\caption{Cost-Optimal Factorized Synchronous Sampling Algorithm}
\label{Cost-Optimal Factorized Synchronous Sampling Algorithm}
\begin{algorithmic}[1]
\STATE \textbf{Input:} Approximate Factorization Scheme $\omega = (\omega^P,\omega^R)$.
\STATE {Solve the optimal division scheme represented by $\{\kappa_p, \{\mathcal{G}_k\}_{k\in [\kappa_p]}\}$;}
\STATE Construct the sampling sets $\{X^P_{\mathcal{G}_k}\}_{k\in[\kappa_p]}$;
\FOR{$k$ = $1,2,\cdots,\kappa_p$}
\STATE Sample from each state-action pair $x\in X_{\mathcal{G}_k}^P$ for $N$ times and get the samples by $\{s^k_{x,i}\}_{x\in X^P_{\mathcal{G}_k}, i\in[N]}$.
\ENDFOR
\STATE \textbf{Output:} Samples $\{s^k_{x,i}\}_{k\in [\kappa_p], x\in X^P_{\mathcal{G}_k}, i\in[N]}$.
\end{algorithmic}
\end{algorithm}
To estimate the reward function components, we can employ similar strategies, specifically the inclusive and exclusive scope strategies, to effectively sample the reward function. By constructing a sampling set tailored to these strategies, we can then proceed with learning the reward function. This approach helps reduce the overall sample complexity of learning the reward function components.

\section{Model-Based Reinforcement Learning with Approximate Factorization}
\label{Sec_model-based}
In this section, focusing on MDPs with approximate factorization and armed with the synchronous sampling approach introduced in the previous section, we propose the \emph{model-based Q-value iteration with approximate factorization} algorithm. This algorithm offers sample complexity guarantees, improving upon prior work.

\subsection{Algorithm Design}
The complete algorithm is summarized in Algorithm \ref{alg:Value Iteration (VI)}, which consists of two primary steps: estimating the model parameters and performing Q-value iteration on the estimated model.

\begin{algorithm}[h]
\caption{Model-Based Q-Value Iteration with Approximate Factorization}
\label{alg:Value Iteration (VI)}
\begin{algorithmic}[1]
\STATE \textbf{Input:} Positive integer $T$ and initialization $\widehat{Q}_0(s,a)=0$ for all $(s,a)\in\mathcal{S}\times \mathcal{A}$.
\STATE Compute empirical transition kernel $\widehat{P}$ and the reward function $\widehat{r}$ through synchronous sampling Algorithm \ref{Cost-Optimal Factorized Synchronous Sampling Algorithm} and Eqs. \eqref{kernel_combine} and \eqref{reward_combine}.
\FOR{$t=1,2,\cdots,T$}
\STATE $\widehat{Q}_{t}(s,a) = {r}(s, a) + \gamma \sum_{s'} \left[\widehat{P}(s'|s,a) \max_{a'} \widehat{Q}_{t-1}(s',a')\right]$ for all $(s,a)\in\mathcal{S}\times \mathcal{A}$.
\ENDFOR
\STATE \textbf{Output:} Estimated Q-value function $\widehat{Q}^*_{\omega} = \widehat{Q}_T$.
\end{algorithmic}
\end{algorithm}

The first step involves estimating the transition kernel and the reward function (see Algorithm \ref{alg:Value Iteration (VI)}, Line 2). Specifically, in Section \ref{Multi-Component Factorized Synchronous Sampling}, we explained how to estimate the transition probabilities \( \widehat{P}_k \) for each component \( k \in [K_\omega] \). For any \( s' \in \mathcal{S} \) and \( x \in \mathcal{X} \), the overall transition probability \( \widehat{P}(s' \mid x) \) is computed by combining the individual component estimates as follows:
\begin{align}
    \widehat{P}(s' \mid x) = \prod\nolimits_{k=1}^{K_\omega} \widehat{P}_k(s'[Z_k^S] \mid x[Z_k^P]). \label{kernel_combine}
\end{align}
Similarly, for the reward function, we aggregate the estimated rewards \( \widehat{r}_i \) for each component \( i \in [\ell_\omega] \) to obtain the overall reward function:
\begin{align}
     \widehat{r}(x) = \sum\nolimits_{i=1}^{\ell_\omega} \widehat{r}_i(x[Z_i^R]).
     \label{reward_combine}
\end{align}

The second step is to apply the value iteration method using these estimated model parameters (see Algorithm \ref{alg:Value Iteration (VI)}, Lines 3–5). In particular, Line 4 of Algorithm \ref{alg:Value Iteration (VI)} represents the empirical version of the Bellman iteration. Through value iteration, we can compute the desired optimal-\( Q \)-value function and derive the corresponding greedy policy as the final solution to the RL problem.

\subsection{Sample Complexity Guarantees}
\label{model-based sample complexity}
We now present the sample complexity guarantees of Algorithm \ref{alg:Value Iteration (VI)}; the proof can be found in Appendix~\ref{proof:them:summary}. Before proceeding, without loss of generalit, we assume that \( |\mathcal{X}[Z_1^P]| \geq |\mathcal{X}[Z_2^P]| \geq \dots \geq |\mathcal{X}[Z_{K_\omega}^P]| \) and \( |\mathcal{X}[Z_1^R]| \geq |\mathcal{X}[Z_2^R]| \geq \dots \geq |\mathcal{X}[Z_{\ell_\omega}^R]| \), i.e., the component scope sets are ordered in descending order based on their cardinality.

\begin{theorem}\label{thm:summary}
Given any approximate factorization scheme \( \omega \), let \( \mathcal{E}_\omega = \gamma(1-\gamma)^{-2} \Delta^P_\omega + (1-\gamma)^{-1} \Delta^R_\omega \). For any confidence level \( \delta > 0 \) and the desired  accuracy level \( \epsilon \in (0,1) \), with probability at least \( 1 - \delta \), {the output Q-function $\widehat{Q}_\omega^* $ from Algorithm \ref{alg:Value Iteration (VI)} after $T\geq\overline{c}_2\log(\epsilon^{-1}(1-\gamma)^{-1})$ iterations satisfies
\begin{align}
    \| \widehat{Q}_\omega^* -Q^* ||_\infty \leq \epsilon + \mathcal{E}_\omega,
\end{align}}
provided that the total number of samples, denoted by \( D_\omega \), satisfies:
\begin{align} 
    D_\omega \geq \frac{\overline{c}_0 \left( \sum_{k \in [\kappa_p]} |\mathcal{X}[Z^P_k]| \right) \log \left( \overline{c}_1 |\mathcal{X}[\cup_{k \in [K_\omega]} Z_k^P]| \delta^{-1} \right)}{\epsilon^2 (1-\gamma)^3} +  \sum\nolimits_{i \in [\kappa_r]} |\mathcal{X}[Z^R_i]|,\label{thm1_sample_complexity}
\end{align}
where $\kappa_p \in[0, K_\omega]$ and $ \kappa_r \in[0, \ell_\omega]$ are problem-dependent parameters, and \( \overline{c}_0 \), \( \overline{c}_1 \) \( \overline{c}_2 \) are absolute constants.
\end{theorem}

Now we are in position to discuss further implications of our results.

\paragraph{Model Misspecification Bias.}  The parameter \( \mathcal{E}_\omega \) is called the model misspecification bias, which consists of two terms: \( \gamma(1-\gamma)^{-2} \Delta^P_\omega \) and \( (1-\gamma)^{-1} \Delta^R_\omega \), both of which are linearly dependent on the approximation errors of the transition kernel and reward function, respectively. This bias arises from the inaccuracies introduced by the factorization of the transition kernel and reward function. If the factorization were accurate, these approximation errors would vanish, and \( \mathcal{E}_\omega \) would reduce to zero, as in the factored MDP case, which is discussed below.

\paragraph{Sample Complexity Implications and Comparison to Prior Works}\label{para:sample-complexity-discussion}
For a fair comparison, we consider our results using a factorization scheme $\omega$ that is perfect with the bias \( \mathcal{E}_\omega = 0\). In this case, our sample complexity required to achieve an 
$\varepsilon$-optimal policy is of the following order:
\begin{align}
    \widetilde{O}\left( \frac{\sum\nolimits_{k \in [\kappa_p]} |\mathcal{X}[Z^P_k]|}{\epsilon^2(1-\gamma)^3} + \sum\nolimits_{i \in [\kappa_r]} |\mathcal{X}[Z^R_i]| \right). \label{eq:model-based-results}
\end{align}

Compared with the minimax-optimal sample complexity \citep{azar2012sample} for solving standard MDPs:
\begin{align}
    \widetilde{O}\left( \frac{|\mathcal{S}| |\mathcal{A}|}{\varepsilon^2 (1-\gamma)^3}\right),
\end{align}
we have identical dependence on \( \delta \), \( \epsilon \), and the effective horizon \( 1/(1-\gamma) \). The key improvement lies in the dependence on the size of the state-action space, where we improve the minimax \(  \widetilde{\mathcal{O}}(|\mathcal{S}| |\mathcal{A}|) \) dependence \citep{azar2012sample} to \( \widetilde{\mathcal{O}}(  \sum\nolimits_{k \in [\kappa_p]} |\mathcal{X}[Z^P_k]|) \leq \widetilde{\mathcal{O}}(  K_\omega \max_k |\mathcal{X}[Z^P_k]|)\). Notably, \(\widetilde{\mathcal{O}}(  K_\omega \max_k |\mathcal{X}[Z^P_k]|)\) is a problem-dependent sample complexity --- almost proportional to the sample complexity of solving the largest individual component among the factorized parts of the entire transition kernel, which is exponentially smaller than \( |\mathcal{S}| |\mathcal{A}| \) due to reduced dimensionality. For example, consider an approximate factoirzation scheme that decomposes an MDP into $10$ disjoint components with identical cardinalities. The corresponding sample complexity becomes $\widetilde{\mathcal{O}}\big([|\mathcal{S}||\mathcal{A}|]^{\frac{1}{10}}\big)$, achieving an exponential reduction in sample complexity relative to the number of factorized components.

In addition, when the bias \( \mathcal{E}_\omega = 0\), the MDPs with approximate factorization reduce to the well-known setting of factored MDPs.
Compared to the best-known sample complexity upper bound \(\widetilde{\mathcal{O}}\left(\sum_{k=1}^{K_\omega} |\mathcal{X}[Z^P_k]| \epsilon^{-2} (1 - \gamma)^{-3} + \sum_{i=1}^{\ell_\omega} |\mathcal{X}[Z^R_i]|\right)\)  for factored MDPs from \citep{chen2020efficient} \footnote{We translate their results to our setting for easy understanding and clear comparison.}, our result \(\widetilde{\mathcal{O}}\left(\sum_{k=1}^{\kappa_p} |\mathcal{X}[Z^P_k]| \epsilon^{-2} (1 - \gamma)^{-3}+\sum\nolimits_{i=1}^{\kappa^r} |\mathcal{X}[Z^R_i]|\right)\) offers strictly better sample complexity, where \(\kappa_p\in[K_\omega]\) and  \(\kappa_r\in[\ell_\omega]\) are instance-dependent parameters. Indeed, \(\kappa_p\) reflects the sparsity of the dependence structure in the MDP. In many real-world applications such as UAV swarm control \citep{campion2018uav} and power system economic dispatch \citep{chen2022reinforcement}, the dynamics of different UAVs or different power generators are highly independent, then \(\kappa_p \ll K_\omega\) and can be small constant if the dependence structure is sparse. Also, by connecting it to the \emph{Graph Coloring Problem}, $\kappa_p\ll K_\omega$ is provable under mild conditions (can refer to the discussion in Section \ref{Cost-Optimal Synchronous Sampling Strategy}). As a result, it indicates our sample complexity can improve the prior arts in factored MDPs by a factor of up to \(\mathcal{O}(K_\omega)\). Similarly, $\kappa_r$ can be significantly smaller than $\ell_\omega$ if the reward function components have such sparse dependence.

The sample complexity lower bound for factored MDPs is \(\widetilde{\mathcal{O}}(\max_{k} |\mathcal{X}[Z^P_k]| \epsilon^{-2} (1 - \gamma)^{-3}+\max_i|\mathcal{X}[Z^R_i]|)\), which is established in \citep{xu2020reinforcement, chen2020efficient}. It worth noting that, our algorithm is the first to match this lower bound in an instance-dependent manner when \(\kappa_p, \kappa_r = \mathcal{O}(1)\), demonstrating that we not only improve upon existing upper bounds but also achieve the theoretical minimum sample complexity under certain conditions.

\paragraph{Trade-Off Between Sample Complexity and Model Misspecification Bias.} There exists a trade-off between the sample complexity \( D_\omega \) and the model misspecification bias \( \mathcal{E}_\omega \). The approximate factorization scheme \( \omega \) can be viewed as a tunable hyperparameter. In general, when the number of components $K_\omega$ increases, we get a finer decomposition of $P$ and the size of each component decays exponentially with $K_\omega$, which significantly reduces the sample complexity $D_\omega$. However, the drawback of increasing $K_\omega$ arbitrarily is that it may result in a larger misspecification bias \( \mathcal{E}_\omega \) due to model mismatch.
To illustrate this trade-off, consider the example depicted in Figure \ref{Bipartite Graph Representation and Approximate Factorization}, where we decompose an MDP into three components by disregarding certain weak dependencies. If the induced error \( \mathcal{E}_\omega \) is smaller than the desired accuracy level \( \epsilon \), this imperfect factorization is sufficient to achieve the target solution. More importantly, such appropximate factorization improves sample efficiency by reducing the problem's dimensionality. While when we require higher accuracy (smaller $\varepsilon$), a more careful factorization scheme with generally fewer components should be chosen. For instance, grouping \( w_t \) and \( p_t \) within a single component will reduce the bias \( \mathcal{E}_\omega \), though at the cost of increased sample requirements. This trade-off will be revisited again in our numerical simulations in Section \ref{Synthetic MDP Tasks}. 
As a final note, when $K_\omega =1$, this theorem recovers the classical sample complexity result (without decomposition) established in \citep{azar2012sample}.

\section{Model-Free Reinforcement Learning with Approximate Factorization}
\label{Sec_model-free}
In this section, we focus on designing model-free RL algorithms to solve MDPs with approximate factorization, offering a more flexible paradigm than model-based RL that eliminates the need for explicitly estimating the environment model (i.e., transitions and reward functions). Towards this, we introduce the \emph{Variance-Reduced Q-Learning with Approximate Factorization} (VRQL-AF) algorithm, the first provably near-optimal model-free approach for both our problem and factored MDPs to the best of knowledge.

\subsection{Warn-Up: Classical Q-Learning}
Our proposed algorithm is a variant of Q-learning \citep{watkins1992q}, one of the most well-known model-free algorithms. With access to a generative model, classical Q-learning (also known as synchronous Q-learning \citep{watkins1992q,bertsekas1996neuro}) learns by approximating the optimal Q-function through stochastic approximation. Specifically, at each iteration \( t \), given the current estimate \( \widehat{Q}_t \), Q-learning samples each state-action pair \( (s, a) \) once,  obtaining the corresponding reward $r(s,a)$ and a random next state $s'$ according to the unknown transition dynamics. Using these samples, the empirical Bellman operator at time step $t$ is constructed as
\begin{align*}
    [\widehat{\mathcal{H}}_t(Q)](s, a) =    r(s,a) + \gamma \max_{a'} Q(s', a'),\quad \forall\,(s,a)\in\mathcal{S}\times \mathcal{A}, Q\in \mathbb{R}^{|\mathcal{S}||\mathcal{A}|}.
\end{align*}
The Q-function estimate $\hat{Q}_t$ is then updated using a small-stepsize version of the empirical Q-value iteration:
\begin{align*}
    \widehat{Q}_{t+1} = \widehat{Q}_{t} +\eta_t(\widehat{\mathcal{H}}_t(\widehat{Q}_{t})-\widehat{Q}_{t}),
\end{align*}
where \( \eta_t \in (0, 1) \) is the learning rate. Such an algorithm can also be viewed as a stochastic approximation algorithm for solving the Bellman optimality equation \citep{bertsekas1996neuro,tsitsiklis1994asynchronous}.

This algorithm updates the Q-value \( \widehat{Q}_t(s, a) \) by incorporating information from the empirical Bellman operator $\widehat{H}_t(\cdot)$, which estimates the expected return for each action based on sampled next states and rewards. Over time, with appropriately chosen learning rates and a sufficient number of iterations, the Q-function converges asymptotically to the optimal Q-function \( Q^* \) \citep{tsitsiklis1994asynchronous}.

\subsection{Algorithm Design}
To address the challenge of low sample efficiency, we propose \emph{Variance-Reduced Q-Learning with Approximate Factorization} (VRQL-AF), which incorporates two key modifications over classical Q-learning: (1) a factored empirical Bellman operator design and (2) a variance-reduced Q-iteration.

The factored Bellman operator is designed to replace the standard empirical Bellman operator in order to enhance sampling efficiency by leveraging the structure. Its construction is outlined in Algorithm \ref{Empirical Bellman Operator Generation}. Specifically, in vanilla Q-learning, we need to sample all state-action pairs to generate a single empirical Bellman operator, which requires $|\mathcal{S}||\mathcal{A}|$ samples and is highly inefficient. Our approach aims to reduce the number of required samples by exploiting the structure of the problem through a factorization scheme, thus enabling the repeated use of the same samples. 

\begin{algorithm}
\caption{Empirical Factored Bellman Operator Generation}
\label{Empirical Bellman Operator Generation} 
\begin{algorithmic}[1]
\STATE \textbf{Input:}  Factorization Scheme $\omega = (\omega^P,\omega^R)$.
\STATE {Solve the optimal Synchronous sampling scheme represented by $\{\kappa_p, \{\mathcal{G}_k\}_{k\in [\kappa_p]}\}$;}
\STATE Construct the sampling sets $\{X^P_{\mathcal{G}_k}\}_{k\in[\kappa_p]}$;
\FOR{$k$ = $1,2,\cdots,\kappa_p$}
\STATE Sample the transition from each state-action pair $x\in X_{\mathcal{G}_k}^P$ once, and obtain the next state by $\{s^k_{x}\}_{k\in [\kappa_p]}$.
\ENDFOR
\STATE Get empirical Bellman operator $\widehat{\mathcal{H}}$ (for any $Q$) following Eq. \eqref{Empirical Bellman Operator};
\STATE \textbf{Output:}  Empirical Bellman operator $\widehat{\mathcal{H}_i}$;
\end{algorithmic}
\end{algorithm}

Specifically, we adopt the synchronous sampling scheme \( \{\kappa_p, \{\mathcal{G}_k\}_{k\in [\kappa_p]}\} \) discussed in Section \ref{Multi-Component Factorized Synchronous Sampling}, and then obtain the corresponding $\kappa_p$ sampling sets \( \{X^P_{\mathcal{G}_k}\}_{k\in[\kappa_p]} \). For each \( k \in [\kappa_p] \), we sample over each entry \( x \) from the set \( \mathcal{X}_{\mathcal{G}_k}^P \) once, obtaining the next state $s^k_{x}$ randomly generated based on transition and corresponding reward. This gives us the samples \( \{ s^k_{x}\}_{x \in \mathcal{X}_{\mathcal{G}_k}^P, k \in [\kappa_p]} \). These samples are sufficient to construct the factored empirical Bellman operator \( \widehat{\mathcal{H}}(Q) \), which is defined as follows:
\begin{align}
    [\widehat{\mathcal{H}}(Q)](\Tilde{s}, \Tilde{a}) = r(\Tilde{s}, \Tilde{a}) + \gamma \max_{a'} Q(s'_{\Tilde{x},1}[Z^S_{1}], s'_{\Tilde{x},2}[Z^S_{2}], \dots, s'_{\Tilde{x},K_\omega}[Z^S_{K_\omega}], a') \label{Empirical Bellman Operator}
\end{align}
for all $\Tilde{x}:=(\Tilde{s}, \Tilde{a}) \in \mathcal{S} \times \mathcal{A}$,
where for each \( j \in [K_\omega] \), \( s'_{\Tilde{x},j} \) is the sample satisfying:
\begin{align*}
    s'_{\Tilde{x},j} \in \{s^{k}_{x}\ \mid \ x\in\mathcal{X}_{\mathcal{G}_k}^P, \ x[Z_j^P] = \Tilde{x}[Z_j^P]\}.
\end{align*}
This ensures that the samples from different components are used to compute the transitions for all dimensions of the Bellman operator. By using the synchronous sampling scheme, we cover the transitions for all components, allowing us to estimate the entire factored Bellman operator efficiently.

The variance reduction technique aims to reduce the variance in the update process, thereby decreasing the number of iterations required for convergence \citep{sidford2018near, sidford2023variance, wainwright2019variance}. Armed with both schemes, our proposed VRQL-AF is summarized in Algorithm \ref{Variance-Reduced Q-Learning}. Specifically, the algorithm employs a two-loop structure. In each outer epoch $\tau$, a new reference is first constructed, followed by an inner loop that iteratively updates the Q-value estimate using both the reference and real-time empirical Bellman operators.

\begin{algorithm}
\caption{Variance-Reduced Q-Learning with Approximate Factorization (VRQL-AF)}
\label{Variance-Reduced Q-Learning} 
\begin{algorithmic}[1]
\renewcommand{\algorithmicrequire}{ \textbf{Input}:}
\REQUIRE Number of epochs $T$; Epoch length $M$; Reference sample size $N_\tau$ ($\tau\leq T$); Learning rate $\eta_t$.
\renewcommand{\algorithmicrequire}{ \textbf{Output}:}
\REQUIRE $\epsilon$-accurate Q Function estimation $\overline{Q}_M$ with $1-\delta$ probability;
\end{algorithmic}
\begin{algorithmic}[1]
\STATE Initialize $\overline{Q}_0(s,a) = 0$ for all $s$ and $a$; 
\FOR{epoch $\tau = 1, 2, ..., T$}
\STATE Generate $N_\tau$ factored empirical Bellman operators $\{\widehat{\mathcal{B}_i}\}_{i=1,2,...,N_\tau}$ though Algorithm \ref{Empirical Bellman Operator Generation};
\STATE Calculate the reference Bellman operator:
\begin{align*}
    \overline{\mathcal{H}}_\tau(\overline{Q}_{\tau-1}) = \frac{1}{N_\tau}\sum\nolimits_{i=1}^{N_\tau}\widehat{\mathcal{B}}_i(\overline{Q}_{\tau-1}).
\end{align*}
\STATE Initialize $Q_0 = \overline{Q}_{\tau-1}$;
\FOR{$t = 0, ..., M-1$}
\STATE Generate factored empirical Bellman operator $\widehat{\mathcal{H}}_{t}$ though Algorithm \ref{Empirical Bellman Operator Generation};
\STATE Compute the variance-reduced update:
\begin{align}
    Q_{t+1} = Q_{t} + \eta_{t} \left(\widehat{\mathcal{H}}_{t}(Q_{t}) - \widehat{\mathcal{H}}_{t}(\overline{Q}_{\tau-1}) + \overline{\mathcal{H}}_{\tau}(\overline{Q}_{\tau-1})-Q_{t}\right); \label{variance-reduced Q update}
\end{align}
\ENDFOR
\STATE $\overline{Q}_{\tau} = Q_{M}$;
\ENDFOR
\RETURN $\overline{Q}_{T}$;
\end{algorithmic}
\end{algorithm}

In particular, in each outer loop $\tau$, we first compute a \emph{reference Bellman operator} \( \overline{\mathcal{H}}(\overline{Q}_{\tau-1}) \), which is an average of \( N_\tau \) factored empirical Bellman operators \( \widehat{\mathcal{B}}_i \) generated through a sampling process applied to the fixed Q-value function \( \overline{Q}_{\tau-1} \) (obtained from previous steps). This reference Bellman operator provides a low-variance estimate of the Q-value function.

Next, the inner loop performs variance-reduced Q-iterations. In each iteration \( t \), we generate a new factored empirical Bellman operator \( \widehat{\mathcal{H}}_{t-1} \) using new samples. In Line~\ref{variance-reduced Q update}, the update of the Q-function uses a linear combination of both the high-variance unbiased stochastic factored empirical Bellman operator $\widehat{\mathcal{H}}_{t}(Q_{t})$ and a term to reduce the variance of the stochastic process using the constructed reference. 
This combination of reference and real-time operators allows the algorithm to reduce the number of iterations required to converge to an accurate Q-value function. Meanwhile, there are four parameters to be designed: the number of epochs $T$ in the outer loop, the number of iteration $M$ in the inner loop, the number of samples $N_\tau$ to generate the reference Bellman operators in each epoch, and the learning rate $\eta_t$ at each iteration, which will be specified in the sample complexity guarantee part.

\subsection{Sample Complexity Guarantees}
We now provide a bound on the sample complexity of Algorithm \ref{Variance-Reduced Q-Learning}; the proof is provided in Appendix \ref{proof_to_thm3}:
{
\begin{theorem}
\label{thm_variance-reduced-Q-learning}
    Given any approximate factorization scheme \( \omega \), a confidence level \( \delta > 0 \), and the desired  accuracy level \( \epsilon \in (0,1) \). Let \( \mathcal{E}_\omega = \gamma(1-\gamma)^{-2} \Delta^P_\omega + (1-\gamma)^{-1} \Delta^R_\omega \), $T = c_1 \log\left((1-\gamma)^{-1}\epsilon^{-1}\right)$, $M = c_2{\log\left({6T|\mathcal{X}[\cup_{k=1}^{K_\omega} Z_k^P]|}{(1-\gamma)^{-1}\delta^{-1}}\right)}{(1-\gamma)^{-3}}$, $N_\tau = c_34^\tau{\log\left(6T|\mathcal{X}[\cup_{k=1}^{K_\omega} Z_k^P]|\right)}{(1-\gamma)^{-2}}$, and $\eta_t = {(1+(1-\gamma)(t+1))^{-1}}$. With probability at least \( 1 - \delta \), the output Q-function $\widehat{Q}_\omega^* $ from Algorithm \emph{VRQL-AF} satisfies 
    \begin{align*}
        \|Q^* - \widehat{Q}^*_\omega\|_\infty \leq \epsilon+\mathcal{E}_\omega,
    \end{align*}
    provided that the total number of samples, denoted by $D_\omega$, satisfies the following lower bound: 
    \begin{align*}
        D_\omega \geq  \frac{c_0(\sum_{k\in[\kappa_p]}|\mathcal{X}[Z^P_k]|)\log\left({c_1|\mathcal{X}[\cup_{k=1}^{K_\omega} Z_k^P]|}{(1-\gamma)^{-1}\delta^{-1}}\right)\log\left({c_2}{(1-\gamma)^{-1}\epsilon^{-1}}\right)}{{\epsilon^2(1-\gamma)^3}}+\sum_{i\in[\kappa_r]}|\mathcal{X}[Z^R_i]|,
    \end{align*}
    where $c_0, c_1, c_2 >0$ are universal constants.
\label{main_thm}
\end{theorem}}

In evaluating the performance of VRQL-AF, we note that the estimation bias $\mathcal{E}_\omega$ is the same as in our model-based algorithm in Theorem \ref{thm:summary}.

Therefore, we focus our discussion on the sample complexity, in the case of perfect factorization $\mathcal{E}_\omega = 0$ for fair comparisons. In this case, our sample complexity required to achieve an 
$\varepsilon$-optimal policy is of the following order:
\begin{align}
    \widetilde{O}\left( \frac{\sum\nolimits_{k \in [\kappa_p]} |\mathcal{X}[Z^P_k]|}{\epsilon^2(1-\gamma)^3} + \sum\nolimits_{i \in [\kappa_r]} |\mathcal{X}[Z^R_i]| \right), \label{eq:eq:model-free-results}
\end{align}
which achieves the same order as our model-based algorithm in \eqref{eq:model-based-results}. This indicates that our proposed model-free VRQL-AF can also overcome the persistent curse of dimensionality challenges in standard RL, as already discussed for our model-based algorithm. For more details, please refer to Section~\ref{para:sample-complexity-discussion}.

\paragraph{New Frontiers in Model-Free Paradigm for Factorizable MDPs.}
VRQL-AF stands out as the first near-optimal model-free algorithm for our approximate factorization framework and also the well-known factored MDPs, which breaks the curse of dimensionality by effectively leveraging problem structure. The technical foundation of these results comes from two aspects: 1) We propose a tailored factored empirical Bellman operator based on the factorization structure and a synchronous sampling scheme. This operator updates the Q-value estimation in a sample-efficient manner by effectively leveraging the scope structures inherent across multiple sub-components of the transition kernel; 2) we integrate advanced variance-reduction techniques with the factored Bellman operator to minimize the oscillation of the stochastic process and accelerate convergence, which requires meticulous design to control the complex cross-dimensional correlations introduced by the factored empirical Bellman operator, ensuring that the variance during the variance-reduced Q-iteration remains manageable.

\section{Numerical Experiments}
\label{Sec_Numerical Experiments}
We now present numerical experiments to evaluate our proposed model-based and model-free algorithms based on approximate factorization with comparisons to state-of-the-art RL algorithms. We focus on two types of tasks: (i) synthetic MDP tasks and (ii) an electricity storage control problem in power system operation. 

\subsection{Synthetic MDP Tasks}
\label{Synthetic MDP Tasks}
The two synthetic MDP tasks that we consider are distinguished by their transition kernels: one with a perfectly factorizable transition kernel and another with an imperfectly factorizable one. For the former task, we consider a three-dimensional state space, where each state is represented as \( s = (s[1], s[2], s[3]) \), along with a one-dimensional action \( a \). The transitions for \( s[1] \) and \( s[2] \) are independent of other components, while \( s[3] \) is influenced by both its own state and the action \( a \). Each substate and the action space consists of 5 discrete elements. Both the transition kernel and the reward function are randomly generated.
We conduct 50 trials to compare our model-based and model-free algorithms (VRQL-AF) with approximate factorization against traditional model-based RL and a variance-reduced Q-learning method without factorization. Figures. \ref{Performance Evaluation_Q_error} and \ref{Decompose_case1} illustrate the \( l_\infty \)-norm of the \( l_\infty \)-norm of the \( Q \)-function error for both settings, respectively, as a function of the number of samples.
As shown in the two figures, our approximate factorization methods (depicted in red) consistently exhibit lower \( Q \)-function error across varying sample sizes compared to the vanilla RL approaches (depicted in blue). This demonstrates that our approach significantly reduces sample complexity, showcasing superior sample efficiency. These findings underline the effectiveness of our factorization-based framework in enhancing learning efficiency for multi-dimensional MDPs.

\begin{figure}[htb]
    \centering
    \subfigure[Model-based RL]
    {
\includegraphics[width=0.40\linewidth]{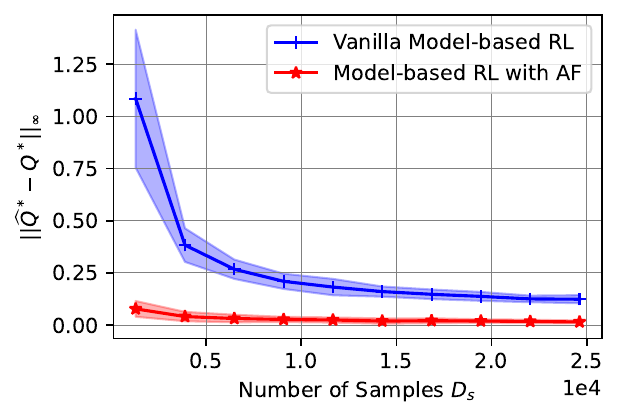}
\label{Performance Evaluation_Q_error}
}
    \subfigure[Model-free RL]
    {
\includegraphics[width=0.40\linewidth]{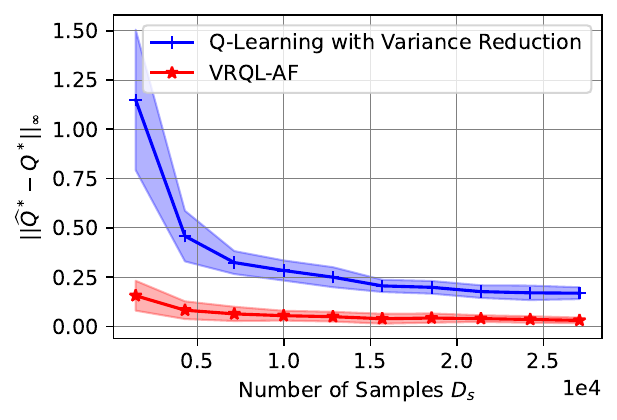}
        \label{Decompose_case1}
    }
    
    \caption{Performance on Perfectly Factorizable MDPs}
    \label{Show1}
    
\end{figure}

The second task involves MDPs with imperfectly factorizable transition kernels. Here, the state consists of 4 substates, each taking one of 5 possible values. To investigate the trade-off between sample complexity and estimation error, we evaluate three factorization schemes: (1) full factorization with \( K_\omega = 4 \), (2) partial factorization with \( K_\omega = 2 \), and (3) no factorization with (\( K_\omega = 1 \)), corresponding to vanilla RL. We conduct 200 trials to compare their performance with different sample amounts.
As shown in Figure \ref{Show2}, our model-based and model-free RL algorithms exhibit distinct convergence behaviors under these different factorization levels. The results indicate two key breaking points $(D_1, \mathcal{E}_1)$ and $(D_2, \mathcal{E}_2)$ in performance trade-offs (illustrated in Figure \ref{Imperfectly Factorizable MDP2}), which inform the optimal factorization choice depending on the sample availability. Specifically, in the small sample range with $D_s\leq D_1$, the full factorization scheme (\( K_\omega = 4 \)) converges quickly but has a higher asymptotic error, making it suitable when the samples are limited or the required accuracy is not high.
In the intermediate sample range with $D_1<D_s\leq D_2$, the partial factorization scheme (\( K_\omega = 2 \)) provides a balanced trade-off, achieving moderate convergence speed and a lower asymptotic error. This scheme is advantageous when a compromise between sample efficiency and accuracy is needed.
In the large sample range with $D_s>D_2$, the no factorization scheme (\( K_\omega = 1 \)) results in the smallest asymptotic error but converges slowly. This approach is best suited when very high precision is required, and a large sample size is available.
Therefore, it highlights the benefits of selecting an appropriate factorization level to match the sample size and desired accuracy. The flexibility in choosing among different factorization strategies enables the optimization of performance based on specific requirements, making our approach adaptable to various settings.

\begin{figure}[tb]
    \centering
    \subfigure[Model-based RL]
    {
\includegraphics[width=0.40\linewidth]{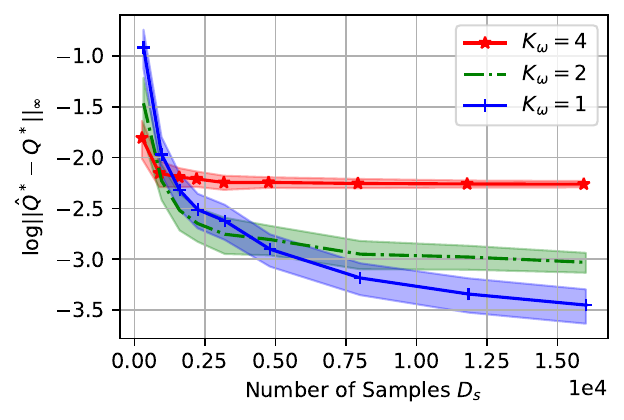}
\label{Imperfectly Factorizable MDP1}
}
    \subfigure[Model-free RL]
    {
\includegraphics[width=0.40\linewidth]{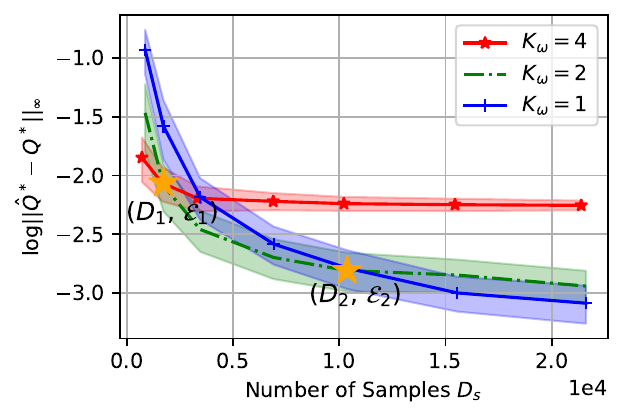}
        \label{Imperfectly Factorizable MDP2}
    }
    
    \caption{Performance on Imperfectly Factorizable MDPs}
    \label{Show2}
    
\end{figure}

\subsection{Wind Farm Storage Control Problem}
We next evaluate the performance of our proposed approach on the wind farm storage control problem introduced in Section \ref{An Illustrative Example}. As depicted in Figure \ref{storage control model}, this problem involves managing energy storage systems to mitigate mismatches between variable wind power generation and demand, thereby minimizing penalty costs.

The control actions consist of charging and discharging decisions influenced by real-time wind generation and unit penalty prices. Constraints include storage capacity limits, prohibiting simultaneous charging and discharging, and lossy energy dynamics. For a detailed model description, please refer to Appendix \ref{detail_storage control}. The total state-action space in this problem is $10^4$.

\begin{figure}[t]
\centerline{\includegraphics[width=0.63\linewidth]{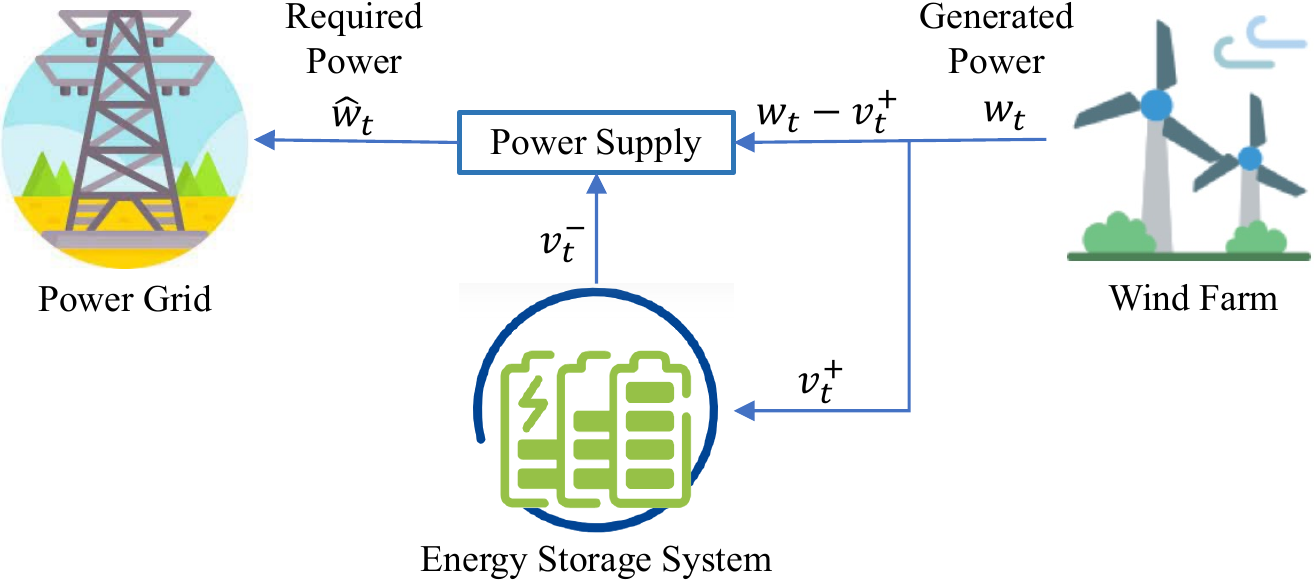}}
\caption{Wind Farm-equipped Storage Control.}
\label{storage control model}

\end{figure}

\begin{figure}
    \centering
    \subfigure[Convergence of Q-function Estimation]
    {
\includegraphics[width=0.40\linewidth]{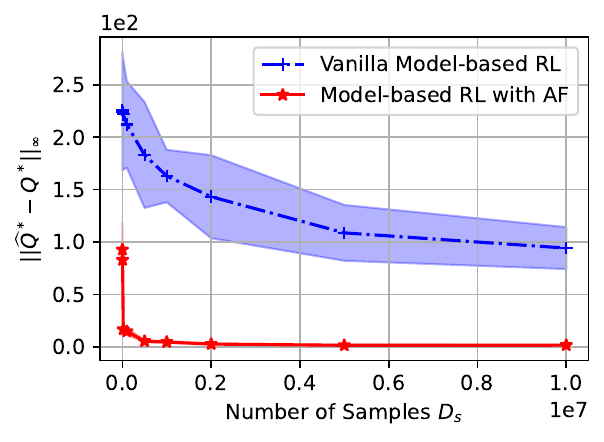}
        \label{storage_control_q_error}
    }
    \subfigure[Economic Performance]
    {
\includegraphics[width=0.38\linewidth]{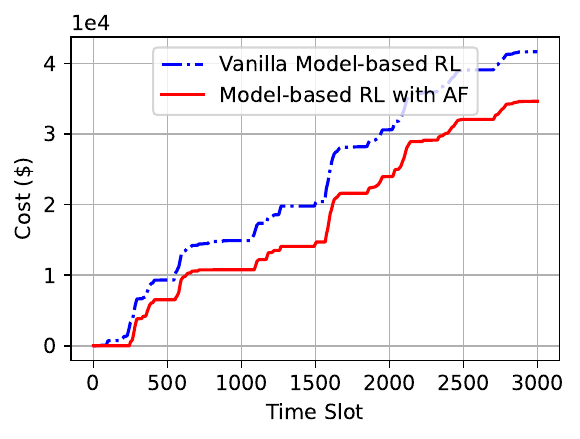}
        \label{storage_control_cost}
    }
    
    \caption{Performance on Storage Control}
    \label{Performance for Storage Control}
    
\end{figure}

Figure \ref{Performance for Storage Control} illustrates the performance of our approach, where we decompose the MDP into three components as described earlier. Specifically, Figure \ref{storage_control_q_error} demonstrates that our approximate factorization-based algorithm converges significantly faster than the vanilla model-based approach by leveraging the model's structural properties. Additionally, our method exhibits considerably lower variance, indicating more stable performance for different sample conditions.
Figure \ref{storage_control_cost} presents the economic benefits of our algorithm using a sample size of 60,000 data points (typically taking the wind farm a year to collect), evaluating cumulative system penalty costs over 3,000 time slots. The results show that our approach reduces costs by 19.3$\%$ compared to the vanilla model-based method, highlighting its enhanced effectiveness in managing complex problems with large-scale state spaces.

\section{Conclusion}
\label{Sec_Conclusion}
Our work introduces a comprehensive framework for approximate factorization in MDPs, effectively addressing the curse of dimensionality that typically hinders scalability in large-scale RL problems. By decomposing MDPs into smaller and manageable components, our approach significantly enhances learning efficiency. We develop a synchronous sampling method to optimize the sampling process across multiple components and propose a corresponding model-based RL algorithm that fully leverages the factorization structure. Additionally, we design variance-reduced Q-learning algorithms incorporating a factored empirical Bellman operator, which facilitates efficient learning in an online setting.

Our theoretical findings demonstrate significant improvements in sample complexity, surpassing existing bounds for both vanilla and factored MDPs. Numerical simulations on synthetic and real-world high-dimensional problems confirm the practical advantages of our factorization-based approach, indicating its potential to transform scalability in RL. This framework not only provides a versatile tool for large-scale RL but also opens up promising directions for future work, including exploring other forms of factorization with provable guarantees and adapting the approach to varied RL environments.

\bibliographystyle{apalike}
\bibliography{arxiv}

\newpage
\begin{center}
    {\LARGE\bfseries Appendices}
\end{center}

\appendix

\section {Preliminary Definitions and Lemmas}
In the following, we provide all necessary definition, preliminary lemmas. For self-contained, we provide the proofs to the lemmas according to the notations of this work.
\begin{definition} \label{var_def}
Let $\text{Var}_P(V)\in \mathbb{R}^{|\mathcal{S}||\mathcal{A}|}$ be the value function variance with transition kernel $P$, which is given by:
    \begin{align*}
        \text{Var}_P(V) = P(V)^2 -(PV)^2,
    \end{align*}
where $P\in\mathbb{R}^{|\mathcal{S}||\mathcal{A}|\times|\mathcal{S}|}$ is the transition kernel, and $(V)^2\in \mathbb{R}^{|\mathcal{S}|}$ is the element-wise product of $V$, i.e.,  $(V)^2 = V\circ V$.
\end{definition}

\begin{definition}
\label{def_epsilon}
We define $\Sigma^\pi_{M}(s,a)$ as the variance of discounted reward under policy $\pi$ and MDP $M$ given state $s$ and action $a$, i.e.,
    \begin{align*}
        \Sigma^\pi_{M}(s,a) = \mathbb{E}\left[\left.\left(\sum_{t=0}^\infty \gamma^t r(s_t, a_t)-Q^\pi_{M}(s_0,a_0)\right)^2\right|s_0 = s, a_0 = a\right],
    \end{align*}
    where $Q^\pi_{M}$ denotes the Q-function induced by policy $\pi$ under MDP $M$.
\end{definition}

\begin{lemma}
\label{decomp_lemma1}
(Lemma 2.2 in \citep{agarwal2019reinforcement})
    For any policy $\pi$, it holds that:
    \begin{align*}
        Q^{\pi} - \widehat{Q}^{\pi} = \gamma(I-\gamma \widehat{P}^{\pi})^{-1}(P-\widehat{P})V^{\pi},
    \end{align*}
    where $Q^\pi$ and $\widehat{Q}^\pi$ are Q-functions induced by the policy $\pi$ under transition kernels $P$ and $\widehat{P}$, respectively, and $\widehat{P}^{\pi} \in \mathbb{R}^{|\mathcal{S}||\mathcal{A}|\times|\mathcal{S}||\mathcal{A}|}$ represents the transition matrix of the Markov chain $\{(s_t,a_t)\}$ induced by $\pi$ with transition kernel $\widehat{P}$.
    % , and $V^\pi$ denotes the value function under policy $\pi$.
\end{lemma}

\begin{proof}
    Recall that $Q^\pi$ is the unique solution of the Bellman equation:
\begin{align*}
    Q^\pi = r + \gamma P^\pi Q^\pi.
\end{align*}
Since $I-\gamma P^\pi$ is invertible when $0<\gamma<1$, we have
\begin{align}
    Q^\pi = (I-\gamma P^\pi)^{-1}r. \label{lemma a.1-Q_invert}
\end{align}
It follows that
    \begin{align*}
        Q^{\pi} - \widehat{Q}^{\pi} &= (I-\gamma P^\pi)^{-1}r - (I-\gamma \widehat{P}^\pi)^{-1}r\\
        &=(I-\gamma \widehat{P}^\pi)^{-1}((I-\gamma \widehat{P}^\pi)-(I-\gamma P^\pi))(I-\gamma P^\pi)^{-1} r\\
        &=\gamma (I-\gamma \widehat{P}^\pi)^{-1} (P^\pi - \widehat{P}^\pi)(I-\gamma P^\pi)^{-1} r\\
        &=\gamma (I-\gamma \widehat{P}^\pi)^{-1} (P^\pi - \widehat{P}^\pi)Q^\pi \\
        &=\gamma (I-\gamma \widehat{P}^\pi)^{-1}(P - \widehat{P})V^{\pi},
    \end{align*}
    where the fourth equality is due to Eq. \eqref{lemma a.1-Q_invert}. The last equality is because $P^\pi Q^\pi = P V^\pi$.
\end{proof}

\begin{lemma}
\label{decomp_lemma2}
(Lemma 2.5 in \citep{agarwal2019reinforcement})
    For any two optimal Q functions $Q^*$ and $\widehat{Q}^*$, which are induced by the same reward function $r$, but different transition kernels $P$ and $\widehat{P}$. It holds that:
    \begin{align*}
        Q^{*} - \widehat{Q}^{*} \leq \gamma(I-\gamma \widehat{P}^{\pi^*})^{-1}(P-\widehat{P})V^{*},\\
        Q^{*} - \widehat{Q}^{*} \geq \gamma(I-\gamma \widehat{P}^{\widehat{\pi}^*})^{-1}(P-\widehat{P})V^{*},
    \end{align*}
    where $\pi^*$ and $\widehat{\pi}^*$ are the optimal policies induced by Q-value functions $Q^*$ and $\widehat{Q}^*$. The matrix \( \widehat{P}^{\pi^*} \in \mathbb{R}^{|\mathcal{S}||\mathcal{A}| \times |\mathcal{S}||\mathcal{A}|} \) represents the transition matrix of the Markov chain \( \{(s_t, a_t)\} \) induced by the policy \( \pi^* \) under the transition kernel \( \widehat{P} \). Similarly, \( \widehat{P}^{\widehat{\pi}^*} \in \mathbb{R}^{|\mathcal{S}||\mathcal{A}| \times |\mathcal{S}||\mathcal{A}|} \) represents the transition matrix of the Markov chain \( \{(s_t, a_t)\} \) induced by the policy \( \widehat{\pi}^* \) under the transition kernel \( \widehat{P} \). The optimal value function \( V^* \) is induced by \( Q^* \), which satisfies \( V^*(s) = \max_a Q^*(s, a) \) for all states \( s \in \mathcal{S} \).
\end{lemma}

\begin{proof}
    The two conditions can be proved using Lemma \ref{decomp_lemma1}. Specifically, the first inequality can be proved as follows:
    \begin{align*}
        Q^{*} - \widehat{Q}^{*} = Q^{\pi^*} - \widehat{Q}^{\widehat{\pi}*} \leq Q^{\pi^*} - \widehat{Q}^{\pi^*} = \gamma(I-\gamma \widehat{P}^{\pi^*})^{-1}(P-\widehat{P})V^{*},
    \end{align*}
    where the inequality is because $\widehat{Q}^{\widehat{\pi}*}\geq \widehat{Q}^\pi$ for any policy $\pi$. The last equality comes from Lemma \ref{decomp_lemma1}.

    For the second condition, we have:
    \begin{align*}
        Q^{*} - \widehat{Q}^{*} &= Q^{\pi^*} - \widehat{Q}^{\widehat{\pi}^*}\\
        &=Q^{\pi^*} - (I-\gamma\widehat{P}^{\widehat{\pi}^*})^{-1}r\\
        &=(I-\gamma\widehat{P}^{\widehat{\pi}^*})^{-1}(I-\gamma\widehat{P}^{\widehat{\pi}^*})Q^{\pi^*}-(I-\gamma\widehat{P}^{\widehat{\pi}^*})^{-1}(I-\gamma{P}^{{\pi}^*})Q^{\pi^*}\\
        &= (I-\gamma\widehat{P}^{\widehat{\pi}^*})^{-1}((I-\gamma\widehat{P}^{\widehat{\pi}^*})-(I-\gamma{P}^{{\pi}^*}))Q^{\pi^*}\\
        &= \gamma(I-\gamma\widehat{P}^{\widehat{\pi}^*})^{-1}({P}^{{\pi}^*}-\widehat{P}^{\widehat{\pi}^*})Q^{\pi^*}\\
        &\geq  \gamma(I-\gamma\widehat{P}^{\widehat{\pi}^*})^{-1}({P}^{{\pi}^*}-\widehat{P}^{{\pi}^*})Q^{\pi^*}\\
        &= \gamma(I-\gamma\widehat{P}^{\widehat{\pi}^*})^{-1}(P-\widehat{P})V^{*}.
    \end{align*}
    The second equality is because $Q^\pi = (I-\gamma P^{\pi})^{-1}r$, and the inequality is because $\widehat{P}^{\widehat{\pi}^*}Q^{\pi^*}\leq \widehat{P}^{{\pi}^*}Q^{\pi^*}$ due to the optimality of policy ${\pi}^*$ regarding Q-function $Q^{\pi^*}$.
\end{proof}

\begin{lemma}
\label{Lemma 4-basic}
(Adapted from Lemma 5 in \citep{azar2012sample}) Given a MDP $M$ with transition kernel $P$ and reward function $r$, the following identity holds for all policy $\pi$:
    \begin{align*}
        \Sigma_M^\pi = \gamma^2{(1-\gamma^2P^{\pi})^{-1}\text{Var}_P(V_M^\pi)}
    \end{align*}
\end{lemma}
\begin{proof}
We start with the definition of $\Sigma_M^\pi(s,a)$: 
\begin{align*}
       & \Sigma^\pi_{M}(s,a) =\mathbb{E}_{P, \pi}\Bigg[\Bigg(\sum_{t=0}^\infty \gamma^t r(s_t, a_t)-Q^\pi_{M}(s_0,a_0)\Bigg)^2\bigg|~s_0 = s, a_0 = a \Bigg]\\
        &=\mathbb{E}_{P, \pi}\left[\left.\left(\sum_{t=1}^\infty \gamma^t r(s_t, a_t) - \gamma Q_M^\pi(s_1, a_1)-(Q^\pi_M(s_0,a_0) - r(s_0,a_0)-\gamma Q_M^\pi(s_1, a_1))\right)^2\right|s_0 = s, a_0 = a\right]\\
        &=\mathbb{E}_{P, \pi}\left[\left.\left(\sum_{t=1}^\infty \gamma^t r(s_t, a_t) - \gamma Q_M^\pi(s_1, a_1)\right)^2\right|s_0 = s, a_0 = a\right]\\
        &\ \ -2\mathbb{E}_{P, \pi}\left[\left.\left(\sum_{t=1}^\infty \gamma^t r(s_t, a_t) - \gamma Q_M^\pi(s_1, a_1)\right)\left(Q^\pi_M(s_0,a_0) - r(s_0,a_0)-\gamma Q_M^\pi(s_1, a_1)\right)\right|s_0 = s, a_0 = a\right]\\
        &\ \ +\mathbb{E}_{P, \pi}\left[\left.\left(Q^\pi_M(s_0,a_0) - r(s_0,a_0)-\gamma Q_M^\pi(s_1, a_1)\right)^2\right|s_0 = s, a_0 = a\right]\\
        &=\gamma^2\mathbb{E}_{P, \pi}\left[\left.\left(\sum_{t=1}^\infty \gamma^{t-1} r(s_t, a_t) -  Q_M^\pi(s_1, a_1)\right)^2\right|s_0 = s, a_0 = a\right]\\
        &\ \ -2\mathbb{E}_{P, \pi}\left[\left.\mathbb{E}\left[\left.\sum_{t=1}^\infty \gamma^t r(s_t, a_t) - \gamma Q_M^\pi(s_1, a_1)\right|s_1, a_1\right]\left(Q^\pi_M(s_0,a_0) - r(s_0,a_0)-\gamma Q_M^\pi(s_1, a_1)\right)\right|s_0 = s, a_0 = a\right]\\
        & \ \ +\mathbb{E}_{P, \pi}\left[\left.\left(Q^\pi_M(s_0,a_0) - r(s_0,a_0)-\gamma Q_M^\pi(s_1, a_1)\right)^2\right|s_0 = s, a_0 = a\right]\\
        &=\gamma^2\mathbb{E}_{P, \pi}\left[\left.\left(\sum_{t=1}^\infty \gamma^{t-1} r(s_t, a_t) -  Q_M^\pi(s_1, a_1)\right)^2\right|s_0 = s, a_0 = a\right]\\
        &\ \ +\mathbb{E}_{P, \pi}\left[\left.\left(Q^\pi_M(s_0,a_0) - r(s_0,a_0)-\gamma Q_M^\pi(s_1, a_1)\right)^2\right|s_0 = s, a_0 = a\right]\\
        &= \gamma^2\mathbb{E}_{P, \pi}\left[\left.\left(\sum_{t=1}^\infty \gamma^{t-1} r(s_t, a_t) -  Q_M^\pi(s_1, a_1)\right)^2\right|s_0 = s, a_0 = a\right]\\
        &\ \ +\gamma^2 \mathbb{E}_{P, \pi}\left[\left.\left(\mathbb{E}_{s_1,a_1 \sim P(\cdot|s_0,a_0)}[Q^\pi_M(s_1,a_1)]- Q_M^\pi(s_1, a_1)\right)^2\right|s_0 = s, a_0 = a\right]\\
        &=  \gamma^2 \sum_{s_1,a_1} P^{\pi}(s_1,a_1 | s,a) \Sigma_M^\pi(s_1,a_1) + \gamma^2 \text{Var}_P(V_M^\pi)(s,a),
    \end{align*}
where the third equality is obtained by dividing the quadratic term; the fourth equality is derived by the law of total expectation; the fifth equality holds due to $\mathbb{E}\left[\left.\sum_{t=1}^\infty \gamma^t r(s_t, a_t) - \gamma Q_M^\pi(s_1, a_1)\right|s_1, a_1\right] = 0$. 
The last equality is derived based on the definitions of $\Sigma^\pi_{M}(s,a)$ and $\text{Var}_P(V_M^\pi)$.
\end{proof}

\begin{lemma}
\label{lemmaA4}
    (Lemma 6 in \citep{wainwright2019variance}) Given two Q-functions $Q^*_{r}$ and $\Tilde{Q}^*$, which are the induced by the same transition kernel $P$, but different reward functions $r$ and $\Tilde{r} = r+\Delta r$. It holds that,
\begin{align*}
    |Q^*-\Tilde{Q}^*| \leq \max\left\{({I}-\gamma {P}^{\pi^*})^{-1}|\Delta r|, ({I}-\gamma {P}^{\Tilde{\pi}^*})^{-1}|\Delta r| \right\},
\end{align*}
where $\pi^*$ and $\Tilde{\pi}^*$ denote the optimal policy induced by $Q^*$ and $\Tilde{Q}^*$, respectively.
\end{lemma}
\begin{proof}
    This lemma can be proved by showing the following two conditions:
    \begin{align}
        \max(Q^*-\Tilde{Q}^*,\boldsymbol{0})\leq ({I}-\gamma {P}^{\pi^*})|\Delta r|,\label{lemmaA4_eq1}\\
        \max(\Tilde{Q}^*-Q^*,\boldsymbol{0})\leq ({I}-\gamma {P}^{\Tilde{\pi}^*})|\Delta r|.\label{lemmaA4_eq2}
    \end{align}

For condition \eqref{lemmaA4_eq1}, we can prove the following:
\begin{align*}
    Q^*-\Tilde{Q}^* &= r+\gamma P^{\pi^*} Q^* - (r+\Delta r + \gamma P^{\Tilde{\pi}^*} \Tilde{Q}^*) \\
    &\leq |\Delta r| + \gamma P^{\pi^*} (Q^*-\Tilde{Q}^*)\\
    &\leq |\Delta r| + \gamma P^{\pi^*} \max(Q^*-\Tilde{Q}^*,\boldsymbol{0}),
\end{align*}
where the second inequality comes from $P^{\pi^*}\Tilde{Q}^*\leq P^{\Tilde{\pi}^*} \Tilde{Q}^*$.

Since the right-hand-side term is positive in all entries. Thus, we have:
\begin{align*}
    \max(Q^*-\Tilde{Q}^*,\boldsymbol{0})\leq |\Delta r| + \gamma P^{\pi^*} \max(Q^*-\Tilde{Q}^*,\boldsymbol{0}).
\end{align*}

Rearranging the inequality yields Eq. \eqref{lemmaA4_eq1}.

For proving Eq. \eqref{lemmaA4_eq2}, we follow a similar routine and get:
\begin{align*}
    \Tilde{Q}^*-Q^*&= (r+\Delta r + \gamma P^{\Tilde{\pi}^*} \Tilde{Q}^*) -(r+\gamma P^{\pi^*} Q^*)\\
    &\leq |\Delta r|+ \gamma P^{\Tilde{\pi}^*}(\Tilde{Q}^*-Q^*)\\
    &\leq |\Delta r|+\gamma P^{\Tilde{\pi}^*}\max(\Tilde{Q}^*-Q^*,\boldsymbol{0}).
\end{align*}
Hence, we have:
\begin{align*}
    \max(\Tilde{Q}^*-Q^*,\boldsymbol{0})\leq |\Delta r| + \gamma P^{\pi^*} \max(\Tilde{Q}^*-Q^*,\boldsymbol{0}).
\end{align*}

Rearranging the inequality yields Eq. \eqref{lemmaA4_eq2}. Using $|Q^*-\Tilde{Q}^*| = \max(\max(Q^*-\Tilde{Q}^*,\boldsymbol{0}),\max(\Tilde{Q}^*-Q^*,\boldsymbol{0}))$ and combining Eq. \eqref{lemmaA4_eq1}-\eqref{lemmaA4_eq2} yields the desired result.
\end{proof}

}

\section{Proof of Theorem \ref{thm:summary}}\label{proof:them:summary}
We aim to ensure  the total number of samples \( D_\omega \) satisfies the following lower bound:
\begin{align}
    D_\omega \geq N_{\text{entry}} \cdot N + D_r, \label{def_D}
\end{align}
where \( N_{\text{entry}} \) denotes the number of unique state-action pairs that must be sampled, \( N \) is the sampling frequency for each pair, and \( D_r \) is the sample complexity required to estimate the exact reward function.

The term $N_{\text{entry}}$ is directly specified in Algorithm \ref{Cost-Optimal Factorized Synchronous Sampling Algorithm}. It is defined as the total number of state-action pairs across all components to be sampled. Mathematically:
\begin{align*}
    N_{\text{entry}} = {\sum\nolimits_{i\in[\kappa_p]}\max\nolimits_{k\in\mathcal{G}_i}|\mathcal{X}[Z_k^P]|}\leq \sum\nolimits_{k\in[\kappa_p]}|\mathcal{X}[Z_k^P]|,
\end{align*}
where $\kappa_p$ is the total number of sampling sets, $\mathcal{G}_i$ is the set of component indices associated with the \(i\)-th sampling set, \( |\mathcal{X}[Z_k^P]| \) denotes the size of the state-action space for component \( k \).

For the sample complexity associated with estimating the reward function, \( D_r \), it is sufficient to sample all necessary state-action pairs once for each reward component \( r_i \) to obtain the exact reward values. Similar to the transition kernel sampling, the sample complexity of the reward function is bounded by:
\[
D_r \leq  \max\nolimits_{i\in[\kappa_r]}|\mathcal{X}[Z_i^R]|
\]
where \( \kappa_r \) is the minimal number of sampling sets required to estimate the reward function, and \( |\mathcal{X}[Z_i^R]| \) is the size of the state-action space for reward component \( r_i \).

Therefore, we only need to determine the required sampling frequency \( N \). Recall that, the error $\|Q^* - \widehat{Q}^*_\omega\|_\infty$ comes from two aspects: 1) the computation error in Algorithm \ref{alg:Value Iteration (VI)} due to finite value function iterations, and 2) finite sample error due to inaccurate estimation of $\widehat{P}$.
Since the value iteration algorithm converges exponentially fast, according to Section 5.2 in \citep{shi2024curious}, $T = \overline{c}_0\log(\frac{1}{(1-\gamma)\epsilon})$ is enough to guarantee the computation error $\leq\mathcal{O}(\epsilon)$. Instead, we focus on the estimation error due to finite samples. To do so, we decompose the estimation error of the \( Q \)-value function into two terms, the bias $\mathcal{E}_\omega$ and the finite sample error $\alpha_N$. By applying the triangle inequality to the error, we get:
\begin{align}
    \|Q^* - \widehat{Q}^*_\omega\|_\infty \leq \underbrace{\|Q^*-Q^*_\omega\|_\infty}_{\text{Bias $\mathcal{E}_\omega$}}+\underbrace{\|Q^*_\omega-\widehat{Q}^*_\omega\|_\infty}_{\text{Finite Sample Error $\alpha_N$}},
\end{align}
where
\begin{itemize}
    \item \( Q^* \) is the optimal \( Q \)-value function induced by the actual optimal policy \( \pi^* \), transition kernel \( P \), and reward function \( r \).
    \item \( \widehat{Q}^*_\omega \) is the estimated \( Q \)-value function based on the estimated policy \( \widehat{\pi}^*_\omega \), estimated transition kernel \( \widehat{P}_\omega \), and estimated reward function \( r_\omega \), where
\begin{align*}
    &\widehat{P}_\omega(s'|x) =  \prod_{k=1}^{K_\omega}\widehat{P}_k(s'[Z_k^S] \mid x[Z_k^P]), \forall s'\in\mathcal{S}, \ \forall x\in \mathcal{X},\\
    &r_\omega(x) = \sum_{i=1}^{\ell_\omega} r_i(x[Z_i^R]), \forall x \in \mathcal{X}.
\end{align*}
    \item \( Q^*_\omega \) is the optimal \( Q \)-value function based on the factorization scheme \( \omega \), which assumes the utilization of infinite samples. It is induced by the policy \( \pi_\omega \), transition kernel \( P_\omega \), and reward function \( r_\omega \), where $P_\omega(s'|x)$ is defined as:
    \begin{align*}
        P_\omega(s'|x) := \lim_{N \to \infty} \widehat{P}_\omega(s'|x).
    \end{align*}
    
    The limits exists due to the fixed sampling algorithm and the law of large numbers. DUe to the deterministic mapping between $P$ and $Q$, we have $\lim_{N\xrightarrow[]{} \infty} \widehat{Q}^*_\omega  = {Q}^*_\omega$ with probability $1$.
\end{itemize}

Thus, the finite sample error term \(\alpha_N:= \|Q^*_\omega - \widehat{Q}^*_\omega\|_\infty \) in the decomposition can be made arbitrarily small as \( N \) increases. Instead, the bias $\mathcal{E}_\omega$ doesn't vanish due to the approximation errors from factorization.

For the bias term $\mathcal{E}_\omega$, it can be bounded as follows: 
\begin{lemma}[Proof in Appendix \ref{proof_error_bound}]
    \label{bound_error}
    Given any factorization scheme $\omega$, the following condition holds:
    \begin{align}
        \mathcal{E}_\omega := \|Q^*-Q^*_\omega\|_\infty \leq \frac{\Delta_\omega^R}{1-\gamma} + \frac{\gamma\Delta_\omega^P}{(1-\gamma)^2}.
    \end{align}
\end{lemma}

Then, we focus on the finite sample error term $\alpha_N$. Using the result in Lemma \ref{decomp_lemma2}, we have:
\begin{align}
    &Q^*_\omega-\widehat{Q}^*_\omega\leq \underbrace{\gamma(I-\gamma \widehat{P}^{\pi^*_\omega}_\omega)^{-1}(P_\omega-\widehat{P}_\omega)V^*_\omega}_{:=\Delta_{1}},\label{MB_key_up_bound}\\
    &Q^*_\omega-\widehat{Q}^*_\omega\geq \underbrace{\gamma(I-\gamma \widehat{P}^{\widehat{\pi}^*_\omega}_\omega)^{-1}(P_\omega-\widehat{P}_\omega)V^*_\omega}_{:=\Delta_{2}}, \label{MB_key_low_bound}
\end{align}
where $\widehat{P}^{\pi^*_\omega}_\omega \in \mathbb{R}^{|\mathcal{S}||\mathcal{A}|\times|\mathcal{S}||\mathcal{A}|}$ represents the transition matrix of Markov chain $\{(s_t, a_t)\}$ induced by policy $\pi^*_\omega$ under transition kernel $\widehat{P}_\omega$, and $\widehat{P}^{\widehat{\pi}^*_\omega}_\omega\in \mathbb{R}^{|\mathcal{S}||\mathcal{A}|\times|\mathcal{S}||\mathcal{A}|}$ is induced by policy $\widehat{\pi}^*_\omega$ under transition kernel $\widehat{P}_\omega$.

Based on Eq. \eqref{MB_key_up_bound} and \eqref{MB_key_low_bound}, the error term $\alpha_N$ satisfies:
\begin{align}
    \alpha_N:=\|Q^*_\omega-\widehat{Q}^*_\omega\|_\infty \leq \max\{ \|\Delta_1\|_\infty, \|\Delta_2\|_\infty\}, 
\end{align}

Now we only need to control $\|\Delta_1\|_\infty$ and $\|\Delta_2\|_\infty$. We first control the absolute value of the common term $(P_\omega-\widehat{P}_\omega)V_\omega^*$ for both $\Delta_1$ and $\Delta_2$ in Lemma \ref{Low_dim_lemma}.  
\begin{lemma}[Proof in Appendix \ref{proof to key lemma1}]
\label{Low_dim_lemma}
Given sample size $N$, then with probability at least $1-\delta$, the following condition holds:
    \begin{align*}
    |(P_\omega-\widehat{P}_\omega)V_\omega^*|\leq \frac{2\log(2|\mathcal{X}[\cup_{k=1}^{K_\omega}Z^P_k]|)}{3N(1-\gamma)}\cdot\boldsymbol{1}+\sqrt{\frac{2\log(2|\mathcal{X}[\cup_{k=1}^{K_\omega}Z^P_k]|)\text{Var}_{P_\omega}(V_\omega^*)}{N}},
\end{align*}
where $|\cdot|$ denotes the absolute value function, $\text{Var}_{P_\omega}(V_\omega^*)$ denotes the value function variance with transition kernel $P_\omega$, as defined in Definition \ref{var_def}.
\end{lemma}

Using the naive bound $\|\gamma(I-\gamma \widehat{P}^{\pi}_\omega)^{-1}\|_\infty \leq \frac{\gamma}{1-\gamma}$ for any $\pi$, and combining Lemma \ref{Low_dim_lemma} with Eq. \eqref{MB_key_up_bound} and \eqref{MB_key_low_bound}, we can show that
\begin{align}
    |\Delta_1| \leq& \frac{2\gamma\log(2|\mathcal{X}[\cup_{k=1}^{K_\omega}Z^P_k]|)}{3N(1-\gamma)^2}\cdot\boldsymbol{1}+\gamma\sqrt{\frac{2\log(2|\mathcal{X}[\cup_{k=1}^{K_\omega}Z^P_k]|)}{N}}(I-\gamma \widehat{P}^{\pi^*_\omega}_\omega)^{-1}\sqrt{\text{Var}_{P_\omega}(V^*_\omega)}
    ,\label{Q_bound_key}\\
    |\Delta_2| \leq& \frac{2\gamma\log(2|\mathcal{X}[\cup_{k=1}^{K_\omega}Z^P_k]|)}{3N(1-\gamma)^2}\cdot\boldsymbol{1}+\gamma\sqrt{\frac{2\log(2|\mathcal{X}[\cup_{k=1}^{K_\omega}Z^P_k]|)}{N}}(I-\gamma \widehat{P}^{\widehat{\pi}^*_\omega}_\omega)^{-1}\sqrt{\text{Var}_{P_\omega}(V^*_\omega)}.\label{Q_bound_key_1}
\end{align}

The remaining challenge is to control $(I-\gamma \widehat{P}^{\widehat{\pi}^*_\omega}_\omega)^{-1}\sqrt{\text{Var}_{P_\omega}(V^*_\omega)}$  and $(I-\gamma \widehat{P}^{\pi^*_\omega}_\omega)^{-1}\sqrt{\text{Var}_{P_\omega}(V^*_\omega)}$ for $\Delta_1$ and $\Delta_2$, respectively. We first bound the term $(I-\gamma \widehat{P}^{\widehat{\pi}^*_\omega}_\omega)^{-1}\sqrt{\text{Var}_{P_\omega}(V^*_\omega)}$ for $\Delta_1$, which is slightly more complex, and bounding another term is analogous. Specifically,
\begin{align*}
    &|(I-\gamma \widehat{P}^{\widehat{\pi}^*_\omega}_\omega)^{-1}\sqrt{\text{Var}_{P_\omega}(V^*_\omega)}|\\ =& (I-\gamma \widehat{P}_\omega^{\widehat{\pi}_\omega^*})^{-1}\sqrt{\text{Var}_{P_\omega}(V^*_\omega)-\text{Var}_{\widehat{P}_\omega}(V^*_\omega) +\text{Var}_{\widehat{P}_\omega}(V^*_\omega)}\\
        =&(I-\gamma \widehat{P}_\omega^{\widehat{\pi}_\omega^*})^{-1}\sqrt{P_\omega(V_\omega^*)^2 - (P_\omega V^*)^2 - \widehat{P}_\omega(V_\omega^*)^2 +(\widehat{P}_\omega V_\omega^*)^2+\text{Var}_{\widehat{P}_\omega}(V_\omega^*)}\\
        =&(I-\gamma \widehat{P}_\omega^{\widehat{\pi}_\omega^*})^{-1}\sqrt{ {(P_\omega-\widehat{P}_\omega)(V_\omega^*)^2} - {((P_\omega V_\omega^*)^2 - (\widehat{P}_\omega V_\omega^*)^2)}+{\text{Var}_{\widehat{P}_\omega}(V_\omega^*)}}\\
        \leq&(I-\gamma \widehat{P}_\omega^{\widehat{\pi}_\omega^*})^{-1}\left(\sqrt{|{(P_\omega-\widehat{P}_\omega)(V_\omega^*)^2}|} + \sqrt{{|(PV^*)^2 - (\widehat{P}_\omega V_\omega^*)^2|}}+\sqrt{{\text{Var}_{\widehat{P}_\omega}(V_\omega^*)}}\right)\\
        =&\underbrace{(I-\gamma \widehat{P}_\omega^{\widehat{\pi}_\omega^*})^{-1}\sqrt{ |{(P_\omega-\widehat{P}_\omega)(V_\omega^*)^2}|}}_{T_1} + \underbrace{(I-\gamma \widehat{P}_\omega^{\widehat{\pi}_\omega^*})^{-1}\sqrt{{|(P_\omega V_\omega^*)^2 - (\widehat{P}_\omega V_\omega^*)^2|}}}_{T_2}\\
         &+\underbrace{(I-\gamma \widehat{P}_\omega^{\widehat{\pi}_\omega^*})^{-1}\sqrt{\text{Var}_{\widehat{P}_\omega}(V_\omega^*)}}_{T_3}.
\end{align*}

This splition allows us to focus on $T_1$, $T_2$ and $T_3$, which can be controlled by the following lemmas:
\begin{lemma}[Proof in Appendix \ref{Proof for Lemma a.8}]
\label{lemma a.8}
    With a probability at least $1-\delta$, $T_1$ satisfies:
    \begin{align*}
    T_1 \leq \frac{1}{(1-\gamma)^2} \sqrt[4]{\frac{2\log\left(2|\mathcal{X}[\cup_{k=1}^{K_\omega}Z^P_k]|\right)}{N}}\cdot \boldsymbol{1}.
\end{align*}
\end{lemma}

\begin{lemma}[Proof in Appendix \ref{Proof for Lemma a.9}]
\label{lemma a.9}
    With a probability at least $1-\delta$, $T_2$ satisfies:
    \begin{align*}
    T_2 \leq \frac{\sqrt{2}}{(1-\gamma)^{2}}\sqrt[4]{\frac{2\log\left(2|\mathcal{X}[\cup_{k=1}^{K_\omega}Z^P_k]|\right)}{N}}\cdot \boldsymbol{1}.
\end{align*}
\end{lemma}
\begin{lemma}[Proof in Appendix \ref{proof_to_lemma_(1-gamma)^3}]
\label{(1-gamma)^3}
     The term $T_3$ can be bounded as follows:
    \begin{align*}
        T_3 \leq \left(\frac{4\gamma}{(1-\gamma)^3}\sqrt{\frac{2\log\left(4|\mathcal{X}[\cup_{k=1}^{K_\omega}Z^P_k]|\right)}{N}} + \sqrt{\frac{1+\gamma}{(1-\gamma)^3}}\right) \cdot \boldsymbol{1}.
    \end{align*}
\end{lemma}

Combining Lemmas \ref{Low_dim_lemma}-\ref{(1-gamma)^3} yields the bound of $\Delta_2$:
\begin{lemma}[Proof in Appendix \ref{proof_lemma_key_bound}]
\label{key_bound}
    With probability at least $1-\delta$, the estimated Q-function $\widehat{Q}^*$ satisfies:
    \begin{align*}
        |\Delta_2| \leq 18 \left(\frac{\log\left(12|\mathcal{X}[\cup_{k=1}^{K_\omega}Z^P_k]|\right)}{N(1-\gamma)^3}\right) + 6\left(\frac{\log\left(12|\mathcal{X}[\cup_{k=1}^{K_\omega}Z^P_k]|\right)}{N(1-\gamma)^3}\right)^{\frac{1}{2}}.
    \end{align*}
\end{lemma}

Analogously, we can show that with probability at least $1-\delta$, $|\Delta_1|$ also satisfies:
\begin{align*}
    |\Delta_1| \leq 18 \left(\frac{\log\left(12|\mathcal{X}[\cup_{k=1}^{K_\omega}Z^P_k]|\right)}{N(1-\gamma)^3}\right) + 6\left(\frac{\log\left(12|\mathcal{X}[\cup_{k=1}^{K_\omega}Z^P_k]|\right)}{N(1-\gamma)^3}\right)^{\frac{1}{2}}.
\end{align*}

Combining the bounds on $|\Delta_1|$ and $|\Delta_2|$, we have that with probability at least $1-\delta$, 
\begin{align*}
    \|Q^*_\omega-\widehat{Q}^*_\omega\|_\infty \leq& \max(|\Delta_1|,|\Delta_2|)\\ \leq& 18 \left(\frac{\log\left(24|\mathcal{X}[\cup_{k=1}^{K_\omega}Z^P_k]|\right)}{N(1-\gamma)^3}\right) + 6\left(\frac{\log\left(24|\mathcal{X}[\cup_{k=1}^{K_\omega}Z^P_k]|\right)}{N(1-\gamma)^3}\right)^{\frac{1}{2}}.
\end{align*}

Taking any $\epsilon = \|Q^*_\omega-\widehat{Q}^*_\omega\|_\infty \leq 1$, we can first verify that $\frac{\log\left(24|\mathcal{X}[\cup_{k=1}^{K_\omega}Z^P_k]|\right)}{N(1-\gamma)^3}\leq 1$. Therefore,
\begin{align*}
    \epsilon \leq (18+6)\left(\frac{\log\left(24|\mathcal{X}[\cup_{k=1}^{K_\omega}Z^P_k]|\right)}{N(1-\gamma)^3}\right)^{\frac{1}{2}}.
\end{align*}

It directly yields the bound for $N$ with probability at least $1-\delta$:
\begin{align}
    N \geq \frac{576\log\left(24|\mathcal{X}[\cup_{k=1}^{K_\omega}Z^P_k]|\right)}{(1-\gamma)^3\epsilon^2}. \label{N_bound_proof}
\end{align}

Substituting Eq. \eqref{N_bound_proof} into Eq. \eqref{def_D} yields:
\begin{align*}
    D_\omega\geq \frac{576 \left( \kappa_p \max_{k \in [K_\omega]} |\mathcal{X}[Z^P_k]| \right) \log \left( 24 |\mathcal{X}[\cup_{k \in [K_\omega]} Z_k^P]| \delta^{-1} \right)}{\epsilon^2 (1-\gamma)^3} + \kappa_r \max_{i \in [\ell_\omega]} |\mathcal{X}[Z^R_i]|.
\end{align*}

Letting $\overline{c}_0 = 576$ and $\overline{c}_1 = 24$ leads to our result.

\section{Proof of Theorem \ref{main_thm}}
\label{proof_to_thm3}
\begin{proof}
Similar to the model-based case, we show the total amount of samples $D_\omega$ should satisfy:
\begin{align}
    D_\omega \geq N_{\text{entry}} \cdot N + D_r, \label{def_D1}
\end{align}
where $N_{\text{entry}}$ denotes the number of sampled state-action pairs for generating a single empirical Bellman operator in Algorithm \ref{Empirical Bellman Operator Generation}, and $N$ denotes the number of generated empirical Bellman operators. Notation $D_r$ denotes the sample complexity of finding the exact reward function.

The term $N_{\text{entry}}$ can be directly obtained in Algorithm \ref{Empirical Bellman Operator Generation}, which equals the sum of numbers of all state-action pairs in all decomposed components, satisfying:
\begin{align*}
    N_{\text{entry}} = {\sum\nolimits_{i\in[\kappa_p]}\max\nolimits_{k\in\mathcal{G}_i}|\mathcal{X}[Z_k^P]|}\leq \sum\nolimits_{k\in[\kappa_p]}|\mathcal{X}[Z_k^P]|,
\end{align*}
where $\kappa_p$ is the total number of sampling sets, $\mathcal{G}_i$ is the set of component indices associated with the \(i\)-th sampling set, \( |\mathcal{X}[Z_k^P]| \) denotes the size of the state-action space for component \( k \).

For the sample complexity associated with estimating the reward function, \( D_r \), it is sufficient to sample all necessary state-action pairs once for each reward component \( r_i \) to obtain the exact reward values. Similar to the transition kernel sampling, the sample complexity of the reward function is bounded by:
\[
D_r \leq  \sum\nolimits_{i\in[\kappa_r]}|\mathcal{X}[Z_i^R]|
\]
where \( \kappa_r \) is the minimal number of sampling sets required to estimate the reward function, and \( |\mathcal{X}[Z_i^R]| \) is the size of the state-action space for reward component \( r_i \).

For the sampling times $N$, according to Algorithm \ref{Variance-Reduced Q-Learning}, we can conclude that:
\begin{align}
    N = \sum\nolimits_{\tau = 1}^T(N_\tau + M) = \sum\nolimits_{\tau = 1}^TN_\tau + MT, \label{def_N_model-free}
\end{align}
where $T$ is the number of epochs, $N_\tau$ is the sampling frequency for estimating the reference Bellman operator in the $\tau$-th epoch, $M$ is the number of variance-reduced updates in a single epoch.

The key of the proof is to show the estimation error decays exponentially when the number of epochs increases, i.e.,
\begin{lemma}[Proof in Appendix \ref{proof_key_lemma_model_free}]\label{key_lemma_model_free}
For any \( \delta \geq 0 \), with probability at least \( 1 - \delta/T \), the Q-function estimate \( \overline{Q}_\tau \) after \( \tau \) epochs satisfies:
\begin{align}
    \|\overline{Q}_\tau - Q^*_\omega\|_\infty \leq \frac{1}{(1-\gamma)2^\tau}, \quad \forall \tau = 1, \ldots, T, \label{key_lemma_res}
\end{align}
provided that the number of samples \( M \) and the number of iterations \( N_\tau \) satisfy the following conditions:
\begin{align}
    M &= c_2 \frac{\log\left(\frac{6T|\mathcal{X}[\cup_{k=1}^{K_\omega} Z_k^P]|}{(1-\gamma)\delta}\right)}{(1-\gamma)^3}, \\
    N_\tau &= c_3 4^\tau \frac{\log\left(6T|\mathcal{X}[\cup_{k=1}^{K_\omega} Z_k^P]|\right)}{(1-\gamma)^2},
\end{align}
where \( \eta_t = \frac{1}{1+(1-\gamma)(t+1)} \), \( c_2, c_3 > 0 \) are sufficiently large constants.
\end{lemma}

We first ensure the error decreases to a relatively small value $\frac{1}{\sqrt{1-\gamma}}$ after $T_1$ epochs. Specifically, we set:
\begin{equation}
    T_1 = \left\lceil \log_2 \left( \frac{1}{\sqrt{1 - \gamma}} \right) \right\rceil.\label{def_T1}
\end{equation}

Using the union bound over all \( \tau \leq T_1 \) in Eq.~\eqref{key_lemma_res}, we obtain that with probability at least \( 1 - \delta \):
\[
\|\overline{Q}_{T_1} - Q^*\|_\infty \leq \frac{1}{\sqrt{1 - \gamma}}.
\]

Substituting \( N_\tau \), \( T_1 \), and \( M \) into Eq.~\eqref{def_N_model-free}, the total number of samples for this phase, denoted \( N_{[1]} \), is bounded by:
\begin{equation}
N_{[1]} \leq c_4 \cdot \frac{\log\left( \frac{6 T_1 |\mathcal{X}[\cup_{k=1}^{K_\omega} Z_k^P]|}{(1 - \gamma) \delta} \right) \cdot \log\left( \frac{1}{1 - \gamma} \right)}{(1 - \gamma)^3},
\end{equation}
where \( c_4 \) is a constant depending on \( c_2 \) and \( c_3 \).

We now show that from initial point $\overline{Q}_{0} = \overline{Q}_{T_1}$, an additional \( T_2 = \left\lceil c \log_2 \left( \frac{1}{\sqrt{1 - \gamma} \epsilon} \right) \right\rceil \) epochs are sufficient to reduce the error to \( \epsilon \), where \( c > 0 \) is a constant.

We utilize the following lemma:

\begin{lemma}[Proof in Appendix~\ref{proof_Refined_analysis}]
\label{Refined_analysis}
Given \( \|\overline{Q}_0 - Q^*\|_\infty \leq \frac{1}{\sqrt{1 - \gamma}} \), and setting:
\begin{align}
M &= c_2 \cdot \frac{\log\left( \frac{6 T |\mathcal{X}[\cup_{k=1}^{K_\omega} Z_k^P]|}{(1 - \gamma) \delta} \right)}{(1 - \gamma)^3}, \\
N_\tau &= c_3 \cdot 4^\tau \cdot \frac{\log\left( 6 T |\mathcal{X}[\cup_{k=1}^{K_\omega} Z_k^P]| \right)}{(1 - \gamma)^2},
\end{align}
then, with probability at least \( 1 - \delta \):
\[
\|\overline{Q}_\tau - Q^*_\omega\|_\infty \leq \frac{1}{2^\tau \sqrt{1 - \gamma}}, \quad \forall \tau \geq 0.
\]
\end{lemma}

The proof of Lemma~\ref{Refined_analysis} follows a similar routine to that of Lemma~\ref{key_lemma_model_free} but leverages the better initial estimate \( \overline{Q}_0 \). This lemma indicates that starting from an initial error of \( \frac{1}{\sqrt{1 - \gamma}} \), the variance-reduced iteration halves the error at each step.

After \( T_2 = \left\lceil c' \log_2 \left( \frac{1}{\sqrt{1 - \gamma} \epsilon} \right) \right\rceil \) epochs (with \( c' > 0 \)), the error reduces to \( \epsilon \). The total number of samples for this phase, denoted \( N_{[2]} \), is bounded by:
\begin{equation}
N_{[2]} \leq c_5 \cdot \frac{\log\left( \frac{6 T_1 |\mathcal{X}[\cup_{k=1}^{K_\omega} Z_k^P]|}{(1 - \gamma) \delta} \right) \cdot \log\left( \frac{1}{1 - \gamma} \right)}{(1 - \gamma)^3 \epsilon^2},
\end{equation}
where \( c_5 \) is a sufficiently large constant.

Combining the samples from both phases, the total number of samples required is:
\[
N = N_{[1]} + N_{[2]}.
\]

Substituting back into Eq.~\eqref{def_D1}, we conclude that the total sample complexity \( D_\omega \) satisfies the desired bound:
\[
D_\omega \geq N_{\text{entry}} \cdot (N_{[1]} + N_{[2]}) + D_r.
\]

This completes the proof.

\end{proof}

\section{Proof for Auxiliary Lemmas}
\subsection{Proof for Lemma \ref{bound_error}}
\label{proof_error_bound}
\begin{proof}
We first write the Bellman equation of $Q^*$ and $Q^*_\omega$ as follows:
\begin{align}
    &Q^*(s,a) = r(s,a) + \gamma\sum_{s'}P(s'|s,a)\max_{a'}Q^*(s',a'),\label{Bellman_Q^*}\\
    &Q_\omega^*(s,a) = r_\omega(s,a) + \gamma\sum_{s'}P_\omega(s'|s,a)\max_{a'}Q_\omega^*(s',a').\label{Bellman_Q^*_omega}
\end{align}

Subtracting Eq. \eqref{Bellman_Q^*} with \eqref{Bellman_Q^*_omega} yields:
\begin{align*}
    Q^*(s,a)-Q_\omega^*(s,a) &= r(s,a)-r_\omega(s,a)+\gamma\sum_{s'}(P(s'|s,a)\max_{a'}Q^*(s',a')-P_\omega(s'|s,a)\max_{a'}Q_\omega^*(s',a'))\\
    &= r(s,a)-r_\omega(s,a)+\gamma\sum_{s'}(P(s'|s,a)(\max_{a'}Q^*(s',a')-\max_{a'}Q_\omega^*(s',a'))\\
    &\quad +(P(s'|s,a)-P_\omega(s'|s,a))\max_{a'}Q_\omega^*(s',a'))
\end{align*}

Taking absolute value on both side yields:
\begin{align}
    |Q^*(s,a)-Q_\omega^*(s,a)| \leq &\|r-r_\omega\|_\infty + \gamma\sum_{s'}(P(s'|s,a)(\max_{a'}Q^*(s',a')-\max_{a'}Q_\omega^*(s',a'))\notag\\
    &+ \gamma\|P-P_\omega\|_\infty \max_{s',a'}Q^*_\omega(s',a')\notag\\
    \leq &\|r-r_\omega\|_\infty + \gamma\|Q^*-Q^*_\omega\|_\infty + \gamma\|P-P_\omega\|_\infty \max_{s',a'}Q^*_\omega(s',a'),
\end{align}
where the last inequality comes from $|\max_a Q^*(s,a)-\max_a Q^*_\omega(s,a)|\leq \max_a |Q^*(s,a)-Q^*_\omega(s,a)|\leq \|Q^*-Q^*_\omega\|_\infty$.
Due to $Q^*_\omega(s',a')\leq \frac{1}{1-\gamma}$ for any state-action pair $(s',a')$, we have:
\begin{align}
    \|Q^*-Q_\omega^*\|_\infty &= \max_{s,a}|Q^*(s,a)-Q_\omega^*(s,a)|\notag\\
    &\leq \|r-r_\omega\|_\infty + \gamma\|Q^*-Q^*_\omega\|_\infty + \frac{\gamma}{1-\gamma}\|P-P_\omega\|_\infty. \label{proof_lemma_Q_omega_bound}
\end{align}

Standard mathematical manipulation on Eq. \eqref{proof_lemma_Q_omega_bound} yields:
\begin{align}
    \|Q^*-Q_\omega^*\|_\infty\leq \frac{\|r-r_\omega\|_\infty}{1-\gamma} + \frac{\gamma\|P-P_\omega\|_\infty}{(1-\gamma)^2}.
\end{align}

Applying the definitions of approximation errors, we have:
\begin{align}
    \|Q^*-Q_\omega^*\|_\infty\leq \frac{\Delta_\omega^R}{1-\gamma} + \frac{\gamma\Delta_\omega^P}{(1-\gamma)^2}.
\end{align}

This concludes our proof.
\end{proof}

\subsection{Proof for Lemma \ref{Low_dim_lemma}}
\label{proof to key lemma1}
\begin{proof}
    We leverage the structure of the factorized transition kernel and show that the vector \((P_\omega - \widehat{P}_\omega) V_\omega^*\) contains multiple identical entries. By identifying the distinct entries, we can focus our analysis on a subset of state-action pairs.
    
    Recall that \((P_\omega - \widehat{P}_\omega)V_\omega^* \in \mathbb{R}^{|\mathcal{S}||\mathcal{A}| \times 1}\) represents the difference between the actual and estimated reward vector.
    Due to the factorized form of the transition kernel \( P_\omega \), the estimated transition probability \( \widehat{P}(s' \mid x) \) is given by:
\begin{align*}
        \widehat{P}(s' \mid x) = \prod_{k=1}^{K_\omega} \widehat{P}_k(s'[Z_k^S] \mid x[Z_k^P]),
    \end{align*}
where each \( \widehat{P}_k(s'[Z_k^S] \mid x[Z_k^P]) \) depends only on the subset \( Z_k^P \) of the state-action pair \( x \).
    As a result, only the state-action components \( x[Z_k^P] \) for \( k \in [K_\omega] \) determine the transition probabilities. Thus, there are at most \( |\mathcal{X}[\cup_{k=1}^{K_\omega} Z_k^P]| \) distinct rows in the matrix \( (P_\omega - \widehat{P}_\omega) \). 

    We use $\mathcal{X}^*$ to denote a set of state-action pairs indicating distinct entries, satisfying:
    \begin{align}
        \mathcal{X}^* = \left\{ x \in \mathcal{X} \ \bigg| \ \forall x', x'' \in \mathcal{X}^*, \ x' \neq x'' \Rightarrow \exists \ k \in [K_\omega] \ \text{such that} \ x'[Z_k^P] \neq x''[Z_k^P] \right\}. \label{def_X^*}
    \end{align}

    For any $x\in \mathcal{X}^*$, we can easily verify that the estimation $\widehat{P}(s' \mid x)$ is unbiased:
    \begin{align*}
        &\mathbb{E}[\widehat{P}(s' \mid x)-{P}(s' \mid x)]\\
        =& \mathbb{E}\left[\prod\nolimits_{i\in[\kappa_p]}(\prod\nolimits_{k\in\mathcal{G}_i}\widehat{P}_k(s'[Z_k^S] \mid x[Z_k^P]))-\prod\nolimits_{i\in[\kappa_p]}(\prod\nolimits_{k\in\mathcal{G}_i}{P}_k(s'[Z_k^S] \mid x[Z_k^P]))\right]\\
        =& \mathbb{E}\left[\prod\nolimits_{i\in[\kappa_p]}\widehat{P}(s'[\cup_{k\in\mathcal{G}_i}Z_k^S] \mid x[\cup_{k\in\mathcal{G}_i}Z_k^P])-\prod\nolimits_{i\in[\kappa_p]}{P}(s'[\cup_{k\in\mathcal{G}_i}Z_k^S] \mid x[\cup_{k\in\mathcal{G}_i}Z_k^P])\right]\\
        =&\prod\nolimits_{i\in[\kappa_p]} \mathbb{E}[\widehat{P}(s'[\cup_{k\in\mathcal{G}_i}Z_k^S] \mid x[\cup_{k\in\mathcal{G}_i}Z_k^P])]-\prod\nolimits_{i\in[\kappa_p]}{P}(s'[\cup_{k\in\mathcal{G}_i}Z_k^S] \mid x[\cup_{k\in\mathcal{G}_i}Z_k^P])\\
        =&0.
    \end{align*}

The third equality is due to the independent sampling of different sampling sets, and the last equality is due to $\mathbb{E}[\widehat{P}(s'[\cup_{k\in\mathcal{G}_i}Z_k^S] \mid x[\cup_{k\in\mathcal{G}_i}Z_k^P])]={P}(s'[\cup_{k\in\mathcal{G}_i}Z_k^S] \mid x[\cup_{k\in\mathcal{G}_i}Z_k^P])$ within each sampling set due to the law of large numbers.
    
    Also, since $\boldsymbol{0}\leq  V^*_\omega \leq \frac{1}{1-\gamma}\cdot \boldsymbol{1}$, we have $\|(P_\omega-\widehat{P}_\omega)V_\omega^*\|_\infty\leq \frac{1}{1-\gamma}$. Combining with the Bernstein's inequality \citep{vershynin2018high} with sample size $N$ yields that:
\begin{align}
    P\left(|(P_\omega(\cdot|x)-\widehat{P}_\omega(\cdot|x))V^*_{\omega}|\geq t\right) \leq 2 \exp\left(-\frac{\frac{N^2}{2}t^2}{N\text{Var}_{P_\omega}(V^*_\omega)_{(x)}+N\frac{t}{3(1-\gamma)}}\right), \label{lemma2_eq1}
\end{align}
where $\text{Var}_{P_\omega}(V^*_\omega)_{(x)}$ denotes the entry of $\text{Var}_{P_\omega}(V^*_\omega)$ corresponding to $x$.

Letting the right-hand-side of Eq. \eqref{lemma2_eq1} equal $\frac{\delta}{|\mathcal{X}[\cup_{k=1}^{K_\omega}Z^P_k]|}$ leads to the following inequality:
\begin{align*}
    \frac{\delta}{|\mathcal{X}[\cup_{k=1}^{K_\omega}Z^P_k]|} = 2 \exp\left(-\frac{\frac{N^2}{2}t^2}{N\text{Var}_{P_\omega}(V_\omega^*)_{(x)}+N\frac{t}{3(1-\gamma)}}\right).
\end{align*}

Rearranging this term yields that
\begin{align*}
    t \leq \frac{2\log(|\mathcal{X}[\cup_{k=1}^{K_\omega}Z^P_k]|)}{3N(1-\gamma)} + {\sqrt{\frac{2\log(|\mathcal{X}[\cup_{k=1}^{K_\omega}Z^P_k]|)\text{Var}_{P_\omega}(V^*_\omega)_{(x)}}{N}}}.
\end{align*}

Taking the union bound across all state-action pairs $x\in\mathcal{X}^*$, and using the identical entry property yields that, with probability at least $1-\delta$:
\begin{align*}
    \left|(P_\omega-\widehat{P}_\omega)V^*\right|\leq& {\sqrt{\frac{2\log(|\mathcal{X}[\cup_{k=1}^{K_\omega}Z^P_k]|)\text{Var}_{P_\omega}(V_\omega^*)}{N}}}+\frac{2\log(|\mathcal{X}[\cup_{k=1}^{K_\omega}Z^P_k]|)}{3N(1-\gamma)}\cdot\boldsymbol{1}.
\end{align*}

This concludes our proof.
\end{proof}

\subsection{Proof for Lemma \ref{lemma a.8}}
\label{Proof for Lemma a.8}
\begin{proof}
    We have that 
    \begin{align*}
        &\|(P_\omega-\widehat{P}_\omega)(V_\omega^*)^2\|_\infty = \left\|(\prod\nolimits_{k=1}^{K_\omega} {P}_k(s'[Z_k^S] \mid x[Z_k^P])-\prod\nolimits_{k=1}^{K_\omega} \widehat{P}_k(s'[Z_k^S] \mid x[Z_k^P]))(V_\omega^*)^2\right\|_\infty.
    \end{align*}
    
    Following the same routine in the proof of Lemma \ref{Low_dim_lemma}, we known $(P_\omega-\widehat{P}_\omega)(V_\omega^*)^2$ contains at most $|\mathcal{X}[\cup_{k=1}^{K_\omega}Z^P_k]|$ distinct entries.

    Meanwhile, for each single entry of $(P_\omega-\widehat{P}_\omega)(V_\omega^*)^2$, denoted by $(P_\omega(\cdot|x)-\widehat{P}_\omega(\cdot|x))(V_\omega^*)^2$, we can show the following conditions:
    \begin{align*}
        &\mathbb{E}((P_\omega(\cdot|x)-\widehat{P}_\omega(\cdot|x))(V_\omega^*)^2) = 0, \forall x \in \mathcal{X},\\
        &\|(P_\omega-\widehat{P}_\omega)(V_\omega^*)^2\|_\infty \leq \|\widehat{P}_\omega(V_\omega^*)^2\|_\infty + \|P_\omega(V_\omega^*)^2\|_\infty\leq \frac{2}{(1-\gamma)^2}.
    \end{align*}
    
    With $N$ \emph{i.i.d.} samples for estimating $(P_\omega(\cdot|x)-\widehat{P}_\omega(\cdot|x))(V_\omega^*)^2\in\mathbb{R}$, we apply the standard Hoeffding's inequality to $(P_\omega(\cdot|x)-\widehat{P}_\omega(\cdot|x))(V_\omega^*)^2$ as follows:
     \begin{align*}
         P\left(|(P_\omega(\cdot|x)-\widehat{P}_\omega(\cdot|x))(V_\omega^*)^2|\geq \epsilon\right) \leq 2 \exp\left(-\frac{2N \epsilon^2}{(\frac{2}{(1-\gamma)^2})^2 \cdot N}\right) = 2 \exp\left(-\frac{(1-\gamma)^4N \epsilon^2}{2}\right), \forall x \in \mathcal{X},
     \end{align*}

     Letting the right-hand-side term be $\frac{\delta}{|\mathcal{X}[\cup_{k=1}^{K_\omega}Z^P_k]|}$ and applying the union bound across all distinct state-action pair $x\in\mathcal{X}^*$ (with definition in Eq. \eqref{def_X^*}) yield:
     \begin{align*}
         \|(P_\omega-\widehat{P}_\omega)(V_\omega^*)^2\|_\infty \leq \frac{1}{(1-\gamma)^2}\cdot\sqrt{\frac{2\log\left(2|\mathcal{X}[\cup_{k=1}^{K_\omega}Z^P_k]|\right)}{N}}.
     \end{align*}

     Thus, $T_1$ satisfies:
     \begin{align*}
         T_1 &= (I-\gamma \widehat{P}^{\widehat{\pi}^*_\omega}_\omega)^{-1}\sqrt{ {|(P_\omega-\widehat{P}_\omega)(V_\omega^*)^2|}}\\
         &\leq\|(I-\gamma \widehat{P}_\omega^{\widehat{\pi}_\omega^*})^{-1}\|_\infty \sqrt{ {\left\|(P_\omega-\widehat{P}_\omega)(V_\omega^*)^2\right\|_\infty}} \cdot \boldsymbol{1}\\
         &\leq \frac{1}{(1-\gamma)^2} \sqrt[4]{\frac{2\log\left(2|\mathcal{X}[\cup_{k=1}^{K_\omega}Z^P_k]|\right)}{N}}\cdot \boldsymbol{1}.
     \end{align*}

     This concludes our proof.
\end{proof}

\subsection{Proof for Lemma \ref{lemma a.9}}
\label{Proof for Lemma a.9}
\begin{proof}
We first bound $\|(P_\omega V_\omega^*)^2 - (\widehat{P}_\omega V_\omega^*)^2\|_\infty$ as follows:
    \begin{align*}
        \|(P_\omega V_\omega^*)^2 - (\widehat{P}_\omega V_\omega^*)^2\|_\infty &= \|(P_\omega V_\omega^*+\widehat{P}_\omega V_\omega^*) (P_\omega V_\omega^*-\widehat{P}_\omega V_\omega^*)\|_\infty\\
        &\leq \|P_\omega V_\omega^*+\widehat{P}_\omega V_\omega^*\|_\infty\|P_\omega V_\omega^*-\widehat{P}_\omega V_\omega^*\|_\infty\\
        &\leq 2\|V_\omega^*\|_\infty\|(P_\omega-\widehat{P}_\omega)V_\omega^*\|_\infty.
    \end{align*}

Applying the Hoeffding's inequality to $\|(P_\omega-\widehat{P}_\omega)V_\omega^*\|_\infty$ yields that, with probability at least $1-\delta$:
\begin{align}
    \|(P_\omega-\widehat{P}_\omega)V^*_\omega\|_\infty &\leq \frac{1}{1-\gamma}\sqrt{\frac{2\log\left(2|\mathcal{X}[\cup_{k=1}^{K_\omega}Z^P_k]|\right)}{N}}. \label{lemma4_key_bound}
\end{align}

Therefore, $T_2$ satisfies:
\begin{align*}
    T_2 &= (I-\gamma \widehat{P}_\omega^{\widehat{\pi}_\omega^*})^{-1}\sqrt{{|(P_\omega V_\omega^*)^2 - (\widehat{P}_\omega V_\omega^*)^2|}} \\
    &\leq \|(I-\gamma \widehat{P}_\omega^{\widehat{\pi}_\omega^*})^{-1}\|_\infty\sqrt{\left\|{(P_\omega V_\omega^*)^2 - (\widehat{P}_\omega V_\omega^*)^2}\right\|_\infty}\cdot \boldsymbol{1}  \\
    &\leq \frac{1}{1-\gamma}\sqrt{2\|V_\omega^*\|_\infty\|(P_\omega-\widehat{P}_\omega)V_\omega^*\|_\infty}\cdot \boldsymbol{1}  \\
    &\leq \frac{\sqrt{2}}{(1-\gamma)^{2}}\sqrt[4]{\frac{2\log\left(2|\mathcal{X}[\cup_{k=1}^{K_\omega}Z^P_k]|\right)}{N}}\cdot \boldsymbol{1}.
\end{align*}
\end{proof}

\subsection{Proof for Lemma \ref{(1-gamma)^3}}
\label{proof_to_lemma_(1-gamma)^3}
\begin{proof}
We decompose $T_3$ as follows:
    \begin{align*}
        T_3 &= (I-\gamma \widehat{P}_\omega^{\widehat{\pi}^*_\omega})^{-1}\sqrt{\text{Var}_{\widehat{P}_\omega}(V_\omega^*)}\\
        &= (I-\gamma \widehat{P}_\omega^{\widehat{\pi}^*_\omega})^{-1}\sqrt{\text{Var}_{\widehat{P}_\omega}(V_\omega^* -\widehat{V}_\omega^{\pi_\omega^*} + \widehat{V}_\omega^{\pi_\omega^*} -\widehat{V}_\omega^{\widehat{\pi}_\omega^*} + \widehat{V}_\omega^{\widehat{\pi}_\omega^*})}\\
    &\leq (I-\gamma \widehat{P}_\omega^{\widehat{\pi}^*_\omega})^{-1}\sqrt{2\text{Var}_{\widehat{P}_\omega}(V_\omega^* -\widehat{V}_\omega^{\pi^*_\omega}) + 2\text{Var}_{\widehat{P}_\omega}(\widehat{V}_\omega^{\widehat{\pi}^*_\omega}) + 2\text{Var}_{\widehat{P}_\omega}(\widehat{V}_\omega^{\widehat{\pi}^*_\omega}-\widehat{V}_\omega^{{\pi}^*_\omega})}\\
    &\leq \underbrace{(I-\gamma \widehat{P}_\omega^{\widehat{\pi}^*_\omega})^{-1}\sqrt{2\text{Var}_{\widehat{P}_\omega}(V_\omega^* -\widehat{V}_\omega^{\pi^*_\omega})}}_{T_{31}} + \underbrace{(I-\gamma \widehat{P}_\omega^{\widehat{\pi}^*_\omega})^{-1}\sqrt{2\text{Var}_{\widehat{P}_\omega}(\widehat{V}_\omega^{\widehat{\pi}^*_\omega})}}_{T_{32}} \\
    & \ \ + \underbrace{(I-\gamma \widehat{P}_\omega^{\widehat{\pi}^*_\omega})^{-1}\sqrt{2\text{Var}_{\widehat{P}_\omega}(\widehat{V}_\omega^{\widehat{\pi}^*_\omega}-\widehat{V}_\omega^{{\pi}^*_\omega})}}_{T_{33}}.
    \end{align*}

Then, we bound $T_{31}$, $T_{32}$ and $T_{33}$ separately. 

\textbf{Step 1. Bounding $T_{31}$:}

For $T_{31}$, the following condition holds:
\begin{align*}
    T_{31} &= (I-\gamma \widehat{P}_\omega^{\widehat{\pi}^*_\omega})^{-1}\sqrt{2\text{Var}_{\widehat{P}_\omega}(V_\omega^* -\widehat{V}_\omega^{\pi^*_\omega})}\\
    & \leq \|(I-\gamma \widehat{P}_\omega^{\widehat{\pi}^*_\omega})^{-1}\|_{\infty}\sqrt{2\|\text{Var}_{\widehat{P}_\omega}(V_\omega^* -\widehat{V}_\omega^{\pi^*_\omega})\|_\infty}\cdot \boldsymbol{1}\\
    & \leq \frac{\sqrt{2}}{1-\gamma}\sqrt{\|V^*_\omega-\widehat{V}_\omega^{\pi^*_\omega}\|^2_\infty}\cdot \boldsymbol{1}\\
    &\leq \frac{\sqrt{2}}{1-\gamma}\sqrt{\|Q_\omega^*-\widehat{Q}_\omega^{\pi^*_\omega}\|^2_\infty}\cdot \boldsymbol{1}.
\end{align*}

Applying Lemma \ref{decomp_lemma1} yields:
    \begin{align*}
        \|Q^*_\omega-\widehat{Q}_\omega^{\pi^*_\omega}\|^2_\infty & = \|\gamma (I-\gamma \widehat{P}_\omega^{{\pi}^*_\omega})^{-1}(P_\omega - \widehat{P}_\omega)V^{*}_\omega\|^2_\infty\\
        &\leq \frac{\gamma^2}{(1-\gamma)^2}\|(P_\omega - \widehat{P}_\omega)V_\omega^{*}\|^2_\infty.
    \end{align*}

Applying Eq. \eqref{lemma4_key_bound}, we have that, with probability at least $1-\delta$, the term $\|Q^*_\omega-\widehat{Q}_\omega^{\pi^*_\omega}\|^2_\infty$ satisfies:
    \begin{align}
        \|Q^*_\omega-\widehat{Q}_\omega^{\pi^*_\omega}\|^2_\infty&\leq \frac{\gamma^2}{(1-\gamma)^4}\left(\sqrt{\frac{2\log\left(2|\mathcal{X}[\cup_{k=1}^{K_\omega}Z^P_k]|\right)}{N}}\right)^2\notag\\
        &=\frac{{2\gamma^2\log\left(2|\mathcal{X}[\cup_{k=1}^{K_\omega}Z^P_k]|\right)}}{N(1-\gamma)^4}.\label{Q_gap_lemma5}
    \end{align}

Therefore, with probability at least $1-\delta$, $T_{31}$ satisfies that:
\begin{align*}
    T_{31} \leq \frac{2\gamma}{(1-\gamma)^3}\sqrt{\frac{2\log\left(2|\mathcal{X}[\cup_{k=1}^{K_\omega}Z^P_k]|\right)}{N}}\cdot \boldsymbol{1}.
\end{align*}
\end{proof}

\textbf{Step 2. Bounding $T_{32}$:}

    Note that, $(1-\gamma)(I-\gamma \widehat{P}_\omega^{\widehat{\pi}^*_\omega})^{-1}$ is a matrix of probability with each row being a probability distribution. For a positive vector $\boldsymbol{v}$ and distribution $\nu$, Jensen's inequality implies that $\nu \sqrt{\boldsymbol{v}}\leq \sqrt{\nu\cdot \boldsymbol{v}}$ (inequality of expectation). This implies:
    \begin{align}
         \|T_{32}\|_\infty &= \left\|(I-\gamma \widehat{P}_\omega^{\widehat{\pi}^*_\omega})^{-1} \sqrt{\text{Var}_{\widehat{P}_\omega}(\widehat{V}_\omega^{\widehat{\pi}^*_\omega})}\right\|_\infty \notag\\ &= \frac{1}{1-\gamma}\left\|(1-\gamma)(I-\gamma \widehat{P}_\omega^{\widehat{\pi}^*_\omega})^{-1} \sqrt{{\text{Var}_{\widehat{P}_\omega}(\widehat{V}_\omega^{\widehat{\pi}^*_\omega})}}\right\|_\infty \notag\\
         &\leq \sqrt{\left\|\frac{1}{1-\gamma}(I-\gamma \widehat{P}_\omega^{\widehat{\pi}^*_\omega})^{-1} {{\text{Var}_{\widehat{P}_\omega}(\widehat{V}_\omega^{\widehat{\pi}^*_\omega})}}\right\|_\infty}.\label{a.6_bound1}
    \end{align}

Also, we can reformulate $\|(I-\gamma \widehat{P}_\omega^{\widehat{\pi}^*_\omega})^{-1}{\text{Var}_{\widehat{P}}(\widehat{V}_\omega^{\widehat{\pi}^*_\omega})}\|_\infty$ as follows:
\begin{align*}
     \|(I-\gamma \widehat{P}_\omega^{\widehat{\pi}^*_\omega})^{-1}{\text{Var}_{\widehat{P}_\omega}(\widehat{V}_\omega^{\widehat{\pi}^*_\omega})}\|_\infty &=  \|(I-\gamma \widehat{P}_\omega^{\widehat{\pi}^*_\omega})^{-1}(I-\gamma^2 \widehat{P}_\omega^{\widehat{\pi}_\omega^*})(I-\gamma^2 \widehat{P}_\omega^{\widehat{\pi}_\omega^*})^{-1}{\text{Var}_{\widehat{P}_\omega}(\widehat{V}_\omega^{\widehat{\pi}^*_\omega})}\|_\infty\\
     &= \|(I-\gamma \widehat{P}_\omega^{\widehat{\pi}^*_\omega})^{-1}(I-\gamma \widehat{P}_\omega^{\widehat{\pi}^*_\omega})(I+\gamma \widehat{P}_\omega^{\widehat{\pi}^*_\omega}) (I-\gamma^2 \widehat{P}_\omega^{\widehat{\pi}_\omega^*})^{-1}{\text{Var}_{\widehat{P}_\omega}(\widehat{V}_\omega^{\widehat{\pi}^*_\omega})}\|_\infty\\
     &=\|(I+\gamma \widehat{P}_\omega^{\widehat{\pi}^*_\omega})(I-\gamma^2 \widehat{P}_\omega^{\widehat{\pi}_\omega^*})^{-1}{\text{Var}_{\widehat{P}_\omega}(\widehat{V}_\omega^{\widehat{\pi}^*_\omega})}\|_\infty\\
     &\leq\|(I+\gamma \widehat{P}_\omega^{\widehat{\pi}^*_\omega})\|_\infty\|(I-\gamma^2 \widehat{P}_\omega^{\widehat{\pi}_\omega^*})^{-1}{\text{Var}_{\widehat{P}_\omega}(\widehat{V}_\omega^{\widehat{\pi}^*_\omega})}\|_\infty\\
     &\leq (1+\gamma) \|(I-\gamma^2 \widehat{P}_\omega^{\widehat{\pi}_\omega^*})^{-1}{\text{Var}_{\widehat{P}_\omega}(\widehat{V}_\omega^{\widehat{\pi}^*_\omega})}\|_{\infty}.
\end{align*}

Thus, we have
\begin{align*}
    \left\|(I-\gamma \widehat{P}^\pi)^{-1} \sqrt{{\text{Var}_{\widehat{P}_\omega}(\widehat{V}_\omega^{\widehat{\pi}^*_\omega})}}\right\|_\infty &\leq \sqrt{\left\|\frac{1+\gamma}{1-\gamma}(I-\gamma^2 \widehat{P}_\omega^{\widehat{\pi}_\omega^*})^{-1} {{\text{Var}_{\widehat{P}_\omega}(\widehat{V}_\omega^{\widehat{\pi}^*_\omega})}}\right\|_\infty}.
\end{align*}

We now connect this result with the definition of $\Sigma_M^\pi$ in Lemma \ref{Lemma 4-basic}:
\begin{align*}
    \Sigma_{{M}}^{\widehat{\pi}^*_\omega} = \gamma^2(1-\gamma^2\widehat{P}_\omega^{\widehat{\pi}^*_\omega})^{-1}{\text{Var}_{\widehat{P}_\omega}(\widehat{V}_\omega^{\widehat{\pi}^*_\omega})},
\end{align*}
where the transition kernel of MDP $M$ is $\widehat{P}_\omega$.
Therefore,
\begin{align*}
    (I-\gamma^2 \widehat{P}_\omega^{\widehat{\pi}_\omega^*})^{-1}{{\text{Var}_{\widehat{P}_\omega}(\widehat{V}_\omega^{\widehat{\pi}_\omega^*})}} = \frac{\Sigma_M^{\widehat{\pi}_\omega^*}}{\gamma^2} \leq \frac{1}{(1-\gamma)^2}\cdot\boldsymbol{1},
\end{align*}
where the inequality comes from $\|\Sigma_M^\pi\|_\infty\leq \frac{\gamma^2}{(1-\gamma)^2}$ for any policy $\pi$, which can be easily verified according to Definition \ref{def_epsilon}:
\begin{align*}
    \|\Sigma^\pi_{M}\|_\infty &= \max_{(s,a)}\left\|\mathbb{E}\left[\left.\left(\sum_{t=0}^\infty \gamma^t r(s_t, a_t)-Q^\pi_{M}(s_0,a_0)\right)^2\right|s_0 = s, a_0 = a\right]\right\|_\infty\\
    &\leq \max_{(s,a)}\mathbb{E}\left[\left.\left\|\sum_{t=0}^\infty \gamma^t r(s_t, a_t)-Q^\pi_{M}(s_0,a_0)\right\|_\infty^2\right|s_0 = s, a_0 = a\right]\\
    &\leq \frac{\gamma^2}{(1-\gamma)^2}.
\end{align*}

Substituting this condition into Eq. \eqref{a.6_bound1} yields:
\begin{align*}
    T_{32} \leq \sqrt{\frac{1+\gamma}{(1-\gamma)^3}}\cdot \boldsymbol{1}.
\end{align*}

\textbf{Step 3. Bounding $T_{33}$:}
The term $T_{33}$ satisfies:
\begin{align*}
    T_{33} &= (I-\gamma \widehat{P}_\omega^{\widehat{\pi}^*_\omega})^{-1}\sqrt{2\text{Var}_{\widehat{P}_\omega}(\widehat{V}_\omega^{\widehat{\pi}^*_\omega}-\widehat{V}_\omega^{{\pi}^*_\omega})}\\
    &\leq \|(I-\gamma \widehat{P}_\omega^{\widehat{\pi}^*_\omega})^{-1}\|_{\infty}\sqrt{2\|\text{Var}_{\widehat{P}_\omega}(\widehat{V}_\omega^{\widehat{\pi}^*_\omega}-\widehat{V}_\omega^{{\pi}^*_\omega})\|_\infty}\cdot \boldsymbol{1}\\
    & \leq \frac{\sqrt{2}}{1-\gamma}\sqrt{\|\widehat{V}_\omega^{\widehat{\pi}^*_\omega}-\widehat{V}_\omega^{{\pi}^*_\omega}\|^2_\infty}\cdot \boldsymbol{1}\\
    &\leq \frac{\sqrt{2}}{1-\gamma}\sqrt{\|\widehat{Q}_\omega^{*}-\widehat{Q}_\omega^{{\pi}^*_\omega}\|^2_\infty}\cdot \boldsymbol{1}.
\end{align*}

Similar to the bound of $T_{31}$ in Eq. \eqref{Q_gap_lemma5}, we have that with probability at least $1-\delta$,
\begin{align*}
        \|\widehat{Q}_\omega^{*}-\widehat{Q}_\omega^{{\pi}^*_\omega}\|^2_\infty\leq \frac{{2\gamma^2\log\left(2|\mathcal{X}[\cup_{k=1}^{K_\omega}Z^P_k]|\right)}}{N(1-\gamma)^4}.
    \end{align*}

Thus, with probability at least $1-\delta$, $T_{33}$ satisfies:
\begin{align*}
    T_{33} \leq \frac{2\gamma}{(1-\gamma)^3}\sqrt{\frac{2\log\left(2|\mathcal{X}[\cup_{k=1}^{K_\omega}Z^P_k]|\right)}{N}}\cdot \boldsymbol{1}.
\end{align*}

\textbf{Step 4. Combining the Results:}

Combining the upper bounds for $T_{31}$, $T_{32}$ and $T_{33}$, we can bound $T_3$ that, with probability at least $1-\delta$,
\begin{align*}
    T_3 \leq \Bigg(\frac{4\gamma}{(1-\gamma)^3}\sqrt{\frac{2\log\left(4|\mathcal{X}[\cup_{k=1}^{K_\omega}Z^P_k]|\right)}{N}} + \sqrt{\frac{1+\gamma}{(1-\gamma)^3}}\Bigg) \cdot \boldsymbol{1}.
\end{align*}

\subsection{Proof for Lemma \ref{key_bound}}
\label{proof_lemma_key_bound}
\begin{proof}
    Taking $\delta$ to be $\frac{\delta}{3}$, and applying Lemma \ref{Low_dim_lemma}-\ref{(1-gamma)^3} yields:
\begin{align}
    |\Delta_2| \leq& \frac{3}{(1-\gamma)^3}\left(\frac{\log\left(12|\mathcal{X}[\cup_{k=1}^{K_\omega}Z^P_k]|\right)}{N}\right) + \frac{7}{(1-\gamma)^2}\left(\frac{\log\left(12|\mathcal{X}[\cup_{k=1}^{K_\omega}Z^P_k]|\right)}{N}\right)^{\frac{3}{4}}\notag\\
    &+ \sqrt{\frac{2}{(1-\gamma)^3}}\left(\frac{\log\left(12|\mathcal{X}[\cup_{k=1}^{K_\omega}Z^P_k]|\right)}{N}\right)^{1/2}.\label{lemma_6-eq1}
\end{align}

Applying Cauchy–Schwarz inequality yields:
\begin{align*}
    &\frac{3}{(1-\gamma)^3}\left(\frac{\log\left(12|\mathcal{X}[\cup_{k=1}^{K_\omega}Z^P_k]|\right)}{N}\right) + \sqrt{\frac{2}{(1-\gamma)^3}}\left(\frac{\log\left(12|\mathcal{X}[\cup_{k=1}^{K_\omega}Z^P_k]|\right)}{N}\right)^{1/2}\\
    \geq & \sqrt{\frac{3}{(1-\gamma)^3}\left(\frac{\log\left(12|\mathcal{X}[\cup_{k=1}^{K_\omega}Z^P_k]|\right)}{N}\right) \cdot \sqrt{\frac{2}{(1-\gamma)^3}}\left(\frac{\log\left(12|\mathcal{X}[\cup_{k=1}^{K_\omega}Z^P_k]|\right)}{N}\right)^{1/2}}\\
    = & \sqrt{3}(1-\gamma)^{-\frac{9}{4}}\left(\frac{\log\left(12|\mathcal{X}[\cup_{k=1}^{K_\omega}Z^P_k]|\right)}{N}\right)^{\frac{3}{4}}\\
    \geq &\frac{\sqrt{3}}{(1-\gamma)^2}\left(\frac{\log\left(12|\mathcal{X}[\cup_{k=1}^{K_\omega}Z^P_k]|\right)}{N}\right)^{\frac{3}{4}}.
\end{align*}

Applying this condition to Eq. \eqref{lemma_6-eq1} yields:
\begin{align*}
    \|Q_\omega^*-\widehat{Q}_\omega^*\|_\infty \leq& \frac{3(1+\frac{7}{\sqrt{3}})}{(1-\gamma)^3}\left(\frac{\log\left(12|\mathcal{X}[\cup_{k=1}^{K_\omega}Z^P_k]|\right)}{N}\right) + \frac{\sqrt{2}(1+\frac{7}{\sqrt{3}})}{\sqrt{(1-\gamma)^3}}\left(\frac{\log\left(12|\mathcal{X}[\cup_{k=1}^{K_\omega}Z^P_k]|\right)}{N}\right)^{1/2}\\
    \leq& 18 \left(\frac{\log\left(12|\mathcal{X}[\cup_{k=1}^{K_\omega}Z^P_k]|\right)}{N(1-\gamma)^3}\right) + 6\left(\frac{\log\left(12|\mathcal{X}[\cup_{k=1}^{K_\omega}Z^P_k]|\right)}{N(1-\gamma)^3}\right)^{\frac{1}{2}}.
\end{align*}

This concludes our proof.

\end{proof}

\subsection{Proof for Lemma \ref{key_lemma_model_free}}
\label{proof_key_lemma_model_free}

We prove this lemma by induction. We first show that the base case (\(\tau = 1\)) satisfies Eq. \eqref{key_lemma_res}, and then prove the inductive condition when $\tau \geq 2$.

\textbf{Step 1. Showing Base Case with \(\tau = 1\):}

Due to the initialization, we have \(\overline{Q}_{\tau-1}(s,a) = \overline{Q}_{0}(s,a) = 0\) for any state-action pair \((s, a)\). Therefore, for \(\tau = 1\), both the empirical and reference Bellman operators equal to the immediate reward:
\[
\widehat{\mathcal{H}}_t(\overline{Q}_{\tau-1})_{s,a} = r(s,a) \quad \text{and} \quad \overline{\mathcal{H}}_\tau(\overline{Q}_{\tau-1})_{s,a} = r(s,a), \quad \forall (s,a).
\]
Given this, the variance-reduced update in Eq. \eqref{variance-reduced Q update} simplifies to the standard Q-learning update:
\[
Q_t = (1 - \eta_{t-1})Q_{t-1} + \eta_{t-1} \widehat{\mathcal{H}}_{t-1}(Q_{t-1}), \quad \forall t = 1, \ldots, M.
\]

Let \(\Delta_t = Q_t - Q^*_\omega\) denote the estimation error at iteration \(t\), where \(Q^*_\omega\) is the unique fixed-point solution of the Bellman equation \(\mathcal{H}(Q) = Q\). Then, we have:
\begin{align}
    \Delta_t &= (1 - \eta_{t-1})\Delta_{t-1} + \eta_{t-1}(\widehat{\mathcal{H}}_{t-1}(Q^*_\omega + \Delta_{t-1}) - \mathcal{H}(Q^*_\omega)) \notag \\
    &= (1 - \eta_{t-1})\Delta_{t-1} + \eta_{t-1}(\widehat{\mathcal{H}}_{t-1}(Q^*_\omega + \Delta_{t-1}) - \widehat{\mathcal{H}}_{t-1}(Q^*_\omega) + \widehat{\mathcal{H}}_{t-1}(Q^*_\omega) - \mathcal{H}(Q^*_\omega)) \notag \\
    &= (1 - \eta_{t-1})\Delta_{t-1} + \eta_{t-1}\underbrace{(\widehat{\mathcal{H}}_{t-1}(Q^*_\omega + \Delta_{t-1}) - \widehat{\mathcal{H}}_{t-1}(Q^*_\omega))}_{\text{Contractive Error: } \mathcal{W}_{t-1}(\Delta_{t-1})} + \eta_{t-1}\underbrace{(\widehat{\mathcal{H}}_{t-1}(Q^*_\omega) - \mathcal{H}(Q^*_\omega))}_{\text{Random Error: } \mathcal{E}_{t-1}}. \label{Iteration}
\end{align}

Observe that the error iteration \(\Delta_t\) consists of three components: the decaying term \((1 - \eta_{t-1})\Delta_{t-1}\), the contractive error term \(\mathcal{W}_{t-1}(\Delta_{t-1})\), and the random error \(\mathcal{E}_{t-1}\).

The contractive error \(\mathcal{W}_{t-1}(\Delta_{t-1})\) depends on both \(\Delta_{t-1}\) and the empirical Bellman operator \(\widehat{\mathcal{H}}_{t-1}\), and it is bounded due to the \(\gamma\)-contractiveness of \(\widehat{\mathcal{H}}_{t-1}\):
\begin{align}
    \|\mathcal{W}_{t-1}(\Delta_{t-1})\|_\infty \leq \gamma \|\Delta_{t-1}\|_\infty. \label{contraction_W_t}
\end{align}

The random error \(\mathcal{E}_{t-1}\) is \emph{i.i.d.} for different values of \(t\).
By applying the iteration in Eq. \eqref{Iteration} and using the contraction condition in Eq. \eqref{contraction_W_t}, we can express \(\Delta_t\) fully as follows:

\begin{lemma}[Proof in Appendix \ref{proof_lemma_2}]
\label{lemma_2}
For any \(t \geq 1\), the estimation error \(\Delta_t\) is bounded above and below by:
\begin{align}
        \Delta_t &\leq \prod_{k=0}^{t-1}(1-(1-\gamma)\eta_k)\|\Delta_0\|_\infty \boldsymbol{1}+\gamma\eta_{t-1}\|P_{t-1}\|_\infty\boldsymbol{1}+\gamma\sum_{i=1}^{t-2}\left(\left(\prod_{j=i+1}^{t-1}(1-(1-\gamma)\eta_j)\right)\eta_i\|P_i\|_\infty\right) \boldsymbol{1}+P_t,\label{upper1}\\
        \Delta_t &\geq -\prod_{k=0}^{t-1}(1-(1-\gamma)\eta_k)\|\Delta_0\|_\infty \boldsymbol{1}-\gamma\eta_{t-1}\|P_{t-1}\|_\infty\boldsymbol{1}-\gamma\sum_{i=1}^{t-2}\left(\left(\prod_{j=i+1}^{t-1}(1-(1-\gamma)\eta_j)\right)\eta_i\|P_i\|_\infty\right) \boldsymbol{1}+P_t,\label{lower1}
    \end{align}
where \(P_t\) represents a discounted sum of the random error \(\mathcal{E}_t\), defined as:
\begin{equation}
    P_t =
    \begin{cases}
        \boldsymbol{0} & \text{if } t = 0, \\
        \sum_{k=0}^{t-1}\left(\left(\prod_{j=k+1}^{t-1}(1 - \eta_j)\right)\eta_k\mathcal{E}_k\right) & \text{if } t \geq 1.
    \end{cases} \label{eq-Pt-1}
\end{equation}
\end{lemma}

To effectively manage the product term \(\prod_{k=0}^{t-1}(1 - (1 - \gamma)\eta_k)\), we apply a step size defined as \(\eta_k = \frac{1}{1 + (1 - \gamma)(k+1)}\). Under this choice of step size, we can demonstrate that the expression \(1 - (1 - \gamma)\eta_k\) satisfies:
\begin{align*}
    1 - (1 - \gamma)\eta_k &= 1 - \frac{1 - \gamma}{1 + (1 - \gamma)(k+1)} \\
    &= \frac{1 + (1 - \gamma)k}{1 + (1 - \gamma)(k+1)} \\
    &= \frac{\eta_k}{\eta_{k-1}},\quad \forall k\geq 1.
\end{align*}

Using this result, we can simplify the product-form coefficients in Eq. \eqref{upper1} as follows:
\begin{align}
    &\prod_{k=0}^{t-1}(1 - (1 - \gamma)\eta_k) = (1-(1-\gamma)\eta_0)\prod_{k=1}^{t-1}\frac{\eta_k}{\eta_{k-1}}= \eta_{t-1}, \quad \forall t, \label{coef_1_value} \\
    &\left(\prod_{j=i+1}^{t-1}(1 - (1 - \gamma)\eta_j)\right)\eta_i = \left(\prod_{j=i+1}^{t-1}\frac{\eta_j}{\eta_{j-1}}\right)\eta_i = \eta_{t-1}, \quad \forall t. \label{coef_2_value}
\end{align}

Substituting Eq. \eqref{coef_1_value} and \eqref{coef_2_value} into the bounds in Eq. \eqref{upper1} and \eqref{lower1}, we obtain:
\begin{align}
    \|\Delta_t\|_\infty \leq \eta_{t-1}\|\Delta_0\|_\infty + \gamma \eta_{t-1} \sum_{k=0}^{t-1} \|P_k\|_\infty + \|P_t\|_\infty. \label{Delta_t_bound}
\end{align}

To complete the proof, we need to control \(\|P_k\|_\infty\) for each \(k\). The following lemma provides the necessary bound:
\begin{lemma}[Proof in Appendix \ref{proof_lemma_3}]
\label{lemma_3}
    Given the step size \(\eta_k = \frac{1}{1 + (1 - \gamma)k}\), for any \(k \geq 1\), with probability at least \(1 - \delta\), the weighted error sum \(P_k\) satisfies:
    \begin{align}
        \|P_k\|_\infty \leq \frac{2}{3(1-\gamma)(1+(1 - \gamma)k)} \log\left(\frac{2|\mathcal{X}[\cup_{k=1}^{K_\omega} Z_k^P]|}{\delta}\right) + \frac{2\sqrt{2\|\sigma^2_{\mathcal{E}}\|_\infty \log\left(\frac{2|\mathcal{X}[\cup_{k=1}^{K_\omega} Z_k^P]|}{\delta}\right)}}{\sqrt{1 + (1 - \gamma)k}},\label{eq_P_k}
    \end{align}
    where \(\sigma^2_{\mathcal{E}}\) is the variance of each random error \(\mathcal{E}_t\), satisfying $\sigma^2_\mathcal{E}(s,a) = \mathrm{Var}\left(\widehat{\mathcal{H}}(Q^*_\omega)(s,a)\right)$.
\end{lemma}

This lemma shows that \(\|P_k\|_\infty\) is decreasing in the iteration number $k$. 
Substituting the bound on $\|P_k\|_\infty$ into Eq. \eqref{Delta_t_bound} yields the following results:
\begin{lemma}[Proof in Appendix \ref{proof_lemma_4}]
\label{lemma_4}
    With probability at least $1-\delta$, the estimation error $\Delta_M$ after $M$ iterations in epoch $\tau = 1$ satisfies:
    \begin{align*}
    \|\Delta_M\|_\infty \leq& \frac{3+2\log\left(\frac{2M|\mathcal{X}[\cup_{k=1}^{K_\omega} Z_k^P]|}{\delta}\right)}{3(1-\gamma)^2M}+\frac{2\log\left(\frac{2M|\mathcal{X}[\cup_{k=1}^{K_\omega} Z_k^P]|}{\delta}\right)\log(1+(1-\gamma)M)}{3(1-\gamma)^3M}\\
    &+\frac{{6\sqrt{2\|\sigma^2_{\mathcal{E}}\|_\infty\log\left(\frac{2M|\mathcal{X}[\cup_{k=1}^{K_\omega} Z_k^P]|}{\delta}\right)}}}{(1-\gamma)^{\frac{3}{2}}M^{\frac{1}{2}}},
\end{align*}
\end{lemma}

Let $M = c_2\frac{\log\left(\frac{6T|\mathcal{X}[\cup_{k=1}^{K_\omega} Z_k^P]|}{(1-\gamma)\delta}\right)}{(1-\gamma)^3}$ with a sufficiently large $c_2$, we have:
\begin{align}
    \|\overline{Q}_1 - Q^*_\omega\|_\infty = \|\Delta_M\|_\infty \leq \frac{\sqrt{\|\sigma^2_\mathcal{E}\|_\infty} + 1}{2} \leq \frac{1}{2(1-\gamma)}
\end{align}
holds with probability at least $1-\frac{\delta}{T}$, where the last inequality is due to $\|\sigma^2_\mathcal{E}\|_\infty\leq \frac{\gamma^2}{(1-\gamma)^2}$.
This finishes the proof of the basic case with $\tau = 1$.

\textbf{Step 2. Showing the Inductive Case with \(\tau \geq 2\):}

We assume that the input $\overline{Q}_\tau$ in epoch $\tau$ satisfies the bound
\[
\|\overline{Q}_\tau - Q^*_\omega\|_\infty \leq \frac{\sqrt{\|\sigma^2_{\mathcal{E}}\|_\infty} + 1}{2^\tau} := b_\tau,
\]
and our goal is to show that $\|\overline{Q}_{\tau+1} - Q^*\|_\infty \leq \frac{b_\tau}{2}$ with probability at least $1-\frac{\delta}{T}$. 

Specifically, $\overline{Q}_{\tau+1}$ is equivalent to the output $Q_M$ of running $M$ rounds of variance-reduced Q-learning from the initialization $Q_0 = \overline{Q}_{\tau}$. The reference Bellman operator $\overline{\mathcal{H}}_\tau$ is obtained using $N_\tau$ empirical samples.

The variance-reduced update can be rewritten into:
\begin{align*}
Q_{t+1}-Q^* = {(1-\eta_t)(Q_{t}-Q^*)} + \eta_t\underbrace{(\widehat{\mathcal{H}}_t(Q_{t})-\widehat{\mathcal{H}}_t(Q^*))}_{\text{Contractive Error: } \mathcal{W}_t(\Delta_t)} + \eta_t\underbrace{(\widehat{\mathcal{H}}_t(Q^*) - \mathcal{H}(Q^*_\omega) - \widehat{\mathcal{H}}_{t}(\overline{Q}_{\tau}) + \overline{\mathcal{H}}_{\tau+1}(\overline{Q}_{\tau}))}_{\text{Variance-reduced Error: } \overline{\mathcal{E}}_t}.
\end{align*}

Observe that, the update form contains three components: the decaying term $(1-\eta_t)(Q_{t}-Q^*)$, the contractive error term $\mathcal{W}_t(\Delta_t)$ and the variance reduced error term $\overline{\mathcal{E}}_t$. The only difference between the variance-reduced iteration and the vanilla iteration comes from the variance-reduced error $\overline{\mathcal{E}}_t$, which can be further split into the following form:
\begin{align*}
    &\widehat{\mathcal{H}}_t(Q^*) - \mathcal{H}(Q^*_\omega) - \widehat{\mathcal{H}}_{t}(\overline{Q}_{\tau}) + \overline{\mathcal{H}}_{\tau+1}(\overline{Q}_{\tau})\\
    =&\widehat{\mathcal{H}}_t(Q^*)  - \widehat{\mathcal{H}}_{t}(\overline{Q}_{\tau}) + \overline{\mathcal{H}}_{\tau+1}(\overline{Q}_{\tau})-\overline{\mathcal{H}}_{\tau+1}({Q}^*_\omega)+\overline{\mathcal{H}}_{\tau+1}({Q}^*_\omega)- \mathcal{H}(Q^*_\omega)\\
    =&\underbrace{\mathcal{H}(\overline{Q}_{\tau})-\mathcal{H}({Q}^*_\omega)-\widehat{\mathcal{H}}_t(\overline{Q}_{\tau})+\widehat{\mathcal{H}}_t({Q}^*_\omega)}_{\overline{\mathcal{E}}^a_t}+\underbrace{\overline{\mathcal{H}}_{\tau+1}(\overline{Q}_{\tau})-\overline{\mathcal{H}}_{\tau+1}({Q}^*_\omega)-\mathcal{H}(\overline{Q}_{\tau})+\mathcal{H}({Q}^*_\omega)}_{\overline{\mathcal{E}}^b} + \underbrace{\overline{\mathcal{H}}_{\tau+1}({Q}^*_\omega)- \mathcal{H}(Q^*_\omega)}_{\overline{\mathcal{E}}^c}
\end{align*}

An important observation is that, $\overline{\mathcal{E}}^a_t$ depends on the empirical Bellman operator sampled in each iteration $t$, while $\overline{\mathcal{E}}^b$ and $\overline{\mathcal{E}}^c$ are independent of each iteration, but only dependent on the reference Bellman operator $\overline{T}_\tau$ sampled in the begining of iteration.

We apply the result in Lemma \ref{lemma_2} for the bound on the error accumulation in iterative updates, and get the following bound on $\Delta_t = Q_t-Q^*$:
\begin{align}
    \|\Delta_t\|_\infty &\leq \eta_{t-1}\|\Delta_0\|_\infty+\gamma\eta_{t-1}\sum\nolimits_{k=0}^{t-1}(\|P^a_k\|_\infty+\|P^b_k\|_\infty+\|P^c_k\|_\infty) + \|P^a_t\|_\infty+\|P^b_t\|_\infty+\|P^c_t\|_\infty,
\end{align}
where $P_t^a$, $P_t^b$ and $P_t^c$ are the discounted sums of $\overline{E}_t^a$, $\overline{E}^b$ and $\overline{E}^c$ satisfying:
\begin{equation}
    P^a_t =
    \begin{cases}
        \boldsymbol{0} & \text{if } t = 0, \\
        \sum_{k=0}^{t-1}\left(\left(\prod_{j=k+1}^{t-1}(1 - \eta_j)\right)\eta_k\overline{\mathcal{E}}^a_k\right) & \text{if } t \geq 1.
    \end{cases} \label{eq-Pt}
\end{equation}
\begin{equation}
    P^b_t =
    \begin{cases}
        \boldsymbol{0} & \text{if } t = 0, \\
        \sum_{k=0}^{t-1}\left(\left(\prod_{j=k+1}^{t-1}(1 - \eta_j)\right)\eta_k\overline{\mathcal{E}}^b\right) & \text{if } t \geq 1.
    \end{cases} \label{eq-Ptb}
\end{equation}
\begin{equation}
    P^c_t =
    \begin{cases}
        \boldsymbol{0} & \text{if } t = 0, \\
        \sum_{k=0}^{t-1}\left(\left(\prod_{j=k+1}^{t-1}(1 - \eta_j)\right)\eta_k\overline{\mathcal{E}}^c\right) & \text{if } t \geq 1.
    \end{cases} \label{eq-Ptc}
\end{equation}

To simplify, we have that $\|P^b_t\|_\infty \leq \|\overline{\mathcal{E}}^b\|_\infty$ and $\|P^c_t\|_\infty \leq \|\overline{\mathcal{E}}^c\|_\infty$.
Since $\|\overline{\mathcal{E}}^b\|_\infty$ and $\|\overline{\mathcal{E}}^c\|_\infty$ are independent of $t$, by substituting $\eta_{t-1} = \frac{1}{1+(1-\gamma)t}$, we have
\begin{align*}
    &\gamma\eta_{t-1}\sum\nolimits_{k=0}^{t-1}(\|\overline{\mathcal{E}}^b\|_\infty+\|\overline{\mathcal{E}}^c\|_\infty) + \|\overline{\mathcal{E}}^b\|_\infty+\|\overline{\mathcal{E}}^c\|_\infty\\
    =&\left(1+\frac{\gamma t}{1+(1-\gamma)t}\right)(\|\overline{\mathcal{E}}^b\|_\infty+\|\overline{\mathcal{E}}^c\|_\infty)\\
    \leq & \frac{2}{1-\gamma} (\|\overline{\mathcal{E}}^b\|_\infty+\|\overline{\mathcal{E}}^c\|_\infty).
\end{align*}

Hence, $\Delta_t$ satisfies:
\begin{align}
    \|\Delta_t\|_\infty &\leq \eta_{t-1}\|\Delta_0\|_\infty+\gamma\eta_{t-1}\sum\nolimits_{k=0}^{t-1}\|P^a_k\|_\infty + \|P^a_t\|_\infty+\frac{2}{1-\gamma}(\|\overline{\mathcal{E}}^b\|_\infty+\|\overline{\mathcal{E}}^c\|_\infty).
\end{align}

The remaining challenge is to bound the following three terms: $\|\overline{\mathcal{E}}^b\|$, $\|\overline{\mathcal{E}}^c\|$ and $\|P_k^a\|$.

\begin{lemma}[Proof in Appendix \ref{proof_lemma_6}]
\label{lemma_6}
    In $\tau$-th iteration, with probability at least $1-\frac{\delta}{3T}$, $\overline{\mathcal{E}}^b$ satisfies:
    \begin{align*}
        \|\overline{\mathcal{E}}^b\|_\infty \leq b_\tau \sqrt{\frac{2\log\left(\frac{6T|\mathcal{X}[\cup_{k=1}^{K_\omega}Z_k^P]|}{\delta}\right)}{N_\tau}}.
    \end{align*}
\end{lemma}

\begin{lemma}[Proof in Appendix \ref{proof_lemma_5}]
\label{lemma_5}
    In $\tau$-th iteration, with probability at least $1-\frac{\delta}{3T}$, $\overline{\mathcal{E}}^c$ satisfies:
    \begin{align*}
        \|\overline{\mathcal{E}}^c\|_\infty \leq \overline{c}_1\sqrt{\frac{\log\left(\frac{2T|\mathcal{X}[\cup_{k=1}^{K_\omega} Z_k^P]|}{\delta}\right)}{N_\tau}}\left(\sqrt{\|\sigma^2_{\mathcal{E}}\|_\infty}+1\right),
    \end{align*}
    where $\overline{c}_1$ is an absolute constant.
\end{lemma}

\begin{lemma}[Proof in Appendix \ref{proof_lemma_7}]
\label{lemma_7}
In the $\tau$-th iteration, with probability at least $1-\frac{\delta}{3T}$, it satisfies:
    \begin{align*}
        \gamma\eta_{t-1}\sum_{k=0}^{t-1}\|P^a_k\|_\infty + \|P^a_t\|_\infty\leq \overline{c}_2b_\tau\left(\frac{\log\left(\frac{6TM|\mathcal{X}[\cup_{k=1}^{K_\omega} Z_k^P]|}{\delta}\right)\log(1+(1-\gamma)t)}{(1-\gamma)^2t}+\frac{\sqrt{\log\left(\frac{6TM|\mathcal{X}[\cup_{k=1}^{K_\omega} Z_k^P]|}{\delta}\right)}}{(1-\gamma)^{\frac{3}{2}}t^{\frac{1}{2}}}\right),
    \end{align*}
    where $\overline{c}_2$ is an absolute constant.
\end{lemma}

Combining Lemmas \ref{lemma_6}, \ref{lemma_5} and \ref{lemma_7} yields:
\begin{align*}
    \|\Delta_M\|_\infty \leq& \frac{b_\tau}{1+(1-\gamma)M}+\frac{\overline{c}_1\sqrt{\frac{\log\left(\frac{2T|\mathcal{X}[\cup_{k=1}^{K_\omega} Z_k^P]|}{\delta}\right)}{N_\tau}}\left(\sqrt{\|\sigma^2_{\mathcal{E}}\|_\infty}+1\right)+b_\tau\sqrt{\frac{\log\left(\frac{2T|\mathcal{X}[\cup_{k=1}^{K_\omega} Z_k^P]|}{\delta}\right)}{N_\tau}}}{1-\gamma}\\
    &+\overline{c}_2b_\tau\left(\frac{\log\left(\frac{6TM|\mathcal{X}[\cup_{k=1}^{K_\omega} Z_k^P]|}{\delta}\right)\log(1+(1-\gamma)M)}{(1-\gamma)^2M}+\frac{\sqrt{\log\left(\frac{6TM|\mathcal{X}[\cup_{k=1}^{K_\omega} Z_k^P]|}{\delta}\right)}}{(1-\gamma)^{\frac{3}{2}}M^{\frac{1}{2}}}\right).
\end{align*}

Substituting $M = c_2\frac{\log\left(\frac{6T|\mathcal{X}[\cup_{k=1}^{K_\omega} Z_k^P]|}{(1-\gamma)\delta}\right)}{(1-\gamma)^3}$ and $N_\tau = c_34^\tau\frac{\log\left(6T|\mathcal{X}[\cup_{k=1}^{K_\omega} Z_k^P]|\right)}{(1-\gamma)^2}$ with large enough $c_2$ and $c_3$ yields that:
\begin{align*}
    \|\Delta_M\|_\infty \leq \frac{\sqrt{\|\sigma^2_{\mathcal{E}}\|_\infty}+1}{2^{\tau+1}}\leq \frac{1}{(1-\gamma)2^{\tau+1}}.
\end{align*}

This concludes our proof.

\subsection{Proof to Lemma \ref{lemma_2}}
\label{proof_lemma_2}
\begin{proof}
    We will prove the result by induction. First, by applying Eq. \eqref{Iteration}, we have that \(\Delta_1\) satisfies:
    \begin{align}
        \Delta_1 &= (1-\eta_0)\Delta_0 + \eta_0(\widehat{\mathcal{H}}_0(Q^*_\omega + \Delta_0) - \widehat{\mathcal{H}}_0(Q^*_\omega)) + \eta_0(\widehat{\mathcal{H}}_0(Q^*_\omega) - \mathcal{H}(Q^*_\omega)) \notag \\
        &\leq (1 - \eta_0)\|\Delta_0\|_\infty \boldsymbol{1} + \eta_0\gamma \|\Delta_0\|_\infty \boldsymbol{1} + \eta_0 \mathcal{E}_0 \notag \\
        &= (1 - (1 - \gamma)\eta_0)\|\Delta_0\|_\infty \boldsymbol{1} + \eta_0 \mathcal{E}_0 \notag \\
        &= (1 - (1 - \gamma)\eta_0)\|\Delta_0\|_\infty \boldsymbol{1} + P_1, \label{lm2-cond1}
    \end{align}
    where the inequality follows from the contraction property \(\widehat{\mathcal{H}}_0(Q^*_\omega + \Delta_0) - \widehat{\mathcal{H}}_0(Q^*_\omega) \leq \gamma \|\Delta_0\|_\infty \boldsymbol{1}\).

    Similarly, \(\Delta_1\) can be lower bounded as follows:
    \begin{align}
        \Delta_1 &= (1-\eta_0)\Delta_0 + \eta_0(\widehat{\mathcal{H}}_0(Q^*_\omega + \Delta_0) - \widehat{\mathcal{H}}_0(Q^*_\omega)) + \eta_0(\widehat{\mathcal{H}}_0(Q^*_\omega) - \mathcal{H}(Q^*_\omega)) \notag \\
        &\geq -(1 - \eta_0)\|\Delta_0\|_\infty \boldsymbol{1} - \eta_0\gamma \|\Delta_0\|_\infty \boldsymbol{1} + \eta_0 \mathcal{E}_0 \notag \\
        &= -(1 - (1 - \gamma)\eta_0)\|\Delta_0\|_\infty \boldsymbol{1} + \eta_0 \mathcal{E}_0 \notag \\
        &= -(1 - (1 - \gamma)\eta_0)\|\Delta_0\|_\infty \boldsymbol{1} + P_1. \label{lm2-cond2}
    \end{align}

    The expressions in Eq. \eqref{lm2-cond1} and \eqref{lm2-cond2} are consistent with the upper and lower bounds provided by Eq. \eqref{upper1} and \eqref{lower1} in Lemma \ref{lemma_2}.

    Now, suppose that the conditions in Eq. \eqref{upper1} and \eqref{lower1} hold for \(\Delta_t\), that is:
    \begin{align}
        \Delta_t &\leq a_t\|\Delta_0\|_\infty \boldsymbol{1} + b_t \boldsymbol{1} + P_t, \label{simp_eq1} \\
        \Delta_t &\geq -a_t\|\Delta_0\|_\infty \boldsymbol{1} - b_t \boldsymbol{1} + P_t, \label{simp_eq2}
    \end{align}
    where 
    \[
    a_t = \prod_{k=0}^{t-1}(1 - (1 - \gamma)\eta_k), \quad b_t = \gamma \eta_{t-1}\|P_{t-1}\|_\infty + \gamma \sum_{i=1}^{t-2}\left(\left(\prod_{j=i+1}^{t-1}(1 - (1 - \gamma)\eta_j)\right)\eta_i\|P_i\|_\infty\right).
    \]

    Applying Eq. \eqref{Iteration}, \(\Delta_{t+1}\) satisfies:
    \begin{align}
        \Delta_{t+1} &= (1-\eta_t)\Delta_t + \eta_t \mathcal{W}_t(\Delta_t) + \eta_t \mathcal{E}_t \notag \\
        &\leq (1-\eta_t)(a_t\|\Delta_0\|_\infty \boldsymbol{1} + b_t \boldsymbol{1} + P_t) + \eta_t \gamma \|\Delta_t\|_\infty \boldsymbol{1} + \eta_t \mathcal{E}_t \label{lm2_eqd_ineq1} \\
        &\leq (1-\eta_t)(a_t\|\Delta_0\|_\infty \boldsymbol{1} + b_t \boldsymbol{1} + P_t) + \eta_t \gamma (a_t\|\Delta_0\|_\infty + b_t + \|P_t\|_\infty)\boldsymbol{1} + \eta_t \mathcal{E}_t \label{lm2_eqd_ineq2} \\
        &= \underbrace{(1 - (1 - \gamma)\eta_t)a_t}_{a_{t+1}}\|\Delta_0\|_\infty \boldsymbol{1} + \underbrace{(\gamma \eta_t\|P_t\|_\infty + (1 - (1 - \gamma)\eta_t)b_t)}_{b_{t+1}}\boldsymbol{1} + \underbrace{(1-\eta_t)P_t + \eta_t \mathcal{E}_t}_{P_{t+1}}, \label{lm2_eqd_ineq3}
    \end{align}
    where inequality \eqref{lm2_eqd_ineq1} is obtained by applying the upper bound condition \eqref{simp_eq1} and the contraction property \(\mathcal{W}_t(\Delta_t) \leq \gamma \|\Delta_t\|_\infty \boldsymbol{1}\); Inequality \eqref{lm2_eqd_ineq2} is obtained by using conditions \eqref{simp_eq1}; Equality \eqref{lm2_eqd_ineq3} is achieved by rearranging the terms.

    Following the same steps, we can establish the corresponding lower bound for \(\Delta_{t+1}\) in Eq. \eqref{lower1}. This completes the proof.
\end{proof}

\subsection{Proof to Lemma \ref{lemma_3}}
\label{proof_lemma_3}
\begin{proof}
Recall the definition in Eq. \eqref{eq-Pt-1}. For any \(t \geq 2\), \(P_t\) can be expressed as a weighted sum:
\[
P_t = \sum\nolimits_{k=0}^{t-1} a_{t,k} \mathcal{E}_k,
\]
where \(a_{t,k} = \left(\prod_{j=k+1}^{t-1}(1 - \eta_j)\right)\eta_k\).

Denote the \((s,a)\)-th entry of \(P_t\) as \(P_t(s,a)\). Since all \(\mathcal{E}_k\) are independent and identically distributed, we can apply Bernstein's inequality \citep{vershynin2018high} to \(P_t(s,a)\). This gives us the following probability bound:
\[
\mathbb{P} \left( |P_t(s,a)| \geq \epsilon_t \right) \leq 2 \exp \left( - \frac{\epsilon_t^2}{2 \sum_{k=0}^{t-1} a_{t,k}^2 \sigma_{\mathcal{E}}^2{(s,a)} + \frac{2}{3} (\max_{k} |a_{t,k}| M_{\mathcal{E}{(s,a)}}) \epsilon_t} \right), \quad \forall (s,a),
\]
where \(\sigma_{\mathcal{E}}^2{(s,a)}\) is the variance of \(\mathcal{E}{(s,a)}\), and \(M_{\mathcal{E}{(s,a)}}\) is the maximum absolute value of \(\mathcal{E}(s,a)\).

Recall that \( P_t \) is the weighted sum of different \( \mathcal{E}_k \)'s, and all \( \mathcal{E}_k \) are \emph{i.i.d.} random variables. Each \( \mathcal{E}_k \) has at most \( |\mathcal{X}[\cup_{k=1}^{K_\omega} Z_k^P]| \) distinct entries due to the sampling scheme, where \( \mathcal{X}[\cup_{k=1}^{K_\omega} Z_k^P] \) represents the set of possible state-action pairs involved in the components' scopes. We denote these distinct state-action pairs by \( \mathcal{X}^* \) defined in Eq. \eqref{def_X^*}.

Taking the union bound over all state-action pairs \(x:=(s,a)\in\mathcal{X}^*\), we have:
\[
\mathbb{P} \left( \|P_t\|_\infty \geq \epsilon_t \right) \leq 2|\mathcal{X}[\cup_{k=1}^{K_\omega}Z^P_k]| \exp \left( - \frac{\epsilon_t^2}{2 \sum_{k=0}^{t-1} a_{t,k}^2 \sigma_{\mathcal{E}}^2{(s,a)} + \frac{2}{3} (\max_{k} |a_{t,k}| M_{\mathcal{E}{(s,a)}}) \epsilon_t} \right).
\]

To set this probability to \(\delta\), let \(\mathbb{P} \left( \|P_t\|_\infty \geq \epsilon_t \right) = \delta\). Then, we can conclude:
\[
\delta \leq 2|\mathcal{X}[\cup_{k=1}^{K_\omega}Z^P_k]| \exp \left( - \frac{\epsilon_t^2}{2 \sum_{k=0}^{t-1} a_{t,k}^2 \sigma_{\mathcal{E}}^2{(s,a)} + \frac{2}{3} (\max_{k} |a_{t,k}| M_{\mathcal{E}{(s,a)}}) \epsilon_t} \right).
\]

Taking the natural logarithm on both sides and rearranging terms gives:
\[
\epsilon_t^2 - \log\left(\frac{2|\mathcal{X}[\cup_{k=1}^{K_\omega}Z^P_k]|}{\delta}\right) \left(2 \sum_{k=0}^{t-1} a_{t,k}^2 \sigma_{\mathcal{E}}^2{(s,a)} + \frac{2}{3} \left(\max_{k} |a_{t,k}| M_{\mathcal{E}{(s,a)}}\right) \epsilon_t\right) \leq 0.
\]

To solve for \(\epsilon_t\), we treat this as a quadratic inequality:
\[
\epsilon_t^2 - \left(\frac{2}{3} \max_{k} |a_{t,k}| M_{\mathcal{E}{(s,a)}} \right) \log\left(\frac{2|\mathcal{X}[\cup_{k=1}^{K_\omega}Z^P_k]|}{\delta}\right) \epsilon_t - 2 \log\left(\frac{2|\mathcal{X}[\cup_{k=1}^{K_\omega}Z^P_k]|}{\delta}\right) \sum_{k=0}^{t-1} a_{t,k}^2 \sigma_{\mathcal{E}}^2{(s,a)} \leq 0.
\]

To simplify further:
\begin{align}
    \epsilon_t \leq \frac{2}{3} \log\left(\frac{2|\mathcal{X}[\cup_{k=1}^{K_\omega}Z^P_k]|}{\delta}\right) \max_{k} |a_{t,k}| M_{\mathcal{E}{(s,a)}} + 2\sqrt{\log\left(\frac{2|\mathcal{X}[\cup_{k=1}^{K_\omega}Z^P_k]|}{\delta}\right) \sum_{k=0}^{t-1} a_{t,k}^2 \sigma_{\mathcal{E}}^2{(s,a)}}. \label{res_epsilon}
\end{align}

To complete the analysis, we need to bound \(\max_k |a_{t,k}|\), \(M_{\mathcal{E}{(s,a)}}\), and \(\sum_{k=0}^{t-1} a_{t,k}^2\).

\textbf{Step 1: Bounding \(M_{\mathcal{E}{(s,a)}}\)} 

For any \(k\), since the empirical Bellman operator \(\widehat{\mathcal{H}}_{k}(Q^*)\) and the true Bellman operator \(\mathcal{H}(Q^*_\omega)\) are both bounded between 0 and \(\frac{1}{1 - \gamma}\), we have:
\[
M_{\mathcal{E}{(s,a)}} = \sup\|\mathcal{E}_k\|_\infty = \sup\|\widehat{\mathcal{H}}_{k}(Q^*) - \mathcal{H}(Q^*_\omega)\|_\infty \leq \frac{1}{1 - \gamma}.
\]

\textbf{Step 2: Bounding \(\max_k |a_{t,k}|\)}

Next, we analyze the ratio \(\frac{a_{t,k}}{a_{t,k-1}}\) to find \(\max_k |a_{t,k}|\). Recall the definition of \(a_{t,k}\):
\[
a_{t,k} = \left(\prod\nolimits_{j=k+1}^{t-1}(1 - \eta_j)\right)\eta_k.
\]

To compute \(\frac{a_{t,k}}{a_{t,k-1}}\), we have:
\[
\frac{a_{t,k}}{a_{t,k-1}} = \frac{\eta_k}{\eta_{k-1}(1 - \eta_k)}.
\]

Using the step size \(\eta_k = \frac{1}{1 + (1 - \gamma)(k+1)}\), we can express:
\[
\frac{a_{t,k}}{a_{t,k-1}} = \frac{1 - (1 - \gamma)\eta_k}{1 - \eta_k}.
\]

Given that \(0 < \gamma < 1\), we have \(\eta_k \leq \frac{1}{1 + (1 - \gamma)k} < 1\), and thus:
\[
\frac{a_{t,k}}{a_{t,k-1}} \geq 1, \quad \forall k.
\]

This means that \(a_{t,k}\) is non-decreasing with respect to \(k\). Hence, the maximum value \(\max_k |a_{t,k}|\) occurs at \(k = t-1\):
\[
\max_k |a_{t,k}| = a_{t,t-1} = \eta_{t-1} =  \frac{1}{1 + (1 - \gamma)t}.
\]

\textbf{Step 3: Bounding \(\sum_{k=0}^{t-1} a_{t,k}^2\)}

We also need to bound \(\sum_{k=0}^{t-1} a_{t,k}^2\). The following lemma provides this bound:
\begin{lemma}[Proof in Appendix \ref{proof_a_{t,k}^2}]
\label{lemma_a_{t,k}^2}
For any \(t \geq 1\), the sum \(\sum_{k=0}^{t-1} a_{t,k}^2\) satisfies:
\[
\sum_{k=0}^{t-1} a_{t,k}^2 \leq \frac{2}{1 + (1 - \gamma)t}.
\]
\end{lemma}

Combining the results from Steps 1, 2, and 3, we can now derive the bound for \(\epsilon_t\). Substituting these bounds into Eq. \eqref{res_epsilon}, we get that with probability at least $1-\delta$:
\[
\|P_t\|_\infty \leq \epsilon_t \leq \frac{2}{3(1-\gamma)(1+(1 - \gamma)t)} \log\left(\frac{2|\mathcal{X}[\cup_{k=1}^{K_\omega}Z^P_k]|}{\delta}\right) + \frac{2\sqrt{2\|\sigma^2_{\mathcal{E}}\|_\infty \log\left(\frac{2|\mathcal{X}[\cup_{k=1}^{K_\omega}Z^P_k]|}{\delta}\right)}}{\sqrt{1 + (1 - \gamma)t}}. \label{Res_1}
\]

This concludes our proof.
\end{proof}

\subsection{Proof of Lemma \ref{lemma_4}}
\label{proof_lemma_4}
\begin{proof}
Letting the upper bound of each $\|P_k\|_\infty$ (in Eq. \eqref{eq_P_k}) with $k = 1, \ldots, M$ holds with probability $1-\frac{\delta}{M}$, we have:
\begin{align*}
    \|\Delta_{M}\|_\infty \leq& \frac{\|\Delta_0\|_\infty}{1+(1-\gamma)M} + \frac{2\log\left(\frac{2M|\mathcal{X}[\cup_{k=1}^{K_\omega} Z_k^P]|}{\delta}\right)}{3(1-\gamma)^2M}  + \frac{2\sqrt{2\|\sigma^2_{\mathcal{E}}\|_\infty\log\left(\frac{2M|\mathcal{X}[\cup_{k=1}^{K_\omega} Z_k^P]|}{\delta}\right)}}{\sqrt{1+(1-\gamma)M}} \\
    & +\frac{2\log\left(\frac{2M|\mathcal{X}[\cup_{k=1}^{K_\omega} Z_k^P]|}{\delta}\right)}{3(1-\gamma)^2M}\sum_{i=1}^M\frac{1}{1+(1-\gamma)i}  \\
    & + \frac{{2\sqrt{2\|\sigma^2_{\mathcal{E}}\|_\infty\log\left(\frac{2M|\mathcal{X}[\cup_{k=1}^{K_\omega} Z_k^P]|}{\delta}\right)}}}{1+(1-\gamma)M} \sum_{i=1}^M\frac{1}{\sqrt{1+(1-\gamma)i}}.
\end{align*}

Note that, 
\begin{align*}
    &\sum_{i=1}^M\frac{1}{1+(1-\gamma)i} \leq \int_{0}^M\frac{1}{1+(1-\gamma)x}dx = \frac{\log\left(1+(1-\gamma)M\right)}{1-\gamma},\\
    &\sum_{i=1}^M\frac{1}{\sqrt{1+(1-\gamma)i}} \leq \int_{0}^M\frac{1}{\sqrt{1+(1-\gamma)x}}dx = \frac{2\left(\sqrt{1+(1-\gamma)M}-1\right)}{1-\gamma}\leq \frac{2\sqrt{1+(1-\gamma)M}}{1-\gamma}.
\end{align*}

Hence, $\Delta_{M}$ satisfies:
\begin{align*}
     \|\Delta_{M}\|_\infty \leq & \frac{\|\Delta_0\|_\infty}{1+(1-\gamma)M} + \frac{2\log\left(\frac{2M|\mathcal{X}[\cup_{k=1}^{K_\omega} Z_k^P]|}{\delta}\right)}{3(1-\gamma)^2M}  + \frac{2\sqrt{2\|\sigma^2_{\mathcal{E}}\|_\infty\log\left(\frac{2M|\mathcal{X}[\cup_{k=1}^{K_\omega} Z_k^P]|}{\delta}\right)}}{\sqrt{1+(1-\gamma)M}}\\
     &+\frac{2\log\left(\frac{2M|\mathcal{X}[\cup_{k=1}^{K_\omega} Z_k^P]|}{\delta}\right)\log(1+(1-\gamma)M)}{3(1-\gamma)^3M}+\frac{{4\sqrt{2\|\sigma^2_{\mathcal{E}}\|_\infty\log\left(\frac{2M|\mathcal{X}[\cup_{k=1}^{K_\omega} Z_k^P]|}{\delta}\right)}}}{(1-\gamma)^{\frac{3}{2}}M^{\frac{1}{2}}}\\
     \leq &\frac{3+2\log\left(\frac{2M|\mathcal{X}[\cup_{k=1}^{K_\omega} Z_k^P]|}{\delta}\right)}{3(1-\gamma)^2M}+\frac{2\log\left(\frac{2M|\mathcal{X}[\cup_{k=1}^{K_\omega} Z_k^P]|}{\delta}\right)\log(1+(1-\gamma)M)}{3(1-\gamma)^3M}\\
     &+\frac{{6\sqrt{2\|\sigma^2_{\mathcal{E}}\|_\infty\log\left(\frac{2M|\mathcal{X}[\cup_{k=1}^{K_\omega} Z_k^P]|}{\delta}\right)}}}{(1-\gamma)^{\frac{3}{2}}M^{\frac{1}{2}}}.
\end{align*}

This concludes our proof.

\end{proof}

\subsection{Proof of Lemma \ref{lemma_6}}
\label{proof_lemma_6}
\begin{proof}
Recall the definition of \(\overline{\mathcal{E}}^b\), which is given by:
\[
\overline{\mathcal{E}}^b = \overline{\mathcal{H}}_{\tau+1}(\overline{Q}_{\tau}) - \overline{\mathcal{H}}_{\tau+1}(Q^*_\omega) - \mathcal{H}(\overline{Q}_{\tau}) + \mathcal{H}(Q^*_\omega).
\]

We first note that the expectation of the difference between the empirical Bellman operators is:
\[
\mathbb{E}_{\overline{\mathcal{H}}_{\tau+1}}\left[\overline{\mathcal{H}}_{\tau+1}(\overline{Q}_{\tau}) - \overline{\mathcal{H}}_{\tau+1}(Q^*_\omega)\right] = \mathcal{H}(\overline{Q}_{\tau}) - \mathcal{H}(Q^*_\omega).
\]

Due to the \(\gamma\)-contraction property of both the empirical Bellman operator \(\overline{\mathcal{H}}_{\tau+1}\) and the true Bellman operator \(\mathcal{H}\), we have:
\[
\|\overline{\mathcal{H}}_{\tau+1}(\overline{Q}_{\tau}) - \overline{\mathcal{H}}_{\tau+1}(Q^*_\omega) - \mathcal{H}(\overline{Q}_{\tau}) + \mathcal{H}(Q^*_\omega)\|_\infty \leq \|\overline{\mathcal{H}}_{\tau+1}(\overline{Q}_{\tau}) - \overline{\mathcal{H}}_{\tau+1}(Q^*_\omega)\|_\infty + \|\mathcal{H}(\overline{Q}_{\tau}) - \mathcal{H}(Q^*_\omega)\|_\infty.
\]

Since both operators are \(\gamma\)-contractions, it follows that:
\[
\|\overline{\mathcal{H}}_{\tau+1}(\overline{Q}_{\tau}) - \overline{\mathcal{H}}_{\tau+1}(Q^*_\omega)\|_\infty \leq \gamma \|\overline{Q}_{\tau} - Q^*_\omega\|_\infty \quad \text{and} \quad \|\mathcal{H}(\overline{Q}_{\tau}) - \mathcal{H}(Q^*_\omega)\|_\infty \leq \gamma \|\overline{Q}_{\tau} - Q^*_\omega\|_\infty.
\]

Combining these, we get:
\[
\|\overline{\mathcal{H}}_{\tau+1}(\overline{Q}_{\tau}) - \overline{\mathcal{H}}_{\tau+1}(Q^*_\omega) - \mathcal{H}(\overline{Q}_{\tau}) + \mathcal{H}(Q^*_\omega)\|_\infty \leq 2\gamma \|\overline{Q}_{\tau} - Q^*_\omega\|_\infty \leq 2\gamma b_\tau.
\]

Applying Hoeffding's inequality to the \((s,a)\)-th entry of \(\overline{\mathcal{E}}^b\), denoted as \(\overline{\mathcal{E}}^b(s, a)\), we obtain:
\[
\mathbb{P}\left(|\overline{\mathcal{E}}^b(s, a)| \geq \epsilon\right) \leq 2\exp\left(-\frac{2N_\tau \epsilon^2}{4b_\tau^2}\right) = 2\exp\left(-\frac{N_\tau \epsilon^2}{2b_\tau^2}\right).
\]

Letting $2\exp\left(-\frac{N_\tau \epsilon^2}{2b_\tau^2}\right) = \frac{\delta }{3T|\mathcal{X}[\cup_{k=1}^{K_\omega}Z_k^P]|}$, and taking union bound for $x\in\mathcal{X}^*$, we have that:
\[
\|\overline{\mathcal{E}}^b\|_\infty \leq b_\tau \sqrt{\frac{2\log\left(\frac{6T|\mathcal{X}[\cup_{k=1}^{K_\omega}Z_k^P]|}{\delta}\right)}{N_\tau}}.
\]

This completes the proof.
\end{proof}

\subsection{Proof of Lemma \ref{lemma_5}}
\label{proof_lemma_5}
\begin{proof}
Recall that \(\overline{E}^c = \overline{T}(Q^*) - \mathcal{H}(Q^*_\omega)\) represents the estimation error using \(N_\tau\) empirical Bellman operators. We know that \(\mathbb{E}[\overline{T}(Q^*)] = \mathcal{H}(Q^*_\omega)\).

Applying Bernstein's inequality \citep{vershynin2018high} to the \((s, a)\)-th entry of \(\overline{E}^c\), denoted as \(\overline{\mathcal{E}}^c_{(s, a)}\), we obtain:
\[
\mathbb{P}\left(|\overline{\mathcal{E}}^c_{(s, a)}| \geq \epsilon\right) \leq 2\exp\left(-\frac{\frac{1}{2} N_\tau \epsilon^2}{\|\sigma^2_{\mathcal{E}}\|_\infty + \frac{\epsilon}{3(1-\gamma)}}\right).
\]

We set the right-hand side probability equal to \(\frac{\delta}{3T|\mathcal{X}[\cup_{k=1}^{K_\omega} Z_k^P]|}\), which gives us:
\[
2\exp\left(-\frac{\frac{1}{2} N_\tau \epsilon^2}{\|\sigma^2_{\mathcal{E}}\|_\infty + \frac{\epsilon}{3(1-\gamma)}}\right) = \frac{\delta}{3T|\mathcal{X}[\cup_{k=1}^{K_\omega} Z_k^P]|}.
\]

Rearranging this equation for \(\epsilon\), we get:
\[
\frac{1}{2} N_\tau \epsilon^2 - \frac{\epsilon \log(6T|\mathcal{X}[\cup_{k=1}^{K_\omega} Z_k^P]|/\delta)}{3(1-\gamma)} - \|\sigma^2_{\mathcal{E}}\|_\infty \log(6T|\mathcal{X}[\cup_{k=1}^{K_\omega} Z_k^P]|/\delta) = 0.
\]

Solving this quadratic equation for \(\epsilon\), we have:
\[
\epsilon \leq \frac{\frac{2\log(6T|\mathcal{X}[\cup_{k=1}^{K_\omega} Z_k^P]|/\delta)}{3(1-\gamma)} + \sqrt{\frac{4\log^2(6T|\mathcal{X}[\cup_{k=1}^{K_\omega} Z_k^P]|/\delta)}{9(1-\gamma)^2} + 2N_\tau \|\sigma^2_{\mathcal{E}}\|_\infty \log(6T|\mathcal{X}[\cup_{k=1}^{K_\omega} Z_k^P]|/\delta)}}{N_\tau}.
\]

Simplifying further, we find:
\[
\epsilon \leq \frac{2\log(6T|\mathcal{X}[\cup_{k=1}^{K_\omega} Z_k^P]|/\delta)}{3N_\tau(1-\gamma)} + \sqrt{\frac{2\|\sigma^2_{\mathcal{E}}\|_\infty \log(6T|\mathcal{X}[\cup_{k=1}^{K_\omega} Z_k^P]|/\delta)}{N_\tau}}.
\]

Applying the union bound over all \(x:=(s, a)\in\mathcal{X}^*\), we conclude that with probability at least \(1 - \frac{\delta}{3T}\), the infinity norm \(\|\overline{\mathcal{E}}^c\|_\infty\) satisfies:
\begin{align}
    \|\overline{\mathcal{E}}^c\|_\infty \leq \frac{2\log(6T|\mathcal{X}[\cup_{k=1}^{K_\omega} Z_k^P]|/\delta)}{3N_\tau(1-\gamma)} + \sqrt{\frac{2\|\sigma^2_{\mathcal{E}}\|_\infty \log(6T|\mathcal{X}[\cup_{k=1}^{K_\omega} Z_k^P]|/\delta)}{N_\tau}}. \label{lemma12_key_eq}
\end{align}

Given that \(N_\tau \geq c \cdot 4^\tau \frac{\log(8T|\mathcal{X}[\cup_{k=1}^{K_\omega} Z_k^P]|/\delta)}{(1-\gamma)^2}\), it follows that \(\frac{\log\left({2T|\mathcal{X}[\cup_{k=1}^{K_\omega} Z_k^P]|}{/\delta}\right)}{N_\tau(1-\gamma)} \leq 1\) for a large enough $c$. Thus, \(\|\overline{\mathcal{E}}^c\|_\infty\) is further bounded by:
\begin{align*}
    \|\overline{\mathcal{E}}^c\|_\infty &\leq \overline{c}_1 \left(\sqrt{\frac{\|\sigma^2_{\mathcal{E}}\|_\infty \log\Bigg(\frac{6T|\mathcal{X}[\cup_{k=1}^{K_\omega} Z_k^P]|}{\delta}\Bigg)}{N_\tau}} + 1\right)\\
    &\leq \overline{c}_1 \sqrt{\frac{\log\left(\frac{6T|\mathcal{X}[\cup_{k=1}^{K_\omega} Z_k^P]|}{\delta}\right)}{N_\tau}} \left(\sqrt{\|\sigma^2_{\mathcal{E}}\|_\infty} + 1\right),
\end{align*}
where \(\overline{c}_1\) is an absolute constant.
This completes the proof.
\end{proof}

\subsection{Proof of Lemma \ref{lemma_7}}
\label{proof_lemma_7}
\begin{proof}
        The term $P_k^a$ evolves in an exactly the same way as $P_k$ discussed before. Hence, we apply the result in Lemma \ref{lemma_3} and get that, with probability at least $1-\frac{\delta}{3TM}$:
        \begin{align*}
            \|P_t^a\|_\infty\leq \frac{2M_{\mathcal{E},a}}{3(1+(1 - \gamma)t)} \log\left(\frac{6TM|\mathcal{X}[\cup_{k=1}^{K_\omega} Z_k^P]|}{\delta}\right) + \frac{2\sqrt{2\|\sigma^2_{\mathcal{E},a}\|_\infty \log\left(\frac{6TM|\mathcal{X}[\cup_{k=1}^{K_\omega} Z_k^P]|}{\delta}\right)}}{\sqrt{1 + (1 - \gamma)t}}, \forall t.
        \end{align*}

        Here $M_{\mathcal{E},a}$ and $\sigma_{\mathcal{E},a}^{2}$ denote the maximal absolute value and variance of error $\overline{\mathcal{E}}^a_t$.

        First, we observe that $\overline{\mathcal{E}}^a_t$ satisfies:
        \begin{align*}
            \|\overline{\mathcal{E}}^a_t\|_\infty &= \|\mathcal{H}(\overline{Q}_{\tau})-\mathcal{H}({Q}^*_\omega)-\widehat{\mathcal{H}}_t(\overline{Q}_{\tau})+\widehat{\mathcal{H}}_t({Q}^*_\omega)\|_\infty\\
&\leq \|\mathcal{H}(\overline{Q}_{\tau})-\mathcal{H}({Q}^*_\omega)\|_\infty+\|\widehat{\mathcal{H}}_t(\overline{Q}_{\tau})-\widehat{\mathcal{H}}_t({Q}^*_\omega)\|_\infty\\
&\leq 2\gamma\|\overline{Q}_{\tau}-Q^*\|_\infty\\
&\leq 2b_\tau.
        \end{align*}

        Thus, we have $M_{\mathcal{E},a}\leq 2b_\tau$ and $\sqrt{\|\sigma_{\mathcal{E},a}^{2}\|_\infty}\leq 2b_\tau$. And $\|P_t^a\|_\infty$ then satisfies:
        \begin{align*}
            \|P_t^a\|_\infty\leq b_\tau\left(\frac{4\log\left(\frac{6TM|\mathcal{X}[\cup_{k=1}^{K_\omega} Z_k^P]|}{\delta}\right)}{3(1+(1-\gamma)t)} + \sqrt{\frac{8\log\left(\frac{6TM|\mathcal{X}[\cup_{k=1}^{K_\omega} Z_k^P]|}{\delta}\right)}{1+(1-\gamma)t}}\right)
        \end{align*}

        Thus, we can control the whole term by:
        \begin{align*}
            \gamma\eta_{t-1}\sum_{k=0}^{t-1}\|P^a_k\|_\infty + \|P^a_t\|_\infty \leq& \gamma\eta_{t-1}\sum_{i=0}^{t-1}b_\tau \left(\frac{4\log\left(\frac{6TM|\mathcal{X}[\cup_{k=1}^{K_\omega} Z_k^P]|}{\delta}\right)}{3(1+(1-\gamma)i)} + \sqrt{\frac{8\log\left(\frac{6TM|\mathcal{X}[\cup_{k=1}^{K_\omega} Z_k^P]|}{\delta}\right)}{1+(1-\gamma)i}}\right)\\
            & +b_\tau\left(\frac{4\log\left(\frac{6TM|\mathcal{X}[\cup_{k=1}^{K_\omega} Z_k^P]|}{\delta}\right)}{3(1+(1-\gamma)t)} + \sqrt{\frac{8\log\left(\frac{6TM|\mathcal{X}[\cup_{k=1}^{K_\omega} Z_k^P]|}{\delta}\right)}{1+(1-\gamma)t}}\right).
        \end{align*}

For analyzing the coefficients:
\begin{align*}
    \gamma \eta_{t-1}\sum_{i=0}^{t-1}\frac{1}{1+(1-\gamma)i} + \frac{1}{1+(1-\gamma)t}\leq& \frac{\gamma}{1+(1-\gamma)t}\int_{0}^t\frac{1}{1+(1-\gamma)x}dx+\frac{1}{1+(1-\gamma)t}\\
    \leq& \frac{\gamma}{1+(1-\gamma)t}\cdot\frac{\log(1+(1-\gamma)t)}{1-\gamma}+\frac{1}{1+(1-\gamma)t}\\
    \leq&\frac{2\log(1+(1-\gamma)t)}{(1-\gamma)^2t}.
\end{align*}
\begin{align*}
    \gamma \eta_{t-1}\sum_{i=0}^{t-1}\frac{1}{\sqrt{1+(1-\gamma)i}} + \frac{1}{\sqrt{1+(1-\gamma)t}}\leq& \frac{\gamma}{1+(1-\gamma)t}\int_{0}^t\frac{1}{\sqrt{1+(1-\gamma)x}}dx+\frac{1}{\sqrt{1+(1-\gamma)t}}\\
    \leq &\frac{\gamma}{1+(1-\gamma)t}\frac{2(\sqrt{1+(1-\gamma)t}-1)}{1-\gamma}+\frac{1}{\sqrt{1+(1-\gamma)t}}\\
    \leq &\frac{3}{(1-\gamma)^{\frac{3}{2}}t^{\frac{1}{2}}}.
\end{align*}

Thus, the whole term satisfies:
\begin{align}
    \gamma\eta_{t-1}\sum_{k=0}^{t-1}\|P^a_k\|_\infty + \|P^a_t\|_\infty \leq \frac{8 b_\tau\log(1+(1-\gamma)t)\log\left(\frac{6TM|\mathcal{X}[\cup_{k=1}^{K_\omega} Z_k^P]|}{\delta}\right)}{3(1-\gamma)^2t} + \frac{6b_\tau\sqrt{2\log\left(\frac{6TM|\mathcal{X}[\cup_{k=1}^{K_\omega} Z_k^P]|}{\delta}\right)}}{(1-\gamma)^{\frac{3}{2}}t^{\frac{1}{2}}}.
\end{align}
        Absorbing the constants concludes our proof.
    \end{proof}

\subsection{Proof of Lemma \ref{lemma_a_{t,k}^2} }
\label{proof_a_{t,k}^2}
\begin{proof}
    Recall the expression for the ratio \(\frac{a_{t,k-1}}{a_{t,k}}\):
    \begin{align*}
        \frac{a_{t,k-1}}{a_{t,k}} = \frac{1 - \eta_k}{1 - (1 - \gamma)\eta_k} = \frac{(1 - \gamma)(k+1)}{(1 - \gamma)(k+1) + \gamma} = \frac{k+1}{k+1 + \frac{\gamma}{1 - \gamma}}.
    \end{align*}

    Using this, we can express \(a_{t,k-i}\) for any \(1 \leq i < k\) as follows:
    \begin{align*}
        a_{t,k-i} &= a_{t,k} \prod_{j=0}^{i-1} \frac{a_{t,k-j-1}}{a_{t,k-j}} = a_{t,k} \prod_{j=0}^{i-1} \frac{k+1-j}{k+1-j + \frac{\gamma}{1 - \gamma}}.
    \end{align*}

    To simplify this product, note that:
    \[
    \frac{k+1-j}{k+1-j + \frac{\gamma}{1 - \gamma}} \leq \frac{k+1}{k+1 + \frac{\gamma}{1 - \gamma}} \quad \forall j.
    \]

    Thus, we can further bound \(a_{t,k-i}\) by:
    \begin{align*}
        a_{t,k-i} &\leq a_{t,k} \prod_{j=0}^{i-1} \frac{k+1}{k+1 + \frac{\gamma}{1 - \gamma}} = a_{t,k} \left(\frac{k+1}{k+1 + \frac{\gamma}{1 - \gamma}}\right)^i.
    \end{align*}

    Letting \(k = t-1\), we can now bound \(\sum_{k=0}^{t-1} a_{t,k}^2\) by considering the following square form:
    \begin{align*}
        \sum_{k=0}^{t-1} a_{t,k}^2 &\leq a_{t,t-1}^2 \left(1 + \sum_{k=1}^{t-1} \left(\frac{t}{t+ \frac{\gamma}{1 - \gamma}}\right)^{2k}\right).
    \end{align*}

    We analyze this term by considering different $\gamma$. 

\noindent\textbf{Case 1: $\gamma \geq \frac{1}{2}$}

    This series is geometric with a ratio \(\left(\frac{t}{t + \frac{\gamma}{1 - \gamma}}\right)^2 < 1\). The sum of a geometric series can be calculated as:
    \[
    \sum_{k=0}^{\infty} r^k = \frac{1}{1 - r} \quad \text{for } |r| < 1.
    \]

    Applying this to our series:
    \begin{align*}
        \sum_{k=0}^{t-1} a_{t,k}^2 &\leq a_{t,t-1}^2 \left(1 + \frac{\left(\frac{t}{t + \frac{\gamma}{1 - \gamma}}\right)^2}{1 - \left(\frac{t}{t + \frac{\gamma}{1 - \gamma}}\right)^2}\right) = a_{t,t-1}^2\left(1+\frac{(1-\gamma)^2t^2}{\gamma(2t(1-\gamma)+\gamma)}\right).
    \end{align*}

Since $\gamma \geq \frac{1}{2}$, we have:
\begin{align*}
    a_{t,t-1}^2\left(1+\frac{(1-\gamma)^2t^2}{\gamma(2t(1-\gamma)+\gamma)}\right)\leq a_{t,t-1}^2\left(1+\frac{(1-\gamma)^2t^2}{t(1-\gamma)}\right) = a_{t,t-1}^2(1+(1-\gamma)t)
\end{align*}

Using $a_{t,t-1} = \frac{1}{1+(1-\gamma)t}$ yields that, 
\begin{align}
    \sum_{k=0}^{t-1} a_{t,k}^2 \leq \frac{1}{1+(1-\gamma)t}. \label{res_at_1}
\end{align}
\textbf{Case 2: $\gamma < \frac{1}{2}$}

We use $\frac{t}{t + \frac{\gamma}{1 - \gamma}}\leq 1$ and ${1-\gamma}\geq \frac{1}{2}$ to obtain that:
\begin{align*}
    a_{t,t-1}^2 \left(1 + \sum_{k=1}^{t-2} \left(\frac{t}{t + \frac{\gamma}{1 - \gamma}}\right)^{2k}\right) \leq a_{t,t-1}^2 \left(1 + \sum_{k=1}^{t-2} 1\right) \leq a_{t,t-1}^2(1+t).
\end{align*}

Since $\gamma < \frac{1}{2}$, we have $1-\gamma > \frac{1}{2}$ and $1 < 2(1-\gamma)$. Thus
\begin{align*}
    a_{t,t-1}^2(1+t) \leq  a_{t,t-1}^2(1+2(1-\gamma)t) \leq 2a_{t,t-1}^2(1+(1-\gamma)t).
\end{align*}

Using $a_{t,t-1} = \frac{1}{1+(1-\gamma)(t-1)}$ yields that, 
\begin{align}
    a_{t,t-1}^2\left(1+\frac{(1-\gamma)^2t^2}{\gamma(2(1-\gamma)t+\gamma)}\right) \leq \frac{2}{1+(1-\gamma)t}. \label{res_at_2}
\end{align}

Combing Eq. \eqref{res_at_1} and \eqref{res_at_2} yields our result.
\end{proof}

\subsection{Proof to Lemma \ref{Refined_analysis}}
\label{proof_Refined_analysis}
\begin{proof}
    We proceed by induction on the epoch index \(\tau\). The base case when \(\tau = 0\) is straightforward, as the initial approximation \(\overline{Q}_0\) can satisfy the desired bound. Our focus is on the inductive step: assuming that the input \(\overline{Q}_\tau\) in epoch \(\tau\) satisfies
\[
\|\overline{Q}_\tau - Q^*_\omega\|_\infty \leq \frac{1}{2^\tau\sqrt{1 - \gamma}} := b'_\tau,
\]
we aim to show that the output \(\overline{Q}_{\tau+1}\) satisfies $\|\overline{Q}_{\tau+1} - Q^*\|_\infty \leq \frac{b'_\tau}{2}$ with probability at least \(1 - \frac{\delta}{T}\).

To achieve this, we analyze the variance-reduced update rule, which can be expressed as
\[
Q_{t+1} = (1 - \eta_t) Q_t + \eta_t \widehat{\mathcal{F}}_t(Q_t),
\]
where \(\widehat{\mathcal{F}}_t(Q_t)\) is the variance-reduced Bellman operator defined by
\[
\widehat{\mathcal{F}}_t(Q_t) = \widehat{\mathcal{H}}_t(Q_t) - \widehat{\mathcal{H}}_t(\overline{Q}_\tau) + \overline{\mathcal{H}}_{\tau+1}(\overline{Q}_\tau).
\]
Here, \(\widehat{\mathcal{H}}_t\) is the empirical Bellman operator based on samples at iteration \(t\), and \(\overline{\mathcal{H}}_{\tau+1}\) is the reference Bellman operator estimated using a larger reference dataset collected in epoch \(\tau+1\).

We additionally define the refined variance-reduced Bellman operator, eliminating the randomness due to sampling:
\[
\mathcal{F}(Q_t) = \mathcal{H}(Q_t) - \mathcal{H}(\overline{Q}_\tau) + \overline{\mathcal{H}}_{\tau+1}(\overline{Q}_\tau),
\]
where \(\mathcal{H}\) is the true Bellman operator. Due to independent and identically distributed sampling, we have \(\mathbb{E}[\widehat{\mathcal{F}}_t(Q_t)] = \mathcal{F}(Q_t)\). Moreover, since \(\mathcal{H}\) is a \(\gamma\)-contraction, it follows that \(\mathcal{F}\) is also a \(\gamma\)-contraction mapping.

We introduce a reference \(Q\)-function \(\widetilde{Q}\), defined as the fixed point of \(\mathcal{F}\):
\[
\widetilde{Q} = \mathcal{F}(\widetilde{Q}) = \mathcal{H}(\widetilde{Q}) - \mathcal{H}(\overline{Q}_\tau) + \overline{\mathcal{H}}_{\tau+1}(\overline{Q}_\tau).
\]
This function \(\widetilde{Q}\) serves as an intermediary between \(Q_M\) and \(Q^*\), potentially closer to \(Q_M\) than \(Q^*\) is.

Our strategy is to bound the error \(\|Q_M - Q^*_\omega\|_\infty\) by decomposing it into two components:
\[
\|Q_M - Q^*_\omega\|_\infty \leq \underbrace{\|Q_M - \widetilde{Q}\|_\infty}_{=: B_1} + \underbrace{\|\widetilde{Q} - Q^*_\omega\|_\infty}_{=: B_2}.
\]
The term \(B_1\) quantifies the distance between the iteratively computed \(Q_M\) and the reference \(\widetilde{Q}\), while \(B_2\) measures the distance between \(\widetilde{Q}\) and the optimal \(Q^*_\omega\).

We aim to bound \(B_1\) and \(B_2\) separately. Intuitively, \(B_1\) will decrease with the number of iterations \(M\) within the epoch, and \(B_2\) will be controlled by the size of the reference dataset \(N_\tau\). We show they can be well bounded when $N_\tau$ and $M$ are properly selected:

\begin{lemma}[Proof in Appendix \ref{proof_lemma_reference_B1}]
\label{lemma_reference_B1}
    In the $\tau$-th epoch, with probability at least $1-\frac{\delta}{2T}$, $B_1$ satisfies: 
    \begin{align*}
        B_1 \leq \frac{b'_\tau+ B_2}{5},
    \end{align*}
    if $N_\tau = c'_1 \frac{4^\tau \log(8T|\mathcal{X}[\cup_{k=1}^{K_\omega} Z_k^P]|)}{(1-\gamma)^2}$ and $M = c'_2 \frac{\log(6T|\mathcal{X}[\cup_{k=1}^{K_\omega} Z_k^P]|(1-\gamma)^{-1}\epsilon^{-1})}{(1-\gamma)^3}$ with large enough $c'_1$ and $c'_2$.
\end{lemma}

\begin{lemma}[Proof in Appendix \ref{proof_lemma_reference_B2}]
\label{lemma_reference_B2}
    In the $\tau$-th epoch, with probability at least $1-\frac{\delta}{2T}$, $B_2$ satisfies:
    \begin{align*}
        B_2 \leq \frac{b'_\tau}{4}.
    \end{align*}
    if $N_\tau = c'_1 \frac{4^\tau \log(8T|\mathcal{X}[\cup_{k=1}^{K_\omega} Z_k^P]|)}{(1-\gamma)^2}$ and $M = c'_2 \frac{\log(6T|\mathcal{X}[\cup_{k=1}^{K_\omega} Z_k^P]|(1-\gamma)^{-1}\epsilon^{-1})}{(1-\gamma)^3}$ with large enough $c'_1$ and $c'_2$.
\end{lemma}

{By combining Lemmas \ref{lemma_reference_B1} and \ref{lemma_reference_B2}, we obtain:
\[
\|Q_M - Q^*_\omega\|_\infty \leq B_1 + B_2 \leq \frac{b'_\tau}{5} + \frac{6 B_2}{5} = \frac{b'_\tau}{5} + \frac{6}{5}\cdot \frac{b'_\tau}{4} \leq \frac{b'_\tau}{2}.
\]
This completes the inductive step and hence the proof of Lemma \ref{Refined_analysis}.
}
\end{proof}

\subsection{Proof to Lemma \ref{lemma_reference_B1}}
\label{proof_lemma_reference_B1}
\begin{proof}
    {
We aim to bound \(B_1 = \|Q_M - \widetilde{Q}\|_\infty\). Our approach mirrors the inductive case in the proof of Lemma \ref{key_lemma_model_free}, with a key difference in the error decomposition.

Consider the difference between the iterates \(Q_{t+1}\) and the reference \(\widetilde{Q}\):
\[
Q_{t+1} - \widetilde{Q} = (1 - \eta_t) Q_t + \eta_t \widehat{\mathcal{F}}_t(Q_t) - \widetilde{Q}.
\]

Subtracting and adding \(\eta_t \mathcal{F}_t(\widetilde{Q})\), and using $\mathcal{F}(\widetilde{Q}) = \widetilde{Q}$, we can rewrite this as
\[
Q_{t+1} - \widetilde{Q} = (1 - \eta_t) (Q_t - \widetilde{Q}) + \eta_t \left( \widehat{\mathcal{F}}_t(Q_t) - \widehat{\mathcal{F}}_t(\widetilde{Q}) \right) + \eta_t \left( \widehat{\mathcal{F}}_t(\widetilde{Q}) - \mathcal{F}(\widetilde{Q}) \right).
\]
The first term represents the contraction towards \(\widetilde{Q}\). The second is $\gamma$-contractive because
\begin{align*}
    \left\| \widehat{\mathcal{F}}_t(Q_t) - \widehat{\mathcal{F}}_t(\widetilde{Q}) \right\|_\infty \leq \gamma \left\| Q_t - \widetilde{Q} \right\|_\infty,
\end{align*}
which is easy to handle, and the third term is an \emph{i.i.d.} zero-mean random variable across iterations \(t\), with the magnitude can be bounded by:
\begin{align*}
    \|\widehat{\mathcal{F}}_t(\widetilde{Q})-\mathcal{F}(\widetilde{Q})\|_\infty = \|\widehat{T}_t(\widetilde{Q})-\widehat{T}_t(\overline{Q}_\tau)\|_\infty + 
    \|{T}_t(\widetilde{Q})-{T}_t(\overline{Q}_\tau)\|_\infty
    \leq 2\gamma \|\widetilde{Q}-\overline{Q}_\tau\|_\infty.
\end{align*}

Hence, the iteration of $Q_{t+1} - \widetilde{Q}$ mirrors the iteration in Eq. \eqref{Iteration}. Following the same routine as the proof of Lemma \ref{lemma_4}, we establish that, for a  step sizes \(\eta_t = \frac{1}{1+(1-\gamma)(t+1)}\), with probability $1-\frac{\delta}{2T}$,
\[
\|Q_M - \widetilde{Q}\|_\infty \leq c \left( \frac{1}{(1 - \gamma) M} + \frac{1}{(1 - \gamma)^{3/2} \sqrt{M}} \right) \log\left( \frac{6TM |\mathcal{X}[\cup_{k=1}^{K_\omega} Z_k^P]|}{\delta} \right) \|\widetilde{Q} - \overline{Q}_\tau\|_\infty,
\]
where \(c\) is a large enough constant. 

By choosing \(M = c_2 \frac{ \log\left( \frac{6T |\mathcal{X}[\cup_{k=1}^{K_\omega} Z_k^P]|}{(1 - \gamma) \delta} \right) }{ (1 - \gamma)^3 }\) with sufficiently large $c_2$,
we can ensure that the right-hand side is less than \(\frac{1}{5} \| \widetilde{Q} - \overline{Q}_\tau \|_\infty\).

Observing that
\[
\| \widetilde{Q} - \overline{Q}_\tau \|_\infty \leq \| \widetilde{Q} - Q^* \|_\infty + \| \overline{Q}_\tau - Q^*_\omega \|_\infty = B_2 + b'_\tau,
\]
we can conclude
\[
B_1 = \|Q_M - \widetilde{Q}\|_\infty \leq \frac{1}{5} \| \widetilde{Q} - \overline{Q}_\tau \|_\infty \leq \frac{1}{5} (B_2 + b'_\tau).
\]

This establishes the desired bound on \(B_1\), completing the proof of Lemma \ref{lemma_reference_B1}.}
\end{proof}

\subsection{Proof to Lemma \ref{lemma_reference_B2}}
\label{proof_lemma_reference_B2}
\begin{proof}
    Our goal is to bound the term \( B_2 = \|\widetilde{Q} - Q^*_\omega\|_\infty \), where \( \widetilde{Q} \) is the fixed point of the operator \( \mathcal{F} \) defined as:
\[
\widetilde{Q} = \mathcal{F}(\widetilde{Q}) = \mathcal{H}(\widetilde{Q}) - \mathcal{H}(\overline{Q}_\tau) + \overline{\mathcal{H}}_{\tau+1}(\overline{Q}_\tau).
\]

Recall that $Q^*_\omega = \mathcal{H}(Q^*_\omega)$. Thus, $\widetilde{Q}$ can be viewed as the fixed point of a Bellman operator with a perturbed reward function, where the perturbation magnitude is $- \mathcal{H}(\overline{Q}_\tau) + \overline{\mathcal{H}}_{\tau+1}(\overline{Q}_\tau)$. We seek to show that when the population size for estimating the reference Bellman operator is large enough, the term $- \mathcal{H}(\overline{Q}_\tau) + \overline{\mathcal{H}}_{\tau+1}(\overline{Q}_\tau)$ converges to $0$ and $\mathcal{F}$ converges to the actual Bellman operator $\mathcal{H}$, allowing $\|\widetilde{Q} - Q^*_\omega\|_\infty \leq \frac{b'_\tau}{4}$.
Essentially, the error $B_2$ comes from the reward function perturbation $\overline{\mathcal{H}}_{\tau+1}(\overline{Q}_\tau)- \mathcal{H}(\overline{Q}_\tau)$. In the first step, we bound the reward perturbation $\Delta r = - \mathcal{H}(\overline{Q}_\tau) + \overline{\mathcal{H}}_{\tau+1}(\overline{Q}_\tau)$ as follows:
\begin{align*}
    |\Delta r| = |\overline{\mathcal{H}}_{\tau+1}(\overline{Q}_\tau)- \mathcal{H}(\overline{Q}_\tau)|&\leq |\overline{\mathcal{H}}_{\tau+1}(\overline{Q}_\tau)-\overline{\mathcal{H}}_{\tau+1}(Q^*_\omega)+{\mathcal{H}}(Q^*_\omega)-{\mathcal{H}}(Q^*_\omega)+\overline{\mathcal{H}}_{\tau+1}(Q^*_\omega)- \mathcal{H}(\overline{Q}_\tau)|\\
    &\leq |\overline{\mathcal{H}}_{\tau+1}(\overline{Q}_\tau)-\overline{\mathcal{H}}_{\tau+1}(Q^*_\omega)|+ |\mathcal{H}(\overline{Q}_\tau)-{\mathcal{H}}(Q^*_\omega)|+|\overline{\mathcal{H}}_{\tau+1}(Q^*_\omega)-{\mathcal{H}}(Q^*_\omega)|\\
    &\leq 2\gamma\|\overline{Q}_\tau-Q^*\|_\infty\cdot\boldsymbol{1} + |\overline{\mathcal{H}}_{\tau+1}(Q^*_\omega)-{\mathcal{H}}(Q^*_\omega)|.
\end{align*}

The first terms can be directly bounded through Hoeffding's inequality. Similar to the proof of Lemma \ref{lemma_6}, with probability at least $1-\frac{\delta}{4T}$,
\begin{align*}
    2\gamma\|\overline{Q}_\tau-Q^*_\omega\|_\infty\cdot\boldsymbol{1} \leq 4b'_\tau \sqrt{\frac{\log(8T|\mathcal{X}[\cup_{k \in [K_\omega]} Z_k^P]|/\delta)}{N_\tau}}.
\end{align*}

The second term can be bounded by applying the result of Eq. \eqref{lemma12_key_eq}\footnote{The only difference is to applying the bound on the vector $|\overline{\mathcal{H}}_{\tau+1}(Q^*_\omega)-{\mathcal{H}}(Q^*_\omega)|$, instead of the infinity norm.} in Lemma \ref{lemma_5}, we have that with probability at least $1-\frac{\delta}{4T}$,
\begin{align*}
   |\overline{\mathcal{H}}_{\tau+1}(Q^*_\omega)-{\mathcal{H}}(Q^*_\omega)| \leq   c\left\{\frac{\log(8T|\mathcal{X}[\cup_{k=1}^{K_\omega} Z_k^P]|/\delta)\cdot\boldsymbol{1}}{N_\tau(1-\gamma)} + \sqrt{\frac{\sigma^2_{\mathcal{E}} \log(8T|\mathcal{X}[\cup_{k=1}^{K_\omega} Z_k^P]|/\delta)}{N_\tau}}\right\},
\end{align*}
where $c$ is a large enough constant.

Combining these two conditions, we can derive that, with probability at least $1-\frac{\delta}{2T}$,
\begin{align}
    |\Delta r| \leq c\left\{\frac{\log(8T|\mathcal{X}[\cup_{k=1}^{K_\omega} Z_k^P]|/\delta)}{N_\tau(1-\gamma)}\cdot\boldsymbol{1} + (b'_\tau\boldsymbol{1} + \sigma_{\mathcal{E}})\sqrt{\frac{ \log(8T|\mathcal{X}[\cup_{k=1}^{K_\omega} Z_k^P]|/\delta)}{N_\tau}}\right\}, \label{def_Delta r}
\end{align}
where $c$ is a large enough constant.

Now we have already obtained the element-wise error bound of the reward function, which is the only difference for inducing $\widetilde{Q}$ and $Q^*$. Applying Lemma \ref{lemmaA4} with Eq. \eqref{def_Delta r}, we have:
\begin{align*}
    |\widetilde{Q}-Q^*_\omega| \leq \max\left\{\Delta_1, \Delta_2\right\},
\end{align*}
where 
\begin{align}
    \Delta_1 &= c({I}-\gamma {P}^{\pi_\omega^*}_\omega)^{-1}\left\{\frac{\log(8T|\mathcal{X}[\cup_{k=1}^{K_\omega} Z_k^P]|/\delta)}{N_\tau(1-\gamma)}\cdot\boldsymbol{1} + (b'_\tau\boldsymbol{1} + \sigma_{\mathcal{E}})\sqrt{\frac{ \log(8T|\mathcal{X}[\cup_{k=1}^{K_\omega} Z_k^P]|/\delta)}{N_\tau}}\right\}, \label{def_final_delta_1}\\
    \Delta_2 &= c({I}-\gamma {P}^{\Tilde{\pi}_\omega^*}_\omega)^{-1}\left\{\frac{\log(8T|\mathcal{X}[\cup_{k=1}^{K_\omega} Z_k^P]|/\delta)}{N_\tau(1-\gamma)}\cdot\boldsymbol{1} + (b'_\tau\boldsymbol{1} + \sigma_{\mathcal{E}})\sqrt{\frac{ \log(8T|\mathcal{X}[\cup_{k=1}^{K_\omega} Z_k^P]|/\delta)}{N_\tau}}\right\}. \label{def_final_delta_2}
\end{align}

\noindent \textbf{Bounding $\Delta_1$:}
For bounding $\Delta_1$, the key is to bound $({I}-\gamma {P}^{\pi_\omega^*}_\omega)^{-1}\sigma_{\mathcal{E}}$. We have:
\begin{align}
    \|({I}-\gamma {P}^{\pi_\omega^*}_\omega)^{-1}\sigma_{\mathcal{E}}\|_\infty &= \gamma^2\|({I}-\gamma {P}^{\pi_\omega^*}_\omega)^{-1}\text{Var}_{P_\omega^*}(V_\omega^*)\|_\infty \leq \frac{\sqrt{2}}{(1-\gamma)^{\frac{3}{2}}},\label{lm19_key_bound111}
\end{align}
where the second inequality comes from Eq. \eqref{a.6_bound1} in Appendix \ref{proof_to_lemma_(1-gamma)^3}. By using $\|({I}-\gamma {P}^{\pi_\omega^*}_\omega)^{-1}\|_\infty\leq \frac{1}{1-\gamma}$ and Eq. \eqref{lm19_key_bound111}, we have:
\begin{align}
    \Delta_1&\leq c\left\{\frac{\log(8T|\mathcal{X}[\cup_{k=1}^{K_\omega} Z_k^P]|/\delta)}{N_\tau(1-\gamma)^2}\cdot\boldsymbol{1} + (b'_\tau + \frac{\sqrt{2}}{(1-\gamma)^{\frac{3}{2}}})\sqrt{\frac{ \log(8T|\mathcal{X}[\cup_{k=1}^{K_\omega} Z_k^P]|/\delta)}{N_\tau}}\cdot\boldsymbol{1}\right\}\notag\\
    &\leq c b'_\tau \left\{\frac{2^\tau\log(8T|\mathcal{X}[\cup_{k=1}^{K_\omega} Z_k^P]|/\delta)}{N_\tau(1-\gamma)^{\frac{3}{2}}}+\frac{2^{\tau+2}}{1-\gamma}\sqrt{\frac{ \log(8T|\mathcal{X}[\cup_{k=1}^{K_\omega} Z_k^P]|/\delta)}{N_\tau}}\right\}\cdot\boldsymbol{1},\label{Delta1_eq1}
\end{align}
where the second inequality is obtained by using $b'_\tau = \frac{2^{-\tau}}{\sqrt{1-\gamma}}$.

\noindent \textbf{Bounding $\Delta_1$:}
For bounding $\Delta_2$, the difference comes from bounding $({I}-\gamma {P}^{\Tilde{\pi}_\omega^*}_\omega)^{-1}\sigma_{\mathcal{E}}$. We define $\sigma^2_{\Tilde{\mathcal{E}}}$ as the variance of induced by Q function $\widetilde{Q}$ and policy $\Tilde{\pi}^*$. Then, we have:
\begin{align}
    \|({I}-\gamma {P}^{\Tilde{\pi}_\omega^*}_\omega)^{-1}\sigma_{\mathcal{E}}\|_\infty &= \|({I}-\gamma {P}^{\Tilde{\pi}_\omega^*}_\omega)^{-1}(\sigma_{\Tilde{\mathcal{E}}}-\sigma_{\Tilde{\mathcal{E}}}+\sigma_{{\mathcal{E}}})\|_\infty\notag\\
    &= \|({I}-\gamma {P}^{\Tilde{\pi}_\omega^*}_\omega)^{-1}\sigma_{\Tilde{\mathcal{E}}}\|_\infty + \|({I}-\gamma {P}^{\Tilde{\pi}_\omega^*}_\omega)^{-1}(\sigma_{\Tilde{\mathcal{E}}}-\sigma_{{\mathcal{E}}})\|_\infty\notag\\
    &\leq \frac{3}{(1-\gamma)^{\frac{3}{2}}} + \frac{\|\widetilde{Q}-Q^*_\omega\|_\infty}{1-\gamma}.\label{def_mismatch_variance}
\end{align}

Combining Eq. \eqref{def_mismatch_variance} with Eq. \eqref{def_final_delta_2}, we have:
\begin{align}
    \Delta_2 \leq &cb'_\tau \left\{\frac{2^\tau\log(8T|\mathcal{X}[\cup_{k=1}^{K_\omega} Z_k^P]|/\delta)}{N_\tau(1-\gamma)^{\frac{3}{2}}}+\frac{2^{\tau+2}}{1-\gamma}\sqrt{\frac{ \log(8T|\mathcal{X}[\cup_{k=1}^{K_\omega} Z_k^P]|/\delta)}{N_\tau}}\right\}\cdot\boldsymbol{1}\notag\\
    &+c {\|\widetilde{Q}-Q^*_\omega\|_\infty}\sqrt{\frac{ \log(8T|\mathcal{X}[\cup_{k=1}^{K_\omega} Z_k^P]|/\delta)}{(1-\gamma)^2 N_\tau}}\cdot\boldsymbol{1}.\label{Delta2_eq1}
\end{align}

Combining Eq. \eqref{Delta2_eq1} and \eqref{Delta1_eq1}, we can conclude that:
\begin{align}
    \|\widetilde{Q}-Q^*_\omega\|_\infty \leq &cb'_\tau \left\{\frac{2^\tau\log(8T|\mathcal{X}[\cup_{k=1}^{K_\omega} Z_k^P]|/\delta)}{N_\tau(1-\gamma)^{\frac{3}{2}}}+\frac{2^{\tau+2}}{1-\gamma}\sqrt{\frac{ \log(8T|\mathcal{X}[\cup_{k=1}^{K_\omega} Z_k^P]|/\delta)}{N_\tau}}\right\}\notag\\
    &+c {\|\widetilde{Q}-Q^*_\omega\|_\infty}\sqrt{\frac{ \log(8T|\mathcal{X}[\cup_{k=1}^{K_\omega} Z_k^P]|/\delta)}{(1-\gamma)^2 N_\tau}}.\label{lm19-key-Q-bound}
\end{align}

Using a sample size $N_\tau =c_34^\tau{\log\left(6T|\mathcal{X}[\cup_{k=1}^{K_\omega} Z_k^P]|\right)}{(1-\gamma)^{-2}}$ with large enough $c_3$, we can guarantee:
\begin{align}
    &c\sqrt{\frac{ 4^\tau\log(8T|\mathcal{X}[\cup_{k=1}^{K_\omega} Z_k^P]|/\delta)}{ N_\tau}}  \leq \frac{1}{2},\label{lm19_key1}\\
    &{\frac{c \cdot 2^\tau \log(8T|\mathcal{X}[\cup_{k=1}^{K_\omega} Z_k^P]|/\delta)}{(1-\gamma)^{\frac{3}{2}}N_\tau}}\leq \frac{1}{8},\label{lm19_key2}\\
    &c\sqrt{\frac{ 4^{\tau+1}\log(8T|\mathcal{X}[\cup_{k=1}^{K_\omega} Z_k^P]|/\delta)}{(1-\gamma)^2 N_\tau}}  \leq \frac{1}{8}.\label{lm19_key3}
\end{align}

Substituting condition \eqref{lm19_key1}-\eqref{lm19_key3} into Eq. \eqref{lm19-key-Q-bound} leads to $\|\widetilde{Q}-Q^*_\omega\|_\infty \leq \frac{b'_\tau}{4}$. This concludes our proof.
\end{proof}

\section{Discussion on the Cost-Optimal Sampling Problem}
\label{Discussion_Cost-Optimal Sampling Problem}

\subsection{Relating the Cost-Optimal Sampling Problem to Graph Coloring}
The Graph Coloring Problem (GCP) is defined as follows: Given a graph \( G = (V, E) \), where \( V \) is a set of vertices and \( E \) is a set of edges, the task is to assign a color to each vertex such that no two adjacent vertices (i.e., vertices connected by an edge) share the same color. The objective is to find the optimal coloring scheme that minimizes the number of colors used, known as the \emph{chromatic number} of the graph \citep{erdHos1966chromatic}.

In the Cost-Optimal Sampling Problem (COSP), components with scopes \( Z_1^P, Z_2^P, \dots, Z_{K_\omega}^P \) must be divided into different groups, where the scopes of components in the same subset do not overlap (i.e., $Z_i^P \cap Z_j^P = \emptyset$ for any $i$, $j$ in the same group). This problem can be modeled as a variant of the GCP. Specifically, each component \( i \) represents a vertex in a graph, and two vertices \( i \) and \( j \) are connected by an edge if their scopes overlap, i.e., \( Z_i^P \cap Z_j^P \neq \emptyset \). Assigning a color to a vertex corresponds to assigning the component to a group. Since components in the same subset must have disjoint scopes, no two connected vertices (representing components with overlapping scopes) can share the same color. The goal is to find a coloring scheme that minimizes the total costs across all groups, where the cost of each subset is determined by the scope with largest factor set space $|\mathcal{X}[Z_i^P]|$ in that group.

\subsection{Proof for NP-Completeness of the COSP}
NP-completeness is used to characterize a subset of problems that are computationally hard to solve. Famous NP-complete problems include the Traveling Salesman Problem (TSP) \citep{lin1965computer}, the Knapsack Problem \citep{kellerer2004multidimensional}, the Hamiltonian Path Problem \citep{gary1979computers}, and the Satisfiability (SAT) Problem \citep{schaefer1978complexity}.

A problem is classified as \emph{NP-complete} if it satisfies two conditions:
\begin{itemize}
    \item \emph{In NP}:  The problem belongs to the class NP, meaning that given a proposed solution, we can verify its feasibility in polynomial time.
    \item \emph{NP-hard}: The problem is at least as hard as any other problem in NP. This means that any problem in NP can be transformed or reduced to an NP-hard problem in polynomial time, then we could solve the NP-hard problem to get the solution of any NP problem.
\end{itemize}

Graph coloring has been shown to be NP-complete \citep{karp2010reducibility}, as it is computationally difficult to find the minimum number of colors for an arbitrary graph. We now show that the COSP is NP-complete by proving: (1) it belongs to NP, and (2) it is at least as hard as the NP-complete GCP.

\subsubsection{The COSP belongs to NP}
In the COSP, given a solution (a partition of components into groups and their associated costs), we can verify its feasibility in polynomial time by:
\begin{itemize}
    \item Verifying that the scopes within each subset are disjoint, which can be done by checking pairwise intersections within each group. This can be done in $\mathcal{O}(K_\omega^2(n+m))$ time.
    \item Calculating the total cost of a partition by identifying the largest scope size in each subset and summing these values, which can also be computed in polynomial time, specifically $\mathcal{O}(K_\omega)$.
\end{itemize}

\subsubsection{The COSP is NP-hard}
We prove NP-hardness by reducing the GCP to the COSP. We demonstrate that for any instance of the GCP, a corresponding instance of the COSP can be constructed, and based on the optimal solution of the COSP, we can derive the optimal solution to the GCP.

Given any instance of the GCP with a graph \( G = (V, E) \), we construct a corresponding instance of the COSP as follows:
\begin{itemize}
    \item \emph{Step 1. Assign unique base dimensions for each vertex}: 
    Let \( |V| \) be the number of vertices, we initial $|V|$ components with the scope for each component \( i \in[|V|]\) satisfies \( Z_i^P = \{i\} \). This ensures that each component starts with a unique, non-overlapping dimension.
\item \emph{Step 2. Create shared dimensions for edges}:  For each edge \( (v_i, v_j) \in E \), we introduce a new and unique dimension \( d_{i,j} \) that will belong to both \( Z_i^P \) and \( Z_j^P \). Specifically, we updates the scopes of components $i$ and $j$ by 
\begin{align*}
    Z_i^P = Z_i^P \cup \{d_{i,j}\}, \quad Z_j^P = Z_j^P \cup \{d_{i,j}\}.
\end{align*}
   This ensures that for every edge in the graph, the corresponding components have overlapping scopes, mimicking the adjacency constraint in the GCP.
\item \emph{Step 3. Enforce disjointness for non-adjacent vertices}: If there is no edge between two vertices \( v_i \) and \( v_j \), the corresponding scopes should remain disjoint, i.e., \( Z_i^P \cap Z_j^P = \emptyset \). This property is automatically maintained because shared dimensions are only introduced for pairs of vertices that are connected by an edge.
\end{itemize}

We provide an example to illustrate the construction process. Consider a small example graph \( G = (V, E) \) with three vertices \( V = \{v_1, v_2, v_3\} \) and edges \( E = \{(v_1, v_2), (v_2, v_3)\} \). The construction process is:
\begin{itemize}
    \item \emph{Step 1. Assign unique base dimensions for each vertex}: \( Z_1^P = \{1\} \), \( Z_2^P = \{2\} \) and \( Z_3^P = \{3\} \).
    \item \emph{Step 2. Create shared dimensions for edges}: For edge \( (v_1, v_2) \), we introduce \( d_{1,2} = 4 \), so $Z_1^P = \{1, 4\}, \quad Z_2^P = \{2, 4\}$. Also, for edge \( (v_2, v_3) \), we introduce \( d_{2,3} = 5 \), so $Z_2^P = \{2, 4, 5\}, \quad Z_3^P = \{3, 5\}$.
    \item \emph{Step 3. Resulting scopes}: the resulting scopes are:
    \begin{align*}
Z_1^P = \{1, 4\}, \quad      Z_2^P = \{2, 4, 5\}, \quad Z_3^P = \{3, 5\}.
    \end{align*}
\end{itemize}

The objective of the COSP is to partition these components into disjoint groups (sets of components with no overlapping scopes) while minimizing the total sampling cost. Since the sampling cost for each component is set to 1, minimizing the total sampling cost is equivalent to minimizing the number of groups (colors) used, as each subset contributes 1 to the total cost. Therefore, solving the COSP on this constructed instance is equivalent to solving the GCP. 

\subsubsection{NP-completeness}
Since the GCP can be reduced to the COSP in polynomial time, and solving the COSP provides a solution to the GCP, the COSP is NP-hard. Furthermore, since the COSP belongs to NP, it is NP-complete.

\section{Detailed Model of Wind Farm-equipped Storage Control}
\label{detail_storage control}

\subsection{System Model}
Consider a wind farm tasked with supplying power to the grid under an electricity supply contract. The predetermined supply quantity of wind power at any given time $t$ is denoted by $\widehat{w}_t$. This value is established based on wind power forecasts, which can be made from a minute to a day in advance, either by the wind farm itself or the Independent System Operator (ISO). As time progresses, the actual wind power generation $w_t$ is sequentially disclosed in real-time. The wind farm is equipped with a storage system, allowing it to either store wind energy, denoted as $u^{+}_{t}$, or release stored energy, denoted as $u^{-}_t$ into the grid at any time $t$. As a result, the total power delivered to the grid, represented by $g_t$, is determined by
\begin{align*}
    g_t = w_t +v^-_t -v^+_t.
\end{align*}

When the mismatch exists between commitment $\widehat{w}_t$ and the delivered power $g_t$, the wind farm will be charged with a penalty cost $c_t(\widehat{w}_t, g_t)$ as follows:
\begin{align*}
    c_t(\widehat{w}_t, g_t) = p^+_t\max(g_t-\widehat{w}_t,0) + p^-_t\max(\widehat{w}_t-g_t,0),
\end{align*}
where $p_t^+$ and $p_t^-$ denote the unit penalty costs for wind generation surplus and shortage at time $t$, respectively.

The wind farm targets to minimize the accumulated mismatch penalty costs across all $T$ time slots by reasonably controling the storage system.
Mathematically, the storage control problem can be formulated as follows:
\begin{align}
    \textbf{(P1)}\quad\min_{v_t^+, v_t^-, \forall t} \quad&\sum\nolimits_{t = 1}^{T}c_t(\widehat{w}_t, g_t)\\
    s.t.\quad & g_t = w_t +v^-_t -v^+_t, \forall t,\label{g_t}\\
    & SoC_1 = \frac{C}{2}, \label{SoC_1}\\
    & SoC_{t+1} = SoC_t + \eta^+v^+_t - \eta^-v^-_t, \forall t,\label{SoC_{t+1}}\\
    & \eta^+v_t^+ \leq C-SoC_t, \forall t,\label{v_t^+}\\
    & \eta^-v_t^- \leq SoC_t, \forall t,\label{v_t^-}\\
    & v_t^+ \leq w_t, \forall t, \label{eta_t^+  w_t}\\
    & v_t^+, v_t^- \geq 0, \forall t,\label{v_t^+, v_t^-}\\
    & v_t^+v_t^- = 0, \forall t.\label{v_t^+v_t^-}
\end{align}

In the optimization, the decision variables at time $t$ include:
\begin{itemize}
    \item $v_t^+$: generated wind power which is charged to the energy storage;
    \item $v_t^-$: discharged energy from the storage to the grid;
\end{itemize}

And the other functions, latent variables, and parameters include:
\begin{itemize}
\item $c_t(\cdot)$: total penalty cost at time $t$;
\item $\widehat{w}_t$: committed wind power supply at time $t$;
\item $w_t$: wind power generation at time $t$;
\item $T$: total duration of storage control decisions;
    \item $p^+_t$, $p^-_t$: unit grid penalty prices for power generation shortage and surplus at time $t$, respectively;
    \item $g_t$: actual supplied energy at time $t$;
    \item $SoC_t$: state-of-charge (SOC) of storage at time $t$;
    \item $C$: energy storage capacity;
    \item $\eta^+$, $\eta^-$: charging and discharging efficiencies of storage;
\end{itemize}

Constraint \eqref{g_t} represents the delivered power; constraints \eqref{SoC_1} and \eqref{SoC_{t+1}} characterize the dynamics of storage; and constraints \eqref{v_t^+} and \eqref{v_t^-} denote the storage capacity limits. Constraint \eqref{eta_t^+  w_t} and \eqref{v_t^+, v_t^-} show the upper and lower limits of storage control actions, and constraint \eqref{v_t^+v_t^-} indicates that the storage cannot be charged and discharged simultaneously.

Due to the inherent stochasticity of penalty prices and renewable energy production, it is impractical to obtain the optimal future storage control decisions.
Therefore, in practice, we often consider sequential storage control.
Specifically, at each time $t$, we determine the current storage control actions $v_t^+, v_t^-, c_t$ based on the currently available information.
Consequently, we establish the following online storage control problem at time $t$ as follows:
\begin{align}
    \textbf{(P2)}\quad\min_{v_t^+, v_t^-} \quad&c_t(\widehat{w}_t, g_t) + \sum_{\tau = t+1}^{\infty}\gamma^{\tau-t}\mathbb{E}(c_{\tau}(\widehat{w}_{\tau}, g_{\tau}))\\
    s.t.\quad & g_t = w_t +v^-_t -v^+_t, \label{p2_c1}\\
    & SoC_{{t}+1} = SoC_{t} + \eta^+v^+_{t} - \eta^-v^-_{t}, \label{SOC_dynamic_p2}\\
    & \eta^+v_t^+ \leq C-SoC_t, \label{eta^+v_t^+_single}\\
    & \eta^-v_t^- \leq SoC_t, \\
    & v_t^+ \leq w_t, \\
    & v_t^+, v_t^- \geq 0, \\
    & v_t^+v_t^- = 0.  \label{NoSimuCD_single}
\end{align}

\subsection{Markov Decision Process Modeling}
We highlight that problem \textbf{(P2)} can be transformed into MDP in the following manner:

\noindent\textbf{Markov Decision Process} $(\mathcal{S},\mathcal{A},\mathcal{P},\mathcal{R})$:
\begin{itemize}
    \item States $\mathcal{S}$: Any state $s\in \mathcal{S}$ is composed of the penalty prices $p^+, p^-$, the committed and real wind power generation $\widehat{w}$ and $w$, and state-of-charge $SoC$. Formally, ${s} = (p^+, p^-, \widehat{w}, w, SoC)$;
    \item Actions $\mathcal{A}$: Any action $a \in \mathcal{A}$ is composed of the charge amount $v^+$ and discharge amount $v^-$. Formally, $a = (v^+, v^-)$;
    \item Transition probability $\mathcal{P}$: $\mathcal{P}$ is the transiting probability matrix $\mathcal{P}_{a} = \{\textbf{Pr}(s_{t+1} = s'|s_t = s, a_t = a), \forall s, s' \in \mathcal{S}, a \in \mathcal{A} \}$, which includes the probability of transiting from state $s$ to $s'$ with action $a$ for all $s$, $s'$ and $a$;
    \item Reward $\mathcal{R}$: $\mathcal{R}$ is the immediate reward (penalty in our case) after transiting from state $s$ to state $s'$ due to action $a$, i.e., $\mathcal{R} = \{r(s, a),\forall s, a\}$. Specifically, in the one-shot decision problem, the penalty $r(s,a)$ equals the negative of the penalty, i.e.,
    \begin{align*}
        r(s,a) = p^+\max(w-\widehat{w},0) + p^-\max(\widehat{w}-w,0).\notag
    \end{align*}
\end{itemize}

We can observe that, for the storage control problem, the state space $\mathcal{S}$ and the action space $\mathcal{A}$ are known. The reward $\mathcal{R}$ is also known once the state $s$ and action $a$ are decided. The only unknown comes from the transition probability $\mathcal{P}$. 
However, some important observations for $\mathcal{P}$ can simplify the problem. Specifically, we can divide the state variables ${s} = (p^+, p^-, \widehat{w}, w, SoC)$ into one deterministic state and several random states. The deterministic state is $SoC$, which can be determined following Eq. \eqref{SOC_dynamic_p2} without any uncertainty. And the random states include $p^+, p^-, \widehat{w}, w$, which are fully random\footnote{Note that, the committed wind power $\widehat{w}$ is essentially not random. However, only the generation mismatch $w-\widehat{w}$ exists in our problem, and such mismatch can be regarded as a random variable.}. 

In our numerical study, we assume $p = p^+=p^-$, and observe that only $\Delta = \widehat{w}-w$ exists in the reward function. Therefore, the state can be rewritten as $\mathcal{S} = (p,\Delta w, SoC)$. For the factorization scheme $\omega$, we use the following factorization to the transition kernel:
\begin{align*}
    \widehat{P}(p_{t+1}, \Delta_{t+1}, SoC_{t+1}|p_{t}, \Delta_{t}, SoC_{t}, a_t) = \widehat{P}(p_{t+1}|p_t)\widehat{P}(\Delta_{t+1}|\Delta_{t})\widehat{P}(SoC_{t+1}|SoC_{t},a_t).
\end{align*}

\subsection{Parameter Settings}
In the numerical study, we utilized the California aggregate wind power generation dataset from CAISO \citep{caiso} containing predicted and real wind power generation data with a 5-minute resolution spanning from January 2020 to December 2020. The penalty price equals the average electricity price of CASIO \citep{caiso} with the matching resolution and periods.
We set $C = 500$ kWh, $\gamma = 0.9$. The discretization levels of $p$, $\Delta w$ and $SoC$ are set to be $8$, $8$, $50$, respectively. The action set includes $3$ discretized choices: charging (discharging) to satisfy $100\%$, $50\%$, and $0\%$ of energy mismatches, respectively. 
\end{document}